\documentclass[preprints,article,accept,moreauthors,pdftex]{Definitions/mdpi}
\DeclareUnicodeCharacter{FF08}{\'{e}}
\firstpage{1} 
\pubvolume{1}
\issuenum{1}
\articlenumber{0}
\pubyear{2021}
\copyrightyear{2020}
\datereceived{} 
\dateaccepted{} 
\datepublished{} 
\hreflink{https://doi.org/} % If needed use \linebreak
\usepackage{amsmath,amssymb,amsfonts,amsthm}
\usepackage{algorithmic}
\usepackage[ruled,linesnumbered]{algorithm2e}%
\usepackage{textcomp}
\usepackage{hyperref}       % hyperlinks
\usepackage{url}            % simple URL typesetting
\usepackage{booktabs}       % professional-quality tables
\usepackage{nicefrac}
\usepackage{microtype}      % microtypography
\usepackage{lipsum}
\usepackage{enumitem}
\usepackage{listings}
\usepackage{graphicx}
\usepackage{epstopdf}
\usepackage{array}
\usepackage{subcaption}
\usepackage{wrapfig}
\usepackage{subcaption}
\usepackage{mathtools,cuted}
\usepackage{color}
\usepackage{colortbl}
\usepackage{cases}
\DeclarePairedDelimiter\abs{\lvert}{\rvert}
\DeclarePairedDelimiter\norm{\lVert}{\rVert}
\usepackage{stackengine,graphicx}
\stackMath

\usepackage{amssymb}
\usepackage[normalem]{ulem}
\usepackage{bm}
\usepackage{mathtools}
\usepackage{mathrsfs}
\usepackage{verbatim}
\usepackage{comment}
\newcommand{\xbf}{\mathbf{x}}
\newcommand{\ybf}{\mathbf{y}}
\newcommand{\rbf}{r}
\newcommand{\bbf}{\mathbf{b}}
\newcommand{\nbf}{\mathbf{n}}
\newcommand{\ubf}{\mathbf{u}}
\newcommand{\vbf}{\mathbf{v}}
\newcommand{\wbf}{\mathbf{w}}

\newcommand{\Pbf}{\mathbf{P}}
\newcommand{\gbf}{\mathbf{g}}
\newcommand{\Abf}{\mathbf{A}}

\DeclareMathOperator{\prox}{\mathrm{prox}}
\newcommand{\etol}{\epsilon_{\mathrm{tol}}}

\newcommand{\R}{\mathbb{R}}
\newcommand{\NN}{\mathbb{N}}

\newcommand{\CC}{\mathbb{C}}
\DeclareMathOperator*{\argmin}{arg\,min}
\newcommand{\F}{\mathcal{F}}

\newcommand{\D}{\mathcal{D}}

\newcommand{\xt}{\mathbf{x}_{t}}

\newcommand{\xtp}{\mathbf{x}_{t+1}}
\newcommand{\xpp}{\mathbf{x}_{p+1}}
\newcommand{\ztp}{\mathbf{z}_{t+1}}
\newcommand{\utp}{\mathbf{u}_{t+1}}
\newcommand{\vtp}{\mathbf{v}_{t+1}}

\newcommand{\Ccal}{\mathcal{C}}

\newcommand{\Xcal}{\mathcal{X}}

\newcommand{\gi}{\gbf_i}
\newcommand{\tl}{t_{l}}
\newcommand{\tlp}{t_{l+1}}
\newcommand{\epst}{\varepsilon_{t}}
\newcommand{\epsp}{\varepsilon_{p}}
\newcommand{\epstp}{\varepsilon_{t+1}}

\newcommand{\taut}{\tau_t}

\newcommand{\gj}{\gbf_j}
\newcommand{\xhat}{\hat{\mathbf{x}}}
\usepackage{tikz}
\usetikzlibrary{shapes,arrows,chains}
\newtheorem{lemma}[theorem]{Lemma}

\Title{An Optimization-Based Meta-Learning Model for MRI Reconstruction with Diverse Dataset}
\TitleCitation{Title}

\Author{Wanyu Bian $^{1,\dagger,}$*\orcidA{}, Yunmei Chen $^{1}$, Xiaojing Ye$^{2}$, and Qingchao Zhang $^{1}$}

\AuthorNames{Wanyu Bian, Yunmei Chen, Xiaojing Ye, and Qingchao Zhang}

\AuthorCitation{Bian, W.; Chen, Y.; Ye, X.; Zhang, Q.}
\address{%
$^{1}$ \quad Department of Mathematics, University of Florida, Gainesville, FL 32611, USA; wanyu.bian@ufl.edu, yun@ufl.edu, qingchaozhang@ufl.edu; \\
$^{2}$ \quad Department of Mathematics and Statistics, Georgia State University, Atlanta, GA 30303, USA; xye@gsu.edu}

\corres{Correspondence: wanyu.bian@ufl.edu;}

\firstnote{Author list in alphabetical order.} 
\abstract{(1) Purpose: This work aims at developing a generalizable MRI reconstruction model in the meta-learning framework.
The standard benchmarks in meta-learning are challenged by learning on diverse task distributions. The proposed network learns the regularization function in a variational model and reconstructs MR images with various under-sampling ratios or patterns that may or may not be seen in the training data by leveraging a  heterogeneous dataset.
(2) Methods: We propose an
unrolling network induced by learnable optimization algorithms (LOA) for solving our nonconvex nonsmooth variational model for MRI reconstruction. In this model, the learnable regularization function contains a  task-invariant common feature encoder and task-specific learner represented by a shallow network. To train the network we split the training data into two parts: training and validation, and introduce a bilevel optimization algorithm. The lower-level optimization trains task-invariant parameters for the feature encoder with fixed parameters of the task-specific learner on the training dataset, and the upper-level optimizes the parameters of the task-specific learner on the validation dataset.
(3) Results: The PSNR increases 1.5 dB on average compared to the network trained through conventional supervised learning on the seen CS ratios. We test the result of quick adaption on the unseen tasks after meta-training, the average PSNR arises 1.22 dB compared to the conventional learning procedure that is directly trained on the unseen CS ratios in the meanwhile saving half of the training time. The average PSNR arises 1.87 dB for unseen sampling patterns comparing to conventional learning; 
(4) Conclusion: We proposed a meta-learning framework consisting of the base network architecture,  design of regularization, and bi-level optimization-based training. The network inherits the convergence property of the LOA and interpretation of the variational model. The generalization ability is improved by the designated regularization and bilevel optimization-based training algorithm.   
}

\keyword{MRI reconstruction; Meta-learning; Domain generalization.} 
\begin{document}
%%%%%%%%%%%%%%%%%%%%%%%%%%%%%%%%%%%%%%%%%%

\section{Introduction}

Deep learning models and algorithms are often built for specific tasks and the samples (usually) follow a certain distribution. Specifically, the source-domain/training samples and target-domain/testing samples are drawn from the same distribution. 
%
%However, this assumption is too strong and may not hold in practical applications. 
%
However, these data sets are often collected at different sources and exhibit substantial heterogeneity, and thus the samples may follow closely related but different distributions in real-world applications.
%
%Hence, generalizability becomes an inevitable concern of deep learning. 
%
Therefore, robust and efficient training of deep neural networks using such data sets is theoretically important and practically relevant in deep learning research.

Meta-learning provides an unique paradigm to overcome this challenge \cite{munkhdalai2017meta, finn2017model,li2018learning, rusu2018meta,  yao2021improving, balaji2018metareg}.
Meta-learning is known as \emph{learning-to-learn} which aims to gain the capability of quickly learning unseen tasks from the experience of learning episodes that covers the distribution of relevant tasks. In a multiple-task scenario, given a family of tasks, meta-learning has been proven as a useful tool for extracting task-agnostic knowledge and improve the learning performance of new tasks from that family \cite{thrun1998learning, hospedales2021meta}.
Most of the modern approaches of meta-learning are focus on \emph{domain adaptation}, where the target domain information is available in meta-training, but this setting is still too ideal to apply in practice. Recently the study on \emph{domain generalization (DG)} has attracted much attention. DG allows the training model to learn representations from several related source domains and gain adequate generalization ability for unseen test distributions. However, DG techniques for image reconstruction are rarely explored in solving inverse problems. Leveraging large-scale heterogeneous MRI data to overcome inflated prediction performance caused by small-scale clinic datasets is of great interest for more precise, predictable, and powerful health care. For the sack of reducing scan time and improve diagnostic accuracy of MRI, the data are often acquired by using different under-sampling patterns (e.g., Cartesian mask,  Radial mask, Poisson mask), under-sampling ratios,  and different settings of the parameters resulting in different contrast (e.g., T1-weighted, T2-weighted,  proton-density (PD), and Flair). Hence the  public data are heterogeneous. This work aims at developing a generalizable MRI reconstruction method in the meta-learning framework that can leverage large-scale heterogeneous datasets to reconstruct MR images with the available small-scale dataset from specific setting of the scan.
%(only)  (The objective of this paper focuses on domain generalization for solving inverse problems and utilizes the domain invariant parameters to make predictions on the unseen domain through a meta-learning procedure.)In this paper we propose to unroll the deep network in a learned gradient descent algorithm as the forward model, and the network parameters are trained under the guidance of meta-knowledge which is learned in an optimization algorithm inspired by bilevel learning.
%(Medical image processing often suffers from data heterogeneous problems because of the inconsistency between the distribution of the training dataset and the testing dataset. Different medical image sets could be scanned from different scanners with different scanning coefficients and inharmonious protocols.  Also, acquiring labeled medical images are more expensive and laborious than natural images since medical images require intensive workforce who are expert in the radiologist or physicists with medical experiences for diagnostics. )

In this work, we propose a meta-learning-based model for solving the MRI image reconstructions problem by leveraging diverse/heterogeneous dataset. 
Data acquisition in compressed sensing MRI (CS-MRI) with k-space (Fourier space) undersampling can be formulated as follows
\begin{equation}\label{MRI}
    \ybf = \Pbf \F \xbf + \nbf,
\end{equation}
where $ \ybf \in \CC^p$ is the measured k-space data with total of $p$ sampled data points, $ \xbf \in \CC^{N\times 1}$ is unknown MR image with $N$ pixels to be reconstructed, $ \F \in \CC^{N \times N}$ is the 2D discrete Fourier transform (DFT) matrix, and $ \Pbf \in \R^{p \times N}$ $(p< N)$ is the binary matrix representing the sampling mask in k-space. %For multi-coil acquisitions in MRI, $  \Ebf = \Pbf \Fbf \Sbf_j \in \CC^{p\times N} $ where $ \Sbf_j \in \R^{N \times N}$ is a diagonal matrix called coil sensitivity map of the $j$th coil, which is either given or estimated in advance. Therefore $ \Ebf \xbf \in \CC^p$ is the vector of partial Fourier coefficients and k-space data acquisition is expressed as Equation \eqref{MRI}, where $ \nbf$ represents the measurement noise.

Reconstruction of $\xbf$ from (noisy) undersampled data $\ybf$ is an ill-posed problem. 
An effective strategy to elevate the ill-posedness issue is to incorporate prior information by adding a regularization term. In light of the substantial success of deep learning and the massive amount of training data available nowadays, we propose to learn the regularization term and then solve $\xbf$ from the following unconstrained optimization problem
\begin{equation}\label{csmodel}
   \bar{\xbf} = \argmin_{\xbf} \frac{1}{2} \| \Pbf \F \xbf - \ybf \|^2 + R(\xbf; \Theta),
\end{equation}
where $\|\cdot\|$ is the standard 2-norm of vectors, the first term in the objective function in \eqref{csmodel} is data fidelity term that ensures consistency between the reconstructed $ \xbf$ and the measured data $\ybf$, and the second term $ R(\xbf; \Theta)$ is the regularization term that introduces prior knowledge to the image to be reconstructed and represented by a convolutional neural network (CNN) with parameters $\Theta$ that learned from training data. %(through convolutional neural network (CNNs).  The coefficient $\lambda$   balances the trade-off between data fidelity and regularization terms)for example, $ \| \psi \xbf \|_1$, which enforces the sparsity of $ \psi \xbf$ with  a fixed sparsifying transform $ \psi$ such as DFT and wavelet transforms. 
We will discuss in details how this regularization is formed in Sections \ref{sec:Method} and \ref{sec:Implementation}.

In this paper, we first introduce a nonconvex and nonsmooth learnable optimization algorithm (LOA) with rigorous convergence guarantees. Then we construct a deep reconstruction network by following the LOA exactly. This approach is inspired by the work \cite{chen2020learnable}, but the LOA developed in this work is different from the one in \cite{chen2020learnable}. 
%
%architecture is the same as algorithm, enherit %modified to improve the model performance and becomes more efficient, which slashes the computational cost. 
%
Furthermore, we propose a novel MRI reconstruction model \eqref{our_model} with a regularization term for multi-task adaptation, which consists of a task-invariant regularization and a task-specific hyperparameter. The former extracts common prior information of images from different tasks, which learns task-invariant parameters $\theta$. The latter exploits the proper task-specific parameters (also called meta-knowledge) $\omega_i$ for each individual task $i$.
The purpose is to increase the generalization ability of the learned regularization,
so that the trained LOA-induced deep reconstruction network can perform well on both seen and unseen tasks.
To train the network we split the training data into two parts:  training and validation, and introduce a bilevel optimization model for learning network parameters. The lower-level optimization learns the task-invariant parameters $\theta$ of the feature encoder with fixed task-specific parameters $\omega_i$ on the training dataset, and the upper-level optimizes the task-specific parameters on the validation dataset. The well-trained network of seen tasks can be applied to the unseen tasks with determined $\theta$. This adaptation only needs a few iterations to update $\omega_i$ with a small number of training samples of unseen tasks.
%Different from the aspiration of most optimization-based meta-learning, which searching for a universal network parameter that can generalize easily to multiple tasks, the adaptive regularization consists of task-invariant parameters $\theta$ and task-specific parameters (also called meta-knowledge) $\omega_i$.(The parameters $\theta$ are trained under the guidance of $\omega_i$ which is learned in an optimization algorithm inspired by bilevel learning in \eqref{eq:bi-level} employing training and validation data. The bilevel training enhances generalizability and reduces the overfitting of the proposed model.)

%In the proposed work, instead of directly learning a feed-forward network, we propose to learn a regularization of variational model solved by the nonsmooth nonconvex optimization algorithm in Section \ref{network}. Then we unroll it to a multi-phase shallow deep network that applies only a few kernels and convolutions of CNNs. 
%The learned regularization term can incorporate the common underlying properties of the training tasks instead of simply "fitting" the training data. 
%The common parameters in regularization are learned from a  heterogeneous dataset using method of meta-learning.
%The proposed meta-learning model with learned regularization unrolling can achieve great generalizability even on heterogeneous data. 

As demonstrated by our numerical experiments in Section \ref{experiments}, our proposed framework yields much improved image qualities using diverse data sets of various undersampling trajectories and ratios for MRI image reconstruction. 
The underpinning theory is that an effective regularization can integrate common feature and prior information from a variety of training samples from diverse data sets, but they need to be properly weighted in the reconstruction of the image under specific sample distribution (i.e., undersampling trajectory and ratios).  %
%
% with insufficient training data through meta-training procedure so that each individual task benefit and make compensation to each other by leveraging the related cross-task information from  heterogeneous MRI datasets among tasks. 
%
%In the final analysis, the proposed LOA-based algorithm plays a key role in learning the entire regularization term in the MRI reconstruction model.
%
Our contributions can be summarized as follows:
\begin{enumerate}[leftmargin=*]
\item LOA inspired network architecture. Our network architecture exactly follows a proposed LOA of guaranteed convergence. So the network is interpretable, parameter-efficient, stable, and generates high-quality reconstructions after proper training.

\item Design of Regularization. Unlike the existing meta-learning methods, the proposed approach of network training can learn adaptive regularizer from diverse data sets.
Our adaptive regularizer consists of the task-invariant portion and the task-specific portion. The task-invariant portion aims at exploiting the common features and shared information across all involved tasks, whereas the task-specific parameters only learn the regularization weight to properly balance the data fidelity and learned regularization in individual tasks. 

\item A novel bilevel network training algorithm for improving generalizability. We utilize train data and validation data for training the designed network, which follows a bilevel optimization algorithm that is trained from a heterogeneous dataset. The lower-level determines the optimal task-invariant parameters for any fixed task-specific parameters, and the upper-level optimizes the task-specific parameters on validation data so that the task-invariant parameter can be adapted to different tasks.

%We propose a novel deep model to recover the under-sampled MRI across different k-space trajectories and sampling patterns synergistically via a meta-learning framework. We propose a novel deep model to recover the under-sampled MRI across different sampling trajectories with different ratios via  meta-learning framework, where the low-level is to  learn parameters in the regularization $r(\sbf, \theta) $ by training a task-invariant feature extraction operator and the corresponding base network, and the upper-level is to adjust the network adaptive to different sampling ratios by tuning an meta-knowledge parameter $\omega$. \\
%\item  Different from existing meta-learning algorithms that devoting applications on classification or regression, the proposed method applied meta-learning on solving variational model and learn the weighted regularizer for fast adaption to various tasks synergistically. 
\item  Better generalization and faster adaption.  On reconstructing the under-sampled MRIs with various scanning trajectories, the network after meta-training can be directly applied to reconstruct the images sampled on multiple seen trajectories and can achieve quick adaption to the new unseen trajectories sampling patterns.
%
%by merely training on inadequate samples within a few numbers of iterations. %The average PSNR increases 1.50 dB and 1.22 dB compared to the conventional supervised training on seen trajectories and unseen trajectories respectively. For unseen sampling patterns, the proposed method improved 1.87 dB in PSNR comparing to conventional learning.
\end{enumerate}

%The proposed meta-learning method improves the reconstruction performance on diversified datasets over the conventional supervised learning.\ye{I think item 1 is no longer a contribution. Just focus on 2. We can say that we improve the training efficiency and reconstruction performance of LDA from diversified data sets using a meta-learning approach.}

The remainder of the paper is organized as follows.
In Section \ref{related_work}, we introduce some related work for both optimization-based meta-learning and deep unrolled networks for MRI reconstructions. We sketch our meta-learning model and the neural network in Section \ref{sec:Method} and describe the implementation details in Section \ref{sec:Implementation}. Section \ref{experiments} introduces the reconstruction results. Section \ref{conclusion} concludes the paper.

%%%%%%%%%%%%%%%%%%%%%%%%%%%%%%%%%%%%%%%%%%
\section{Related work}\label{related_work}

Meta-learning methods have in recent years yielded impressive results in various fields with different techniques, and plays a different roles in different communities, it has a very broad definition and perspective in the literature \cite{hospedales2021meta}. Several survey papers \cite{ hospedales2021meta, huisman2021survey} have a detailed introduction of this development. 
Three categories of meta-learning techniques are commonly grouped  \cite{yao2020automated, lee2018gradient, huisman2021survey} as metric-based methods \cite{koch2015siamese, vinyals2016matching, snell2017prototypical},
model-based methods  \cite{mishra2017simple, ravi2016optimization, qiao2018few, graves2014neural},
and optimization-based methods \cite{finn2017model, rajeswaran2019meta, li2017meta}, which are often cast as a bi-level optimization problem and exhibits relatively better generalizability on wider task distributions. We mainly focus on optimization-based meta-learning in this paper.

\subsection{ Optimization-based approaches}\label{l2l}
In the context of optimization-based meta-learning, a popular strategy is gradient descent based inner optimization \cite{finn2017model,antoniou2018train,rajeswaran2019meta,li2017meta,nichol2018first,finn2019online,grant2018recasting,finn2018probabilistic,yoon2018bayesian}. The optimization problem is often cast as a bilevel learning, where the outer level (upper/leader level) optimization is solved subject to the optimality of inner level (lower/follower level) optimization so that the optimized model can generalize well on the new data. The inner level encounters new tasks and tries to learn the associated features quickly from the training observations, the outer level accumulates task-specific meta-knowledge across previous tasks and the meta-learner provides support for the inner level so that it can quickly adapt to new tasks. e.g. the Model-Agnostic Meta-Learning (MAML) \cite{finn2017model} aims at searching for an adaptive initialization of the network parameters for new tasks with only a few steps that updates in the inner level. In recent years, a large number of followup works of MAML proposed to improve  generalization using similar strategy \cite{lee2018gradient, rusu2018meta, finn2018probabilistic, grant2018recasting, nichol2018first, vuorio2019multimodal, yao2019hierarchically,yin2020metalearning}.
Deep bilevel learning \cite{jenni2018deep} seeks to obtain better generalization that trained on one task and generalize well on another task. Their model is to optimize a regularized loss function to find network parameters from training set and identify hyperparameters so that the network performs well on validation dataset.

When the unseen tasks lie in inconsistent domains with the meta-training tasks, as revealed in \cite{chen19closerfewshot}, the generalization behavior of the meta-learner will be compromised.
This phenomenon partially arises from the meta-overfitting on the already seen meta-training tasks, which is identified as memorization
problem in \cite{yin2020metalearning}.
A meta-regularizer forked with information theory is proposed in \cite{yin2020metalearning} to handle the memorization
problem by regulating the information dependency during the task adaption. 
MetaReg \cite{balaji2018metareg} decouples the entire network into the feature network and task network, where the meta-regularization term is only applied to the task network. They first update the parameters of the task network with meta-train set to get the domain-aligned task network, then update the parameters of meta-regularization term on meta-test set to learn the cross-domain generalization. Different from MetaReg, Feature-Critic Networks \cite{li2019feature} exploits the meta-regularization term to pursue a domain-invariant feature extraction network. The meta-regularization is designed as a feature-critic network that takes the extracted feature as input. The parameters of the feature extraction network are updated by minimizing the new meta-regularized loss. The auxiliary parameters in the feature-critic network are learned by maximizing the performance gain over the non-meta case.  
To effectively evaluate the performance of the meta-learner, several new benchmarks \cite{Rebuffi17,48798,yu2020meta} were developed  under more realistic settings operate well on diverse visual domains. As mentioned in \cite{48798}, the generalization to unseen tasks within multimodal or heterogeneous datasets remains as a challenge to the existing meta-learning methods. 

The aforementioned methods pursue domain generalization for the classification networks that learned regularization function to learn cross-domain generalization. Our proposed method is aiming for solving the inverse problem and we construct an adaptive regularization that not only learns the universal parameters among tasks, but also the task-aware parameters. The designated adaptive regularizer assists the generalization ability of the deep model so that the well-trained model could be able to perform well on a  heterogeneous datasets of both seen or unseen tasks.

\subsection{MRI reconstruction models and algorithms}

MRI reconstruction is generally formulated as a variation model in the format of \eqref{csmodel},  which is a data fidelity term plus a regularization term. Traditional methods employ handcrafted regularization terms such as total variation (TV), and the solution algorithm follows a theoretical justification. However, it is hard to tune the associated parameter  (often require hundreds even thousands of iterations to converge) to capture subtle details and satisfy clinic diagnostic quality.

Deep learning based model leverages large dataset and further explore the potential improvement of reconstruction performance comparing to traditional methods and has successful applications in clinic field \cite{lundervold2019overview, liang2020deep, sandino2020compressed, mccann2017convolutional, zhou2020review, singha2021deep, chandra2021deep, ahishakiye2021survey}. 
However, training generic deep neural networks (DNNs) may prone to over-fitting when data is scarce as we mentioned in the beginning. Also, the deep network structure behaves like a black box without mathematical interpretation.
To improve the interpretability of the relation between the topology of the deep model and reconstruction results, a new emerging class of deep learning-based methods is known as \emph{learnable optimization algorithms} (LOA) attracts attention in the literature \cite{liu2020deep, liang2020deep}. LOA was proposed to map existing optimization algorithms to structured networks where each phase of the networks correspond to one iteration of an optimization algorithm or replace some ingredients such as proximal operator \cite{cheng2019model,bian2020deep}, matrix transformations \cite{yang2018admm, hammernik2018learning, zhang2018ista}, non-linear operators \cite{yang2018admm, hammernik2018learning}, and denoiser/regularizer \cite{aggarwal2018modl, schlemper2017deep} etc., by CNNs to avoid difficulty for solving non-smooth non-convex problems. 

Different from the current LOA approach for image reconstruction that simply imitates the algorithm iterations without any convergence justification or just replaces some hardly solvable components with sub-networks, our proposed LOA-induced network inherits the convergence property of the proposed Algorithm \ref{alg:lda}, where we provide convergence analysis in Appendix \ref{convergence}. The proposed network retains the interpretability of the variational model, parameter efficiency and contributes to better generalization.

%%%%%%%%%%%%%%%%%%%%%%%%%%%%%%%%%%%%%%%%%%
\section{Proposed Method}\label{sec:Method}
%%%%%%%%%%%%%%%%%%%%%%%%%%%%%%%%%%%%%%%%%%
\subsection{LOA-induced reconstruction network}
\label{network}
Our deep algorithmic unrolling reconstruction network represented as $F_{\Theta}$ where $\Theta$ denotes the set of all learnable parameters in the model. Motivated by the variational model, for an input partial k-space data $\ybf$ in $\D_{\tau_i}$, we desire the network output $F_{\Theta}(\ybf)$ to be an optimizer of the following minimization problem as \eqref{model}. We use "$\approx$" since in practice we only do finite-step optimization algorithm to approximate the optimizer 
\begin{equation}\label{model}
F_{\Theta}(\ybf) \approx \argmin_{\xbf} \big\{ \phi_{\Theta}(\xbf, \ybf) := f(\xbf, \ybf)   + R(\xbf; \Theta) \big\},
\end{equation}
where $f$ is the data fidelity term that usually takes the form $f(\xbf, \ybf)  = \frac{1}{2} \| \Pbf \F \xbf - \ybf \|^2$ in standard MRI setting, where $\F$ and $\Pbf$ represent the Fourier transform and the binary under-sampling mask for k-space trajectory respectively. 
However, we remark that our approach can be directly applied to a much broader class of image reconstruction problems as long as $f$ is continuously differentiable with Lipschitz continuous gradient.

In this section, we introduce a learned optimization algorithm (LOA) for solving \eqref{model}, where the network parameters $\Theta$ are learned and fixed. Since $\Theta$ are fixed in \eqref{model}, we will omit them in the derivation of the LOA below and write  $R(\xbf; \Theta)$ as $R(\xbf)$ for notation simplicity. Then we generate a multi-phase network induced by the proposed LOA, i.e. the network architecture follows the algorithm exactly such that one phase of the network is just one iteration of the LOA.

In our implementation, to incorporate sparsity along with the learned features, we parameterize the function $R (\xbf)= \kappa \cdot r(\xbf)$, where $\kappa>0$ is a weight parameter that needs be chosen properly depending on the specific task (e.g., noise level, undersampling ratio, etc.), and $r$ is a regularizer parameterized as a (composition of) neural networks and can be adapted to a broad range of imaging applications. In this work, we parameterize $r$ as the composition of the $l_{2,1}$ norm and a learnable feature extraction operator $\gbf_j(\xbf)$. 
That is, we set $r$ in \eqref{model} to be
\begin{equation}\label{eq:r}
r(\xbf) := \|\gbf_j(\xbf)\|_{2,1} = \sum_{j = 1}^{m} \|\gbf_{j}(\xbf)\|.
\end{equation}
Here $\gbf_j(\cdot)=\gbf_j(\cdot;\theta)$ is parametrized as a convolutional neural network (CNN) where $\theta$ is the learned and fixed network parameter in $r(\cdot;\theta)$ as mentioned above. We also consider $\kappa$ to be learned and fixed as $\theta$ for now, and will discuss how to learn both of them in the next subsection.
We use smooth activation function in $\gbf$ formulate in \eqref{eq:sigma} which renders $\gbf$ a smooth but nonconvex function. Due to the nonsmooth $\|\cdot\|_{2,1}$, $r$ is therefore a nonsmooth nonconvex function.

Since the minimization problem in \eqref{model} is nonconvex and nonsmooth, we need to derive an efficient LOA to solve it. This solver will be termed as $F_{\Theta}(\ybf)$.
Here we first consider smoothing the $ l_{2,1}$ norm that for any fixed $\gbf(\xbf)$:
\begin{equation}\label{eq:l21}
r_{\varepsilon} (\xbf) = \sum\nolimits^m_{j=1}  \sqrt{\| \gbf_{j} (\xbf) \|^2 + \varepsilon^2} -\varepsilon.
\end{equation}
We denote $R_{\varepsilon} = \kappa \cdot r_{\varepsilon}$.
%
%In \eqref{model} the network was introduced in an implicit way, in the following text we will apply some optimization algorithm to solve for \eqref{model} which will then contribute to an explicit multi-phase network that is dubbed as algorithmic unrolling network with a fixed number of phases. 
%
The LOA derived here is inspired by the proximal gradient descent algorithm and iterates the following steps to solve the smoothed problem:
\begin{subequations}\label{prox}
    \begin{align}
        \ztp & = \xt - \alpha_t \nabla f(\xt) \label{prox_u}\\
    \xtp & = \prox_{\alpha_t R_{\epst} } (\ztp ), \label{prox_sub}
    \end{align}
\end{subequations}
where $\epst$ denotes the smoothing parameter $\varepsilon$ at the specific iteration $t$, and proximal operator is defined as $ \prox_{\alpha g}(\bbf) := \argmin_{\xbf} \left\| \xbf-\bbf \right\| + \alpha g(\xbf)$ in \eqref{prox_sub}. A quick observation from \eqref{eq:l21} is that $R_{\varepsilon} \rightarrow R$ as $\varepsilon$ diminishes, so later we will intentionally push $\epst \rightarrow 0$ at Line 16 in Algorithm \ref{alg:lda}. Then one can readily show that $ R_{\varepsilon}(x) \le R(x) \le R_{\varepsilon}(x) + \varepsilon$ for all $x$ and $\varepsilon > 0$. From Algorithm \ref{alg:lda}, line 16  will automatically reduce $\varepsilon$ and the iterates will converge to the solution of the original nonsmooth nonconvex problem \eqref{model}, a rigorous sense will be made precisely in the convergence analysis in Appendix \ref{convergence}.

Since $R_{\epst}$ is a complex function involving a deep neural network, its proximal operator does not have close-form and cannot be computed easily in the subproblem in \eqref{prox_sub}.
To overcome this difficulty, we consider to approximate $R_{\epst}$ by
\begin{subequations}
\begin{align}
\hat{R}_{\epst} (\ztp) & = R_{\epst}(\ztp) + \nonumber \langle  \nabla R_{\epst}(\ztp), \xbf-\ztp \rangle + \frac{1}{\beta_t} \norm{\xbf-\ztp}^2. \label{eq:u} 
\end{align}
\end{subequations}
Then we update $ \utp  = \prox_{\alpha_t \hat{R}_{\epst}  } (\ztp )$ to replace \eqref{prox_sub}, therefore we obtain
\begin{equation}\label{ut+1}
    \utp = \ztp -  \taut \nabla R_{\epst}(\ztp), \text{ where } \taut = \frac{\alpha_t \beta_t}{\alpha_t + \beta_t}.
\end{equation}
In order to guarantee the convergence of the algorithm, we introduce the standard gradient descent of $\phi_{\epst} $ (where $\phi_{\epst}  := f + R_{\epst}$) at $ \xbf$:
\begin{equation}\label{grad_dst}
    \vtp = \argmin_{\xbf} \langle \nabla f(\xt), \xbf - \xt \rangle + \langle \nabla R_{\varepsilon}(\xt) , \xbf - \xt \rangle + \frac{1}{2 \alpha_t} \| \xbf - \xt \|^2,
\end{equation}
which yields 
\begin{equation}\label{vt+1}
    \vtp =\xt - \alpha_t  \nabla \phi_{\epst}(\xt) ,
\end{equation}
to serve as a safeguard for the convergence.
Specifically, we set $\xbf_{t+1} = \ubf_{t+1}$ if $\phi_{\epst}(\ubf_{t+1}) \le \phi_{\epst}(\vbf_{t+1})$; otherwise we set $\xbf_{t+1} = \vbf_{t+1}$. Then we repeat this process.

Our algorithm is summarized in Algorithm \ref{alg:lda}. 
%
%In this algorithm, we have two candidates: $\utp$ which arises from the linear approximation of proximal mapping, and $\vtp$ from standard gradient descent. The architecture of the unrolling sub-network for candidate $\utp$ can separate the updating schemes of the non-learnable data-fidelity term and the learnable prior term.
%
The prior term with unknown parameters has the exact residual update itself which makes the training process more fluent \cite{he2016deep}. The condition checking in Line 5 is introduced to make sure that it is in the energy descending direction. Once the condition in Line 5 fails, it comes to $\vtp$ and the line search in Line 12  guarantees the appropriate step size can be achieved within finite steps to make the function value decrease. From Line 3 to Line 14, we can regard that it solves a problem of minimizing $\phi_{\varepsilon_t}$ with $\epst$ fixed. Line 15 is to update the value of $\epst$ depending on a reduction criterion. The detailed analysis of this mechanism and in-depth convergence justification is shown in Appendix \ref{convergence}. The corresponding unrolling network exactly follows the Algorithm \ref{alg:lda}, thus shares the same convergence property. Compared to LDA \cite{chen2020learnable} computes both candidates $\utp$, $\vtp$ at every iteration and then chooses the one that achieves a smaller function value, we propose the criteria above in Line 5 for updating $\xtp$, 
which potentially
saves extra computational time for calculating the candidate $\vtp$ and potentially mitigates the frequent alternations between the two candidates. Besides, the smoothing method proposed in this work is more straightforward than smoothing in dual space \cite{chen2020learnable} whereas still keeping provable convergence as shown in Theorem \ref{theorem a6}.

The proposed LOA-induced network is a multi-phase network whose architecture exactly follows the proposed LOA \ref{alg:lda} in the way that each phase corresponding to one iteration of the algorithm with shared parameters. These unknown parameters are trained by the method in the next section.

\begin{algorithm}[t]
\caption{Algorithmic Unrolling Method with Provable Convergence}
\label{alg:lda}
\begin{algorithmic}[1]
\STATE \textbf{Input:} Initial $\xbf_0$, $0<\rho, \gamma<1$, and $\varepsilon_0$, $a, \sigma >0$. Max total phases $T$ or tolerance $\etol>0$.
\FOR{$t=0,1,2,\dots,T-1$}
\STATE $\ztp =  \xt - \alpha_{t} \nabla f(\xt)$
\STATE $\utp = \ztp - \taut \nabla R_{\epst} (\ztp)$, 
\IF{ $\| \nabla \phi_{\epst} (\xt) \| \leq a \| \utp - \xt \| \  \mbox{and}   \  \phi_{\epst}(\utp) - \phi_{\epst}(\xt) \leq - \frac{1}{a}\| \utp - \xt \|^2 $ holds} 
\STATE set $\xtp = \utp$,
\ELSE
\STATE $\vtp = \xt - \alpha_{t}  \nabla \phi_{\epst}(\xt)$, \label{marker}
\IF{ $ \phi_{\epst}(\vtp) - \phi_{\epst}(\xt) \le - \frac{1}{a} \| \vtp - \xt\|^2$ holds}
\STATE set $\xtp = \vtp$,
\ELSE
\STATE update $\alpha_{t} \leftarrow \rho \alpha_{t}$,
then \textbf{go to}~\ref{marker},
\ENDIF
\ENDIF
\STATE \textbf{if} $\|\nabla \phi_{\epst}(\xtp)\| < \sigma \gamma {\epst}$,  set $\epstp= \gamma {\epst}$;  \textbf{otherwise}, set $\epstp={\epst}$.
\STATE \textbf{if} $\sigma {\epst} < \etol$, terminate.
\ENDFOR
\STATE \textbf{Output:} $\xbf_T$.
\end{algorithmic}
\end{algorithm}
%%%%%%%%%%%%%%%%%%%%%
\subsection{Bilevel optimization algorithm for network training}
%In the previous section, we illustrate the forward network structure, which unfolds Algorithm \ref{alg:lda}. 
In this section, we consider the parameter training problem of the LOA-induced network. 
Specifically, we develop a bilevel optimization algorithm for training our network parameters $ \Theta$ from diverse data sets. 
%
%The bilevel optimization problem is formulated as \eqref{eq:bi-level}, and we propose a network training algorithm in Algorithm \ref{alg:model}.

As shown in Section \ref{network}, we design $R (\xbf;\Theta)= \kappa \cdot r(\xbf;\Theta)$, where $r$ is going to be learned to capture the intrinsic property of the underlying common features across all different tasks. 
To account for the large variations in the diverse training/validation data sets, we introduce a task-specific parameter $\omega_i$ to approximate the proper $\kappa$ for the $i$th task. Specifically, for the $i$th task, the weight $\kappa$ is set to $\sigma(\omega_i) \in (0, 1)$, where $\sigma(\cdot)$ is the sigmoid function. Therefore, $\kappa = \sigma(\omega_i)$ finds the proper weight of $r$ for the $i$-th task according to its specific sampling ratio or pattern.
The parameters $\omega_i$ are to be optimized in conjunction with $\Theta$ through the hyperparameter tuning process below.

Suppose that we are given $\mathcal{M}$ data pairs $\{(\ybf_m, \hat{\xbf}_m) \}_{m = 1} ^{\mathcal{M}}$ for the use of training and validation where $\ybf_m$ is the observation, which is partial k-space data in our setting, and $\hat{\xbf}_m$ is the corresponding ground truth image. The data pairs are then sampled into $\mathcal{N}$ tasks $\{ \D_{\tau_i} \}_{i = 1} ^ {\mathcal{N}}$, where each $\D_{\tau_i}$ represents the collection of data pairs in the specific task $\tau_i$.  In each task $\tau_i$, we further divide the data into the task-specific training set $\D^{tr}_{\tau_i}$ and validation set $\D^{val}_{\tau_i}$. Architecture of our base network
exactly follows the LOA \ref{alg:lda} developed in previous section with learnable parameters $\theta$ and a task-specific parameter $\omega_i$ for the $i$th task.
% whose detailed structure is illustrated in Section \ref{network}. 
More precisely, for one data sample denoted by $(\ybf^{(i)}_{j},\hat{\xbf}^{(i)}_{j})$ in task $\tau_i$ with index $j$, we propose the algorithmic unrolling network for task $\tau_i$ as
\begin{equation}\label{our_model}
F_{\theta,\omega_i}(\ybf^{(i)}_{j}) \approx \argmin_{\xbf} f(\xbf, \ybf^{(i)}_{j}) + \sigma(\omega_i) r(\xbf;\theta),    
\end{equation}
%\ye{the next sentence is not needed since it's already explained in the previous subsection. Just follow with "We define the task-specific loss..."}
$\theta$ denotes the learnable common parameters across different tasks with task-invariant representation whereas $\omega_i$ is a task-specific parameter for task $\tau_i$. Here $\omega_i \in \mathbb{R}$ and $\sigma$ is the sigmoid function such that $\sigma(\omega_i) \in (0,1)$. 
The weight $\sigma(\omega_i)$ represents the weight of $r$ associated with the specific task $\tau_i$.
We denote $\omega$ to be the set $\{\omega_i\} _ {i = 1} ^ {\mathcal{N}}$. The detailed architecture of this network is illustrated in Section \ref{network}. We define the task-specific loss
\begin{equation}\label{loss_sum}
\ell_{\tau_i}(\theta, \omega_i  ; \D_{\tau_i}) = \sum _ {j=1}^{ |\D_{\tau_i}|} \ell \big( F_{\theta,\omega_i}(\ybf^{(i)}_{j}), \hat{\xbf}^{(i)}_{j} \big),    
\end{equation}
%\ye{either $\sum_{j \in D}$ or $\sum_{j=1}^{|D|}$.}
where $|\D_{\tau_i}|$ represents the cardinality of $\D_{\tau_i}$ and 
\begin{equation}\label{loss}
  \ell \big( F_{\theta,\omega_i}(\ybf^{(i)}_{j}), \hat{\xbf}^{(i)}_{j} \big) := \frac{1}{2} \|F_{\theta,\omega_i}(\ybf^{(i)}_{j}) - \hat{\xbf}^{(i)}_{j}\|^2.
\end{equation}

For the sake of preventing the proposed model from overfitting the training data, we introduce a novel learning framework by formulating the network training as a bilevel optimization problem to learn $\omega$ and $\theta $ in \eqref{our_model} as
\begin{subequations}
\label{eq:bi-level}
\begin{align}
  \min_{ \theta, \ \omega = \omega_{i \in 1:\mathcal{N}}}  \quad & \sum^{\mathcal{N}}_{i=1} \ell _{\tau_i}( \theta (\omega) , \omega_i  ; \D^{val}_{\tau_i}) \\
 \mbox{s.t.}\quad \quad  & \theta(\omega)  = \argmin_{\theta} \sum^{\mathcal{N}}_{i=1} \ell_{\tau_i} ( \theta , \omega_i ; \D^{tr}_{\tau_i}).
\end{align}
\end{subequations}
In \eqref{eq:bi-level}, the lower-level optimization learns the task-invariant parameters $\theta$ of the feature encoder with fixed task-specific parameter $\omega_i$ on the training dataset, and the upper-level adjusts the task-specific parameters $\{\omega_i\}$ so that the task-invariant parameters $\theta$ can perform robustly on validation dataset as well. 
For simplicity, we omit the summation and redefine $\mathcal{L}(\theta, \omega ; \D) := \sum^{\mathcal{N}}_{i=1}\ell _{\tau_i}(\theta, \omega ; \D)$, then briefly rewrite \eqref{eq:bi-level} as 
\begin{equation}
  \min_{ \theta, \omega} \mathcal{L}( \theta(\omega), \omega ; \D^{val}) \ \ \ \ \ \mbox{s.t.} \ \ \theta(\omega) =   \argmin_{\theta}\mathcal{L}( \theta, \omega ; \D^{tr}).
  \label{simplified bi-level}
\end{equation}
Then we relax \eqref{simplified bi-level} into a single-level constrained optimization where the lower-level problem is replaced with its first-order necessary condition following \cite{mehra2019penalty}:
\begin{equation}
  \min_{ \theta, \omega} \mathcal{L}( \theta, \omega ; \D^{val}) \ \ \ \ \ \mbox{s.t.} \ \ \nabla_{\theta} \mathcal{L}( \theta, \omega ; \D^{tr}) = 0.
  \label{simplified bi-level-1}
\end{equation}
which can be further approximated by an unconstrained problem by a penalty term as
\begin{equation}
  \min_{ \theta, \omega} \big\{ \widetilde{\mathcal{L}}( \theta, \omega ; \D^{tr}, \D^{val}) := \mathcal{L}( \theta, \omega ; \D^{val}) + \frac{\lambda}{2} \| \nabla_{\theta} \mathcal{L}( \theta, \omega ; \D^{tr}) \|^2 \big\}.
  \label{simplified bi-level-2}
\end{equation}
%With the increasing number of training pairs, an modern machine learning model has to be trained in mini-batch setting. In Section \ref{Comprehensive meta training}, we will explain the mini-batch setup in detail.

%To tackle the increasing number of training samples, a modern machine learning model need to be trained in mini-batch setting. 
%
We adopt the stochastic gradients of the loss functions on mini-batch data sets in each iteration.
In our model, we need to include the data pairs of multiple tasks in one batch, therefore we propose the cross-task mini-batches when training. At each training iteration, we randomly sample the training data pairs $\mathcal{B}^{tr}_{\tau_i} = \{(\ybf^{(i)}_{j}, \hat{\xbf}^{(i)}_{j}) \in \D^{tr}_{\tau_i} \}_{j = 1}^{\mathcal{J}^{tr}}$ and the validation pairs $\mathcal{B}^{val}_{\tau_i} = \{(\ybf^{(i)}_{j}, \hat{\xbf}^{(i)}_{j}) \in \D^{val}_{\tau_i} \}_{j = 1}^{\mathcal{J}^{val}}$ on each task $\tau_i$. Then the overall training and validation mini-batchs $\mathcal{B}^{tr}$ and $\mathcal{B}^{val}$ used in every training iteration is composed of the sampled data pairs from the entire set of tasks, i.e. $\mathcal{B}^{tr} = \bigcup_{i = 1}^{\mathcal{N}} \{\mathcal{B}^{tr}_{\tau_i}\}$ and $\mathcal{B}^{val} = \bigcup_{i = 1}^{\mathcal{N}} \{\mathcal{B}^{val}_{\tau_i}\}$. Thus in each iteration, we have $\mathcal{N} \cdot \mathcal{J}^{tr}$ and $\mathcal{N}  \cdot \mathcal{J}^{val}$ data pairs used for training and validation respectively. To solve the minimization problem \eqref{simplified bi-level-2}, we utilize the stochastic mini-batch alternating direction method summarized in Algorithm \ref{penelty_method} which is modified from \cite{mehra2019penalty}.
In Algorithm \ref{penelty_method}, the parameter $\delta_{tol}$ is to control the accuracy, the algorithm will terminate when $\delta \le \delta_{tol}$. In addition, $\lambda$ is the weight for the second constraint term of \eqref{simplified bi-level-2}, in the beginning, we set $\lambda$ to be small so as to achieve a quick starting convergence then gradually increase its value to emphasize the constraint. 

\begin{algorithm}
\caption{Stochastic mini-batch alternating direction penalty method}\label{alg:model}
\begin{algorithmic}[1]
\STATE \textbf{Initialize}  $ \theta$, $ {\omega}$, $\delta$, $\lambda$ and $\nu_\delta \in(0, 1)$, \ $\nu_\lambda > 1$.
\WHILE{$\delta > \delta_{tol}$}
\STATE Sample cross-task training batch $\mathcal{B}^{tr} = \bigcup_{i = 1}^{\mathcal{N}} \{(\ybf^{(i)}_{j}, \hat{\xbf}^{(i)}_{j}) \in \D^{tr}_{\tau_i} \}_{j = 1 : \mathcal{J}^{tr}}$
\STATE Sample cross-task validation batch $\mathcal{B}^{val} = \bigcup_{i = 1}^{\mathcal{N}} \{(\ybf^{(i)}_{j}, \hat{\xbf}^{(i)}_{j}) \in \D^{val}_{\tau_i} \}_{j = 1 : \mathcal{J}^{val}}$
\WHILE{$\|\nabla_{\theta}\widetilde{\mathcal{L}}( \theta, \omega ; \mathcal{B}^{tr}, \mathcal{B}^{val})\|^2 + \| \nabla_{\omega}\widetilde{\mathcal{L}}( \theta, \omega ; \mathcal{B}^{tr}, \mathcal{B}^{val})\| ^2 > \delta$}
\FOR{$k=1,2,\dots,K$ (inner loop)}
%\STATE  Sample batches of tasks $ \tau_i \sim p(\tau)$
\STATE $ \theta \leftarrow \theta - \rho_{\theta}^k \nabla_{\theta}\widetilde{\mathcal{L}}( \theta, \omega ; \mathcal{B}^{tr}, \mathcal{B}^{val})$
\ENDFOR
\STATE $ \omega \leftarrow \omega - \rho_{\omega} \nabla_{\omega}\widetilde{\mathcal{L}}( \theta, \omega ; \mathcal{B}^{tr}, \mathcal{B}^{val})$
\ENDWHILE
\STATE \textbf{update} $\delta \leftarrow \nu_\delta \delta$, $\ \lambda \leftarrow \nu_\lambda \lambda$
\ENDWHILE
\STATE \textbf{output:} $\theta, {\omega}$.
\end{algorithmic}
\label{penelty_method}
\end{algorithm}

\section{Implementation}
\label{sec:Implementation}
\subsection{Feature extraction operator}
We set the feature extraction operator $\gbf$ to a vanilla $l$-layer CNN with the componentwise nonlinear activation function $\varphi$ and no bias, as follows: 
\begin{equation}\label{eq:g}
  \gbf(x) = \wbf_l * \varphi \cdots \ \varphi ( \wbf_3 * \varphi ( \wbf _2 * \varphi ( \wbf _1 * x ))),
\end{equation}
where $\{\wbf _q \}_{q = 1}^{l}$ denote the convolution weights consisting of $d$ kernels with identical spatial kernel size, and $*$ denotes the convolution operation. 
Here $\varphi$ is constructed to be the smoothed rectified linear unit as defined below
\begin{equation}\label{eq:sigma}
\varphi (x) = 
\begin{cases}
0, & \mbox{if} \ x \leq -\delta, \\
\frac{1}{4\delta} x^2 + \frac{1}{2} x + \frac{\delta}{4}, & \mbox{if} \ -\delta < x < \delta, \\
x, & \mbox{if} \ x \geq \delta,
\end{cases}
\end{equation}
where the prefixed parameter $\delta$ is set to be $0.001$ in our experiment.
The default configuration of the feature extraction operator is set as follows: the feature extraction operator $\gbf$ consists of $l=3$ convolution layers and all convolutions are with $4$ kernels of spatial size $3 \times 3$.

\subsection{Setups}
\label{Task-specific network pre-training}
As our method introduces an algorithmic unrolling network, there exists a one-to-one correspondence between the algorithm iterations and the neural network phases (or blocks). Each phase of the forward propagation can be viewed as one algorithm iteration, which motivates us to imitate the iterating of the optimization algorithm and use a stair training strategy \cite{chen2020learnable}. At the first stage, we start training the network parameters using 1 phase, then after the the loss converges, we add more phases (1 phase each time) then continue the training process. We repeat this procedure and stop it until the loss does not decrease any more when we add more blocks. We minimize the loss for $100$ epochs/iterations each time using the SGD-based optimizer Adam \cite{kingma2014adam} with $\beta_1 = 0.9$, $\beta_2 = 0.999$ and initial learning rate set to $10^{-3}$ as well as the mini-batch size $8$. The Xavier Initializer \cite{Glorot10understandingthe} is used to initialize the weights of all convolutions. The initial smoothing parameter $\varepsilon_0$ is set to be $0.001$ then learned together with other network parameters. The input $\xbf_0$ of the unrolling network is obtained by Zero-filling strategy\cite{bernstein2001effect}.  The deep unrolling network was implemented using the Tensorflow toolbox \cite{tensorflow2015-whitepaper} in Python programming language. Our code will be publicly shared depending on acceptance.

\section{Numerical Experiments}\label{experiments}
\subsection{Dataset}
To validate the performance of the proposed method, the data we used are from Multimodal Brain Tumor Segmentation Challenge 2018 \cite{menze2014multimodal}, in which the training dataset contains four modalities ($T1, T1_{c}, T2$ and $FLAIR$) scanned from 285 patients, and the validation dataset contains images from 66 patients, each with volume size $ 240 \times 240 \times 155$.  Each modality consists of two types of gliomas: 75 volumes of low-grade gliomas (LGG) and 210 volumes of high-grade gliomas (HGG). Our implementation interested in HGG MRI in two modalities: T1 and T2 images, and we choose 30 patients from each modality
in the training dataset for training our network. In the validation dataset, we randomly picked 15 patients as our validation data, and 6 patients in the training dataset as testing data which are are distinct from our training set and validation set.  We cropped the 2D image size to be $ 160 \times 180$ in the center region and picked adjacent $10$ slices in the center of each volume, which results in a total of $300$ images as our training data, $150$ images in our validation data, and a total of $60$ images as testing data. The number of data mentioned here is the amount of a single task, but since we employ multi-tasks training, the total number of images in each dataset should multiply the number of tasks. For each 2D slice, we normalize the spatial intensity by dividing the maximum pixel value.

\subsection{Experiment settings and Comparison results}

All the experiments are implemented on a Windows workstation with Intel Core i9 CPU at 3.3GHz and an Nvidia GTX-1080Ti GPU with 11GB of graphics card memory via TensorFlow \cite{abadi2016tensorflow}. The parameters in the proposed network are initialized by using Xavier initialization \cite{glorot2010understanding}.
We trained the meta-learning network with four tasks synergistically associated with four different CS ratios: 10\%, 20\%, 30\%, and 40\%, and test the well-trained model on the testing dataset with the same masks of these four ratios. We have 300 training data for each CS ratio, which amount to total of 1200 images in the training dataset. The results for $T1$ and $T2$ MR reconstructions are shown in Tables \ref{results_same_ratio_t1} and \ref{results_same_ratio_t2} respectively. The associated reconstructed images are displayed in Figures \ref{figure_same_ratio_t1} and \ref{figure_same_ratio_t2}.  We also test the well-trained meta-learning model on unseen tasks with radio masks for skewed ratios: 15\%, 25\%, 35\%, and random Cartesian masks with ratios 10\%, 20\%, 30\% and 40\%. The task-specific parameter for the unseen tasks are retrained for different masks with different sampling ratios individually with fixed task-invariant parameters $\theta$. In this experiments, we only need to learn $ \omega_i$ for three skewed CS ratios with radio mask and four regular CS ratios with Cartesian masks. The experimental training proceed on less data and iterations, where we performed on 100 MR images with 50 epochs. For example, for reconstructing MR images with CS ratio 15\% radio mask, we fix the parameter $\theta$ and retrain the task-specific parameter $\omega$ on 100 raw data with 50 epochs, then test with renewed $\omega$ on our testing data set with raw measurement that sampled from radio mask with CS ratio 15\%. The results associated with  radio masks are shown in Table \ref{results_dif_ratio_t1} and \ref{results_dif_ratio_t2}, Figure \ref{figure_dif_ratio_t1} and \ref{figure_dif_ratio_t2}  for $T1$ and $T2$ images respectively. The results associated with Cartesian masks are list in 
Table \ref{results_same_ratio_t2_cts} and reconstructed images are displayed in Figure  \ref{figure_same_ratio_t2_cts}.

We compared proposed meta-learning method with conventional supervised learning, which was trained with one task at each time, and only learns task-invariant parameter $\theta$ without task-specific parameter $ \omega_i$. The forward network of conventional learning  unrolls Algorithm \ref{alg:lda} with 11 phases which is the same as meta-learning. We merged training set and validation set which result in $450$ images for training the conventional supervised learning. The training batch size was set as 25 and applied total of 2000 epochs, while in meta-learning we applied 100 epochs with batch size 8.  Same testing set was used in both meta-learning and conventional learning to evaluate the performance these two methods.

We made comparisons between meta-learning and the conventional network on these seven different CS ratios (10\%, 20\%, 30\%, 40\%, 15\%, 25\%, and 35\%) in terms of two types of random under-sampling patterns: radio mask and Cartesian mask. The parameters for both meta-learning and conventional learning networks are trained via Adam optimizer \cite{kingma2014adam}, and they both learn the forward unrolled task-invariant parameter $\theta$.  Network training of conventional method uses the same network configuration as meta-learning network in terms of the number of convolutions, depth and size of CNN kernels, phase numbers and parameter initializer, etc.
The major difference in the training process between these two methods is that meta-learning is proceed on multi-tasks by leveraging task-specific parameter $\omega_i$ that learned from Algorithm \ref{alg:model} and the common feature among tasks are learned from the feed forward network that unrolls Algorithm \ref{alg:lda}, while conventional learning solve for task specific problem by simply unrolls forward network via Algorithm \ref{alg:lda} where both training and testing are implemented on the same task. To investigate the generalizability of meta-learning, we test the well-trained meta-learning model on MR images in different distributions in terms of two types of sampling masks with various trajectories. Training and testing of conventional learning are applied with the same CS ratios, that is, if the conventional method trained on CS ratio 10\% then will be tested on a dataset with CS ratio 10\%, etc. 

Because of the nature of MR images are represented as complex-value, we applied complex convolutions \cite{WANG2020136} for each CNN, that is, every kernel consists real part and imaginary part. Total of 11 phases are achieved if we set the termination condition $\etol = 1\times 10^{-5}$ and the parameters of each phase are shared except for the step sizes. Three convolutions are used in $\gbf$, each convolution contains 4 filters with spatial kernel size $ 3\times3$. We set $\nu_\delta =0.95 $ and the parameter $\delta$ in Algorithm \ref{alg:model} was initialized as $\delta_0 =1 \times 10^{-3}$ stopped at value $\delta_{tol} = 4.35 \times 10 ^ {-6}$, and we set total of 100 epochs. For training conventional method we set 2000 epochs with same number of phase, convolutions and kernel sizes as meta-learning. Initial $\lambda$ was set as $1 \times 10^{-5} $ and $ \nu_\lambda = 1.001$. 

We evaluate our reconstruction results on testing data sets using three metrics: PSNR, structural similarity (SSIM) \cite{wang2004image} and normalized mean squared error NMSE.
The following formulations compute the PSNR, SSIM and NMSE between reconstructed image $ \xbf$ and ground truth $\hat{\xbf}$. 
\begin{equation}
    PSNR = 20\log_{10} \big(  \max(\abs{\hat{\xbf}})  \big/ \frac{1}{N}\| \hat{\xbf} - \xbf \|^2 \big),
\end{equation}
where $N$ is the total number of pixels of ground truth.
\begin{equation}
    SSIM = \frac{(2\mu_{\xbf} \mu_{\hat{\xbf}} + C_1)(2\sigma_{\xbf \hat{\xbf}} + C_2)}{(\mu_{\xbf}^2 + \mu_{\hat{\xbf}}^2+C_1)( \sigma_{\xbf}^2 + \sigma_{\hat{\xbf}}^2 + C_2)},
\end{equation}
 where $\mu_{\xbf}, \mu_{\hat{\xbf}}$ represent local means, $\sigma_{\xbf}, \sigma_{\hat{\xbf}}$ denote standard deviations and $\sigma_{\xbf \hat{\xbf}} $ covariance between $\xbf$ and $\hat{\xbf} $, $ C_1 = (k_1 L)^2, C_2 = (k_2 L)^2$ are two constants which avoid zero denominator, and $ k_1 =0.01, k_2 =0.03$. $L$ is the largest pixel value of MR image. 
\begin{equation}
   NMSE(\xbf,\hat{\xbf})=  \frac{\|  \xbf-\hat{\xbf}\|_2^2}{\|  \xbf\|_2^2} .
\end{equation}

\subsection{Quantitative and Qualitative Comparisons at different trajectories in radio mask}

We evaluate the performance of well-trained Meta-learning and conventional learning. Table \ref{results_same_ratio_t1}, \ref{results_same_ratio_t2} and \ref{results_same_ratio_t2_cts} report the quantitative results of averaged numerical performance with standard deviations and associated descaled task specific meta-knowledge $ \sigma(\omega_i)$. From the experiments that implemented with radio masks, we observe that the average PSNR value of Meta-learning has improved 1.54 dB in T1 brain image among all the four CS ratios comparing to conventional method, and for T2 brain image, Meta-learning improved 1.46 dB in PSNR.  Since the general setting of Meta-learning aims to take advantage of the information provided from each individual task, each task associate with an individual sampling mask that may have complemented sampled points, the performance of the reconstruction from each task benefits from other tasks. Smaller CS ratios will inhibit the reconstruction accuracy, due to the sparse undersampled trajectory in raw measurement, while Meta-learning exhibits the favorable potential ability to solve this issue even in the situation of insufficient amount of training data. 

In general supervised learning,  training data need to be in the same or similar distribution, heterogeneous data exhibits different structure variations of features which hinders CNNs to extract features efficiently. In our experiments, raw measurements sampled from different ratios of compressed sensing display different levels of incompleteness, these undersampled measurements do not fall in the same distribution but they are related. Different sampling masks are shown at the bottom of Figure \ref{figure_same_ratio_t1} and \ref{figure_dif_ratio_t1} may have complemented sampled points, in the sense that some of the points which  $40\%$ sampling ratio mask does not sample have been captured by other masks. In our experiment, different sampling masks provide their own information from their sampled points so that four reconstruction tasks help each other to achieve an efficient performance. Therefore, it explains the reason that Meta-learning is still superior to conventional learning when the sampling ratio is large.

Meta-learning expands a new paradigm for supervised learning, the purpose is to quickly learn multiple tasks. Meta-learning only needs to learn task-invariant parameters one time for the common feature that can be shared with four different tasks, and each $ \sigma(\omega_i)$ provides the task-specific weighting parameters which give the guidance as "learning to learn". Comparing to conventional learning, which needs to be trained four times with four different masks since the task-invariant parameter cannot be generalized to other tasks, which is time-intensive. From Table \ref{results_same_ratio_t1} and \ref{results_same_ratio_t2}, we observe that small CS ratio needs higher value of $\sigma(\omega_i) $. In fact, in our model \eqref{model} the task-specific parameters behave as weighted constraints for task-specific regularizers, and the tables indicate that lower CS ratios require larger weights to apply on the regularization. 

Qualitative comparison between conventional and Meta-learning methods are shown in Figure \ref{figure_same_ratio_t1} and \ref{figure_same_ratio_t2}, which display the reconstructed MR images of the same slice for T1 and T2 respectively, we label the zoomed-in details of HGG in the red boxes. We observe the evidence that conventional learning is more blurry and lost sharp edges, especially in lower CS ratios. From the point-wise error map, we find meta-learning has the ability to reduce noises especially in some detailed and complicated regions comparing to conventional learning.

\subsection{Quantitative and Qualitative Comparisons at skewed trajectories in different sampling patterns}

In this section, we test the generalizability of the proposed model that tests on unseen tasks. We fix the well-trained task-invariant parameter $\theta$  and only train $\omega_i$ for sampling ratios 15\%, 25\% and 35\% with radio masks and sampling ratios 10\%, 20\%, 30\% and 40\% with Cartesian masks. In this experiment, we only used 100 training data for each CS ratio and apply a total of 50 epochs. The averaged evaluation values and standard deviations are listed in Table \ref{results_dif_ratio_t1} and \ref{results_dif_ratio_t2} for reconstructed T1 and T2 brain images respectively that proceed with radio masks, and Table \ref{results_same_ratio_t2_cts} shows the qualitative performance for reconstructed T2 brain image that applied random Cartesian sampling masks.  In T1 image reconstruction results, meta-learning improved 1.6921 dB in PSNR for 15\% CS ratio, 1.6608 dB for 25\% CS ratio, and 0.5764 dB for 35\% comparing to the conventional method,  which in the tendency that the level of reconstruction quality for lower CS ratios improved more than higher CS ratios. A similar trend happens in T2 reconstruction results with different sampling masks. The qualitative comparisons are illustrated in Figure \ref{figure_dif_ratio_t1}, \ref{figure_dif_ratio_t2} and \ref{figure_same_ratio_t2_cts} for T1 and T2 images tested in skewed CS ratios in radio masks, and T2 images tested in Cartesian masks with regular CS ratios respectively. 
In the experiments that conducted with radio masks,
meta-learning is superior to conventional learning especially at CS ratio 15\%, one can observe that the detailed region in red boxes keeps edges and is more close to the true image, while conventional method reconstructions are hazier and lost details in some complicated tissue. The point-wise error map also indicates that Meta-learning has the ability to suppress noises.

Training with Cartesian masks is more difficult than radio masks, especially for conventional learning where the network is not very deep since the network only applied three convolutions each with four kernels. Table \ref{results_same_ratio_t2_cts} indicates that the average performance of meta-learning improved about 1.87 dB comparing to conventional methods with T2 brain images. These results further demonstrate  that meta-learning has the benefit of parameter efficiency, the performance is much better than conventional learning even if we apply a shallow network with small size of training data.

The numeric experimental results discussed above intimates that Meta-learning is capable of fast adaption to new tasks and has more robust generalizability for a broad range of tasks with heterogeneous diverse data. Meta-learning can be considered as an efficient technique for solving difficult tasks by leveraging the features extracted from easier tasks.

\begin{specialtable}[htb]
\caption{ Quantitative evaluations of the reconstructions on T1 brain image associated with various sampling ratios of \textbf{radial} masks.  \label{results_same_ratio_t1}}
\addtolength{\tabcolsep}{-2pt}
\begin{tabular}{cccccc}
\toprule
\textbf{CS Ratio} & \textbf{Methods} & \textbf{PSNR}	& \textbf{SSIM}	& \textbf{NMSE} & $ \sigma(\omega_i)$\\
\midrule
10\% & Conventional		& 21.7570 $\pm$ 1.0677  & 0.5550 $\pm$ 0.0412 &  0.0259 $\pm$ 0.0082  &  \\
   & Meta-learning		& 23.2672 $\pm$ 1.1229  & 0.6101 $\pm$ 0.0436 & 0.0184 $\pm$ 0.0067 & 0.9218\\
   \midrule
20\% & Conventional		& 26.6202 $\pm$ 1.1662  & 0.6821 $\pm$ 0.0397 &  0.0910 $\pm$ 0.0169 & \\ 
   & Meta-learning		&  28.2944 $\pm$ 1.2119  & 0.7640 $\pm$ 0.0377  &  0.0058 $\pm$ 0.0022 & 0.7756\\
   \midrule
30\% & Conventional		& 29.5034 $\pm$ 1.4446  & 0.7557 $\pm$ 0.0408 &  0.0657 $\pm$ 0.0143 & \\
   & Meta-learning		&  31.1417 $\pm$ 1.5866  & 0.8363 $\pm$ 0.0385  & 0.0031 $\pm$ 0.0014 & 0.6501\\
   \midrule
40\% & Conventional		& 31.4672 $\pm$ 1.6390  & 0.8111 $\pm$ 0.0422 &  0.0029 $\pm$ 0.0014 & \\
   & Meta-learning	&  32.8238 $\pm$ 1.7039  & 0.8659 $\pm$ 0.0370  & 0.0022 $\pm$ 0.0010 & 0.6447\\
\bottomrule
\end{tabular}
\end{specialtable}

\begin{specialtable}[htb]
\caption{Quantitative evaluations of the reconstructions on T1 brain image associated with various sampling ratios of \textbf{radial} masks. Meta-learning was trained with CS ratio 10\%, 20\%, 30\% and 40\% and test with skewed ratios 15\%, 25\% and 35\%. Conventional method performed regular training and testing on same CS ratios 15\%, 25\% and 35\%. \label{results_dif_ratio_t1}}
\addtolength{\tabcolsep}{-2pt}
\begin{tabular}{cccccc}
\toprule
\textbf{CS Ratio} & \textbf{Metods} & \textbf{PSNR}	& \textbf{SSIM}	& \textbf{NMSE} & $ \sigma(\omega_i)$ \\
\midrule
15\% & Conventional		& 24.6573 $\pm$ 1.0244  & 0.6339 $\pm$ 0.0382 &  0.1136 $\pm$ 0.0186  &\\
   & Meta-learning		& 26.3494 $\pm$ 1.0102  & 0.7088 $\pm$ 0.0352 & 0.0090 $\pm$ 0.0030 & 0.9429\\
   \midrule
25\% & Conventional		& 28.4156 $\pm$ 1.2361  & 0.7533 $\pm$ 0.0368 &  0.0741 $\pm$ 0.0141 & \\ 
   & Meta-learning		& 30.0764 $\pm$ 1.4645  & 0.8135 $\pm$ 0.0380 & 0.0040 $\pm$ 0.0017 & 0.8482\\
   \midrule
35\% & Conventional		& 31.5320 $\pm$ 1.5242  & 0.7923 $\pm$ 0.0420 &  0.0521 $\pm$ 0.0119  &\\
   & Meta-learning		& 32.1084 $\pm$ 1.6481  & 0.8553 $\pm$ 0.0379 &  0.0025 $\pm$ 0.0011 & 0.6552\\
\bottomrule
\end{tabular}
\end{specialtable}

\begin{specialtable}[H]
\caption{Quantitative evaluations of the reconstructions on T2 brain image associated with various sampling ratios of \textbf{radial} masks.  \label{results_same_ratio_t2}}
\addtolength{\tabcolsep}{-2pt}
\begin{tabular}{cccccc}
\toprule
\textbf{CS Ratio} & \textbf{Metods} & \textbf{PSNR}	& \textbf{SSIM}	& \textbf{NMSE} & $ \sigma(\omega_i)$\\
\midrule
10\% & Conventional	& 23.0706 $\pm$ 1.2469  & 0.5963 $\pm$ 0.0349 &  0.2158  $\pm$  0.0347  &\\
   & Meta-learning	& 24.0842 $\pm$ 1.3863 & 0.6187 $\pm$ 0.0380 & 0.0112 $\pm$ 0.0117  & 0.9013	\\
   \midrule
20\% & Conventional & 27.0437  $\pm$ 1.0613  & 0.6867 $\pm$ 0.0261 &  0.1364 $\pm$ 0.0213 & \\ 
   & Meta-learning	& 28.9118 $\pm$ 1.0717 & 0.7843 $\pm$ 0.0240	& 0.0122 $\pm$ 0.0030 & 0.8742\\
   \midrule
30\% & Conventional	 & 29.5533 $\pm$ 1.0927 & 0.7565 $\pm$ 0.0265 & 0.1023 $\pm$ 0.0166 & \\
   & Meta-learning	& 31.4096 $\pm$ 0.9814 & 0.8488 $\pm$ 0.0217 & 0.0069 $\pm$ 0.0019 & 0.8029\\
   \midrule
40\% & Conventional	& 32.0153 $\pm$ 0.9402 & 0.8139 $\pm$ 0.0238 & 0.0770 $\pm$ 0.0128 & \\
   & Meta-learning	& 33.1114 $\pm$ 1.0189 & 0.8802 $\pm$ 0.0210 & 0.0047 $\pm$ 0.0015 & 0.7151\\
\bottomrule
\end{tabular}
\end{specialtable}

\begin{specialtable}[H]
\caption{Quantitative evaluations of the reconstructions on T2 brain image associated with various sampling ratios of \textbf{radial} masks. Meta-learning was trained with CS ratio 10\%, 20\%, 30\% and 40\% and test with 15\%, 25\% and 35\%. Conventional method performed regular training and testing on same CS ratios 15\%, 25\% and 35\%. \label{results_dif_ratio_t2}}
\addtolength{\tabcolsep}{-2pt}
\begin{tabular}{cccccc}
\toprule
\textbf{CS Ratio} & \textbf{Metods} & \textbf{PSNR}	& \textbf{SSIM}	& \textbf{NMSE} & $ \sigma(\omega_i)$ \\
\midrule
15\% & Conventional		& 24.8921 $\pm$ 1.2356 & 0.6259 $\pm$ 0.0285 & 0.1749 $\pm$ 0.0280  & \\
   & Meta-learning		& 26.7031 $\pm$ 1.2553 & 0.7104 $\pm$ 0.0318 & 0.0205 $\pm$ 0.0052  & 0.9532\\
   \midrule
25\% & Conventional		& 29.0545 $\pm$ 1.1980 & 0.7945 $\pm$ 0.0292 & 0.1083 $\pm$ 0.0173  & \\ 
   & Meta-learning		& 30.0698 $\pm$ 0.9969 & 0.8164 $\pm$ 0.0235 &  0.0093 $\pm$ 0.0022 & 0.8595\\
   \midrule
35\% & Conventional		& 31.5201 $\pm$ 1.0021 & 0.7978 $\pm$ 0.0236 & 0.0815 $\pm$ 0.0129  & \\
   & Meta-learning		& 32.0683 $\pm$ 0.9204 & 0.8615 $\pm$ 0.0209 & 0.0059 $\pm$ 0.0014  & 0.7388\\
\bottomrule
\end{tabular}
\end{specialtable}

\begin{specialtable}[H]
\caption{Quantitative evaluations of the reconstructions on T2 brain image associated with various sampling ratios of random \textbf{Cartesian} masks.  \label{results_same_ratio_t2_cts}}
\addtolength{\tabcolsep}{-2pt}
\begin{tabular}{cccccc}
\toprule
\textbf{CS Ratio} & \textbf{Metods} & \textbf{PSNR}	& \textbf{SSIM}	& \textbf{NMSE} & $ \sigma(\omega_i)$\\
\midrule
10\% & Conventional	& 20.8867 $\pm$ 1.2999 & 0.5082 $\pm$ 0.0475 & 0.0796 $\pm$ 0.0242  &\\
   & Meta-learning	& 22.0434 $\pm$ 1.3555 & 0.6279 $\pm$ 0.0444 & 0.0611 $\pm$ 0.0188  & 0.9361	\\
   \midrule
20\% & Conventional & 22.7954 $\pm$ 1.2819 & 0.6057 $\pm$ 0.0412 & 0.0513 $\pm$ 0.0157 & \\ 
   & Meta-learning	& 24.7162 $\pm$ 1.3919 & 0.6971 $\pm$ 0.0380 & 0.0329 $\pm$ 0.0101 & 0.8320\\
   \midrule
30\% & Conventional	 & 24.2170 $\pm$ 1.2396 & 0.6537 $\pm$ 0.0360 & 0.0371 $\pm$ 0.0117 & \\
   & Meta-learning	&  26.4537 $\pm$ 1.3471 & 0.7353 $\pm$ 0.0340 & 0.0221 $\pm$ 0.0068 & 0.6771\\
   \midrule
40\% & Conventional	& 25.3668 $\pm$ 1.3279 & 0.6991 $\pm$ 0.0288 & 0.1657 $\pm$ 0.0265 & \\
   & Meta-learning	& 27.5367 $\pm$ 1.4107 & 0.7726 $\pm$ 0.0297 & 0.0171 $\pm$ 0.0050 & 0.6498\\
\bottomrule
\end{tabular}
\end{specialtable}

\begin{figure}[H]
\centering
\includegraphics[width=0.2\linewidth, angle=90]{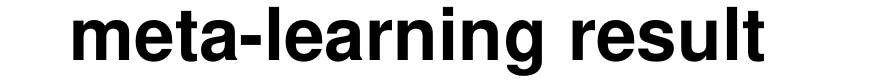}
\includegraphics[width=0.18\linewidth]{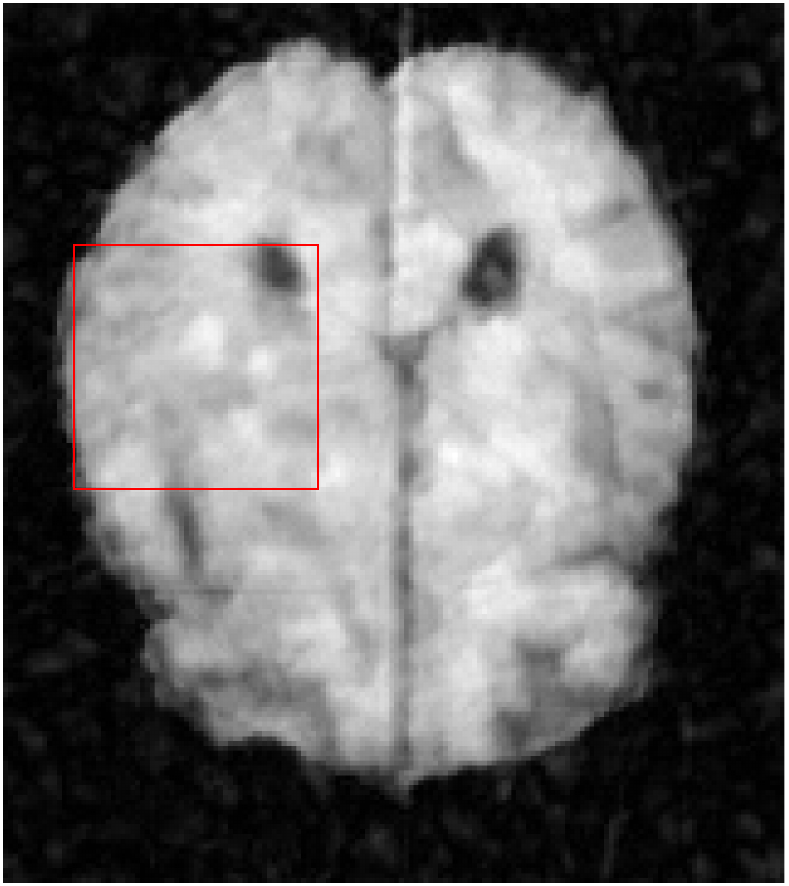}
\includegraphics[width=0.18\linewidth]{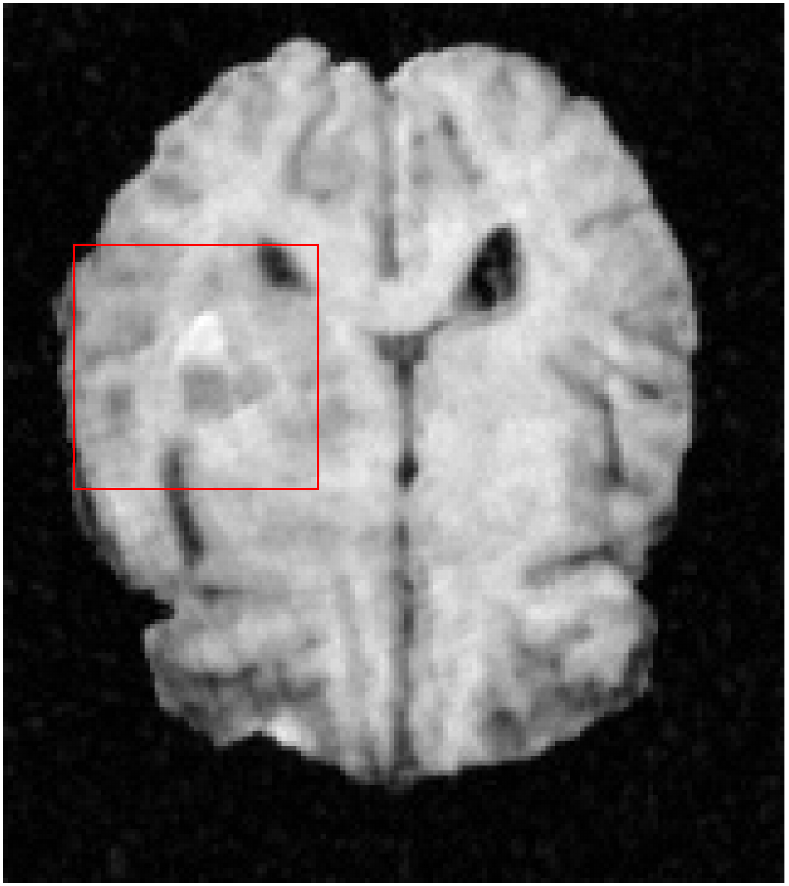}
\includegraphics[width=0.18\linewidth]{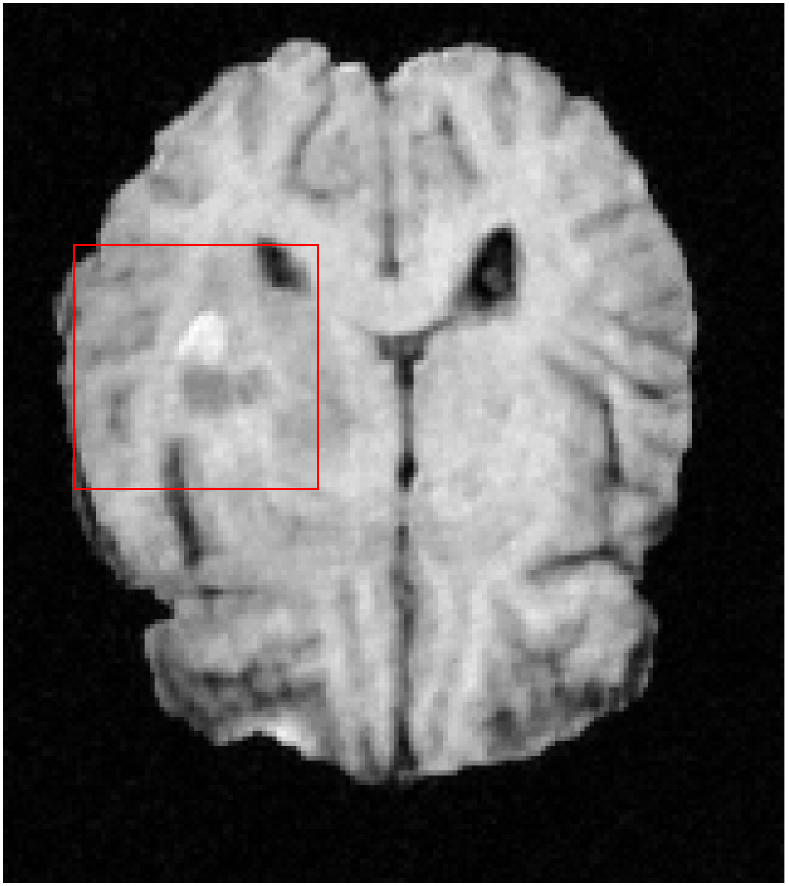}
\includegraphics[width=0.18\linewidth]{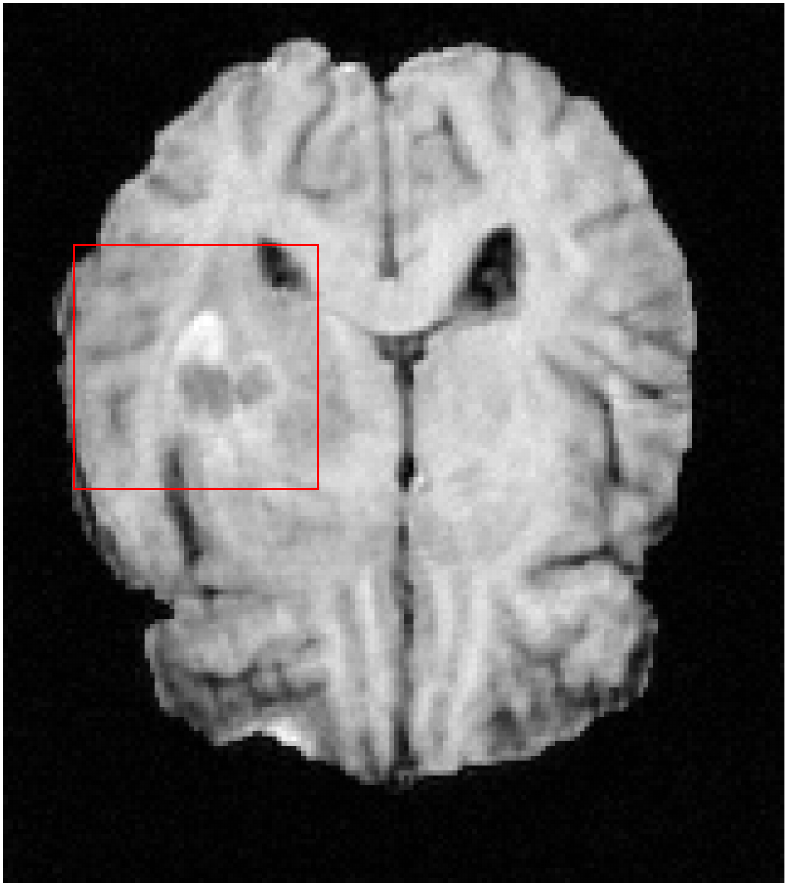}
\includegraphics[width=0.18\linewidth]{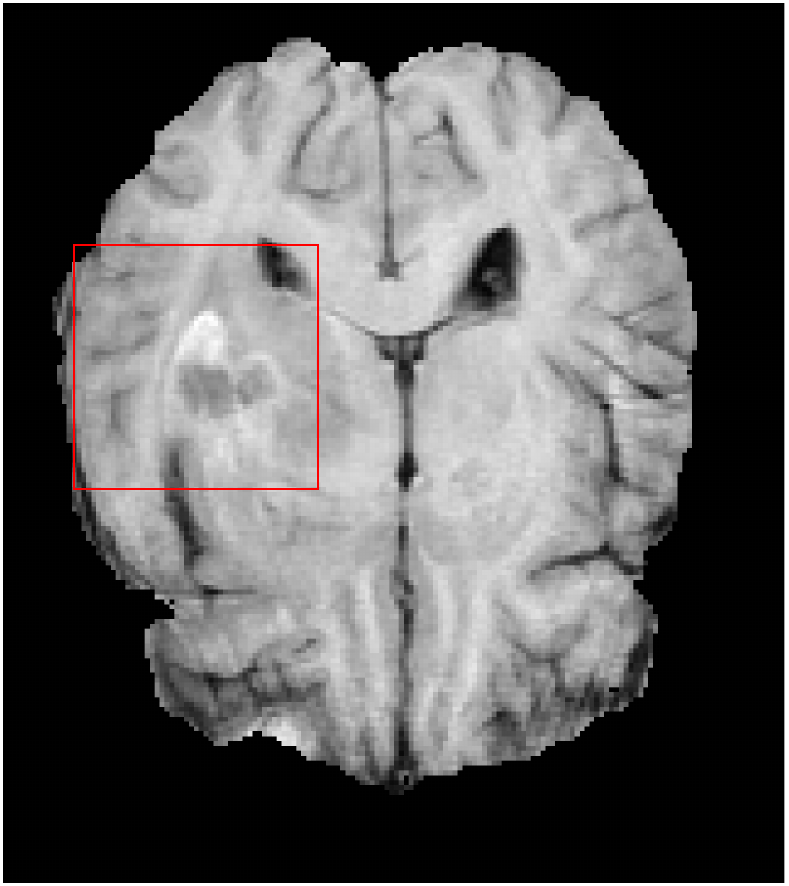}\\
\includegraphics[width=0.2\linewidth, angle=90]{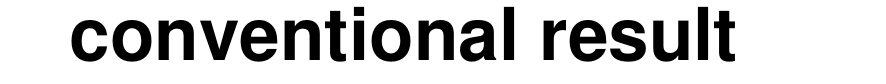}
\includegraphics[width=0.18\linewidth]{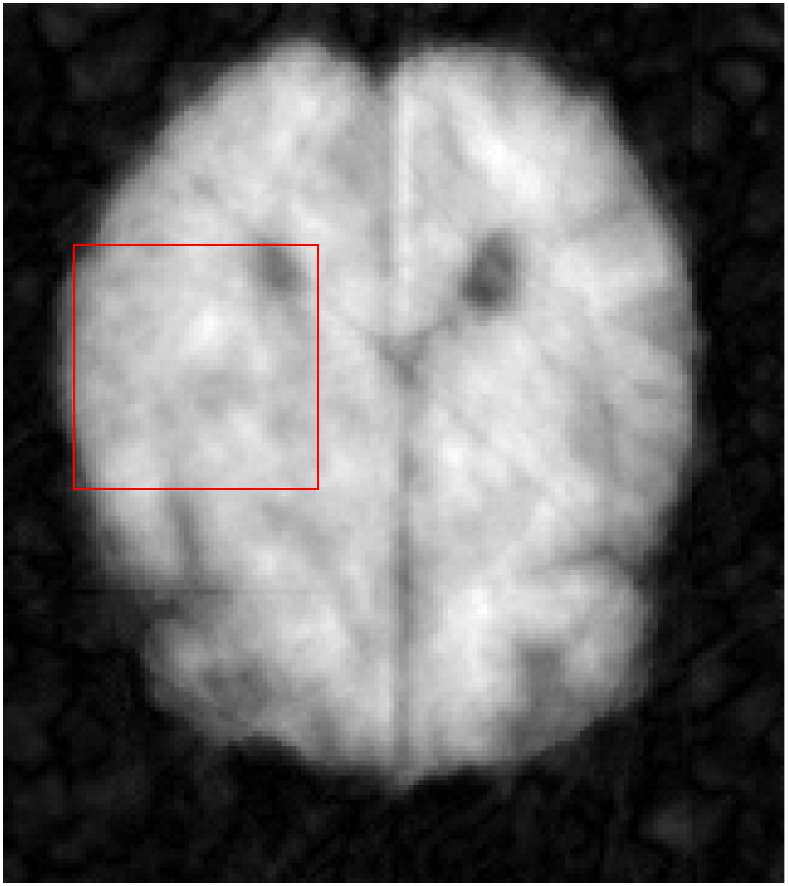}
\includegraphics[width=0.18\linewidth]{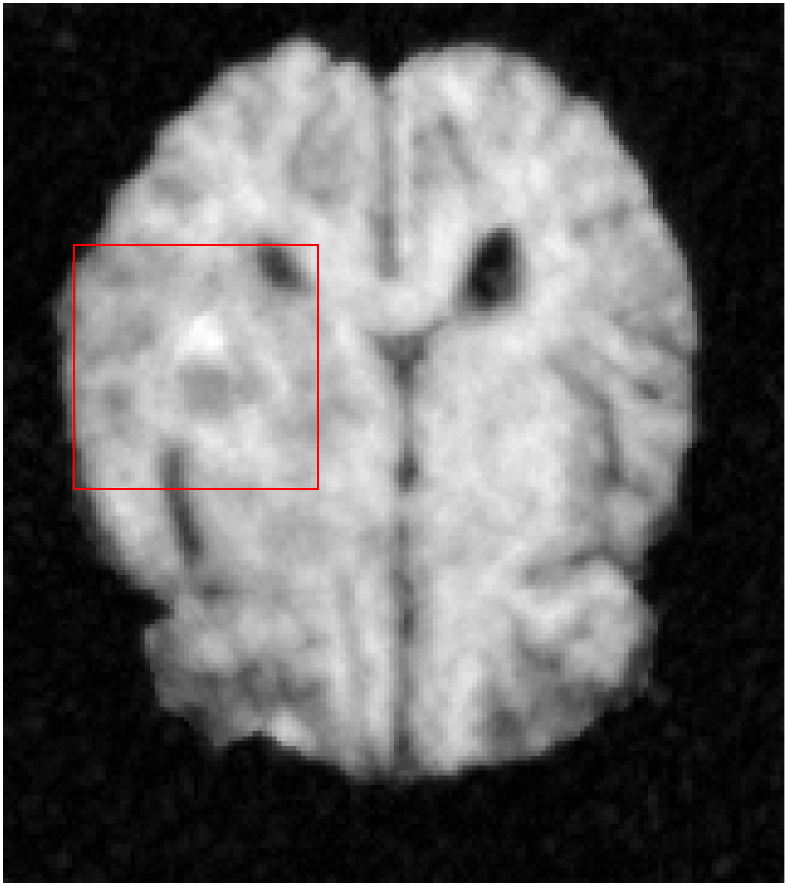}
\includegraphics[width=0.18\linewidth]{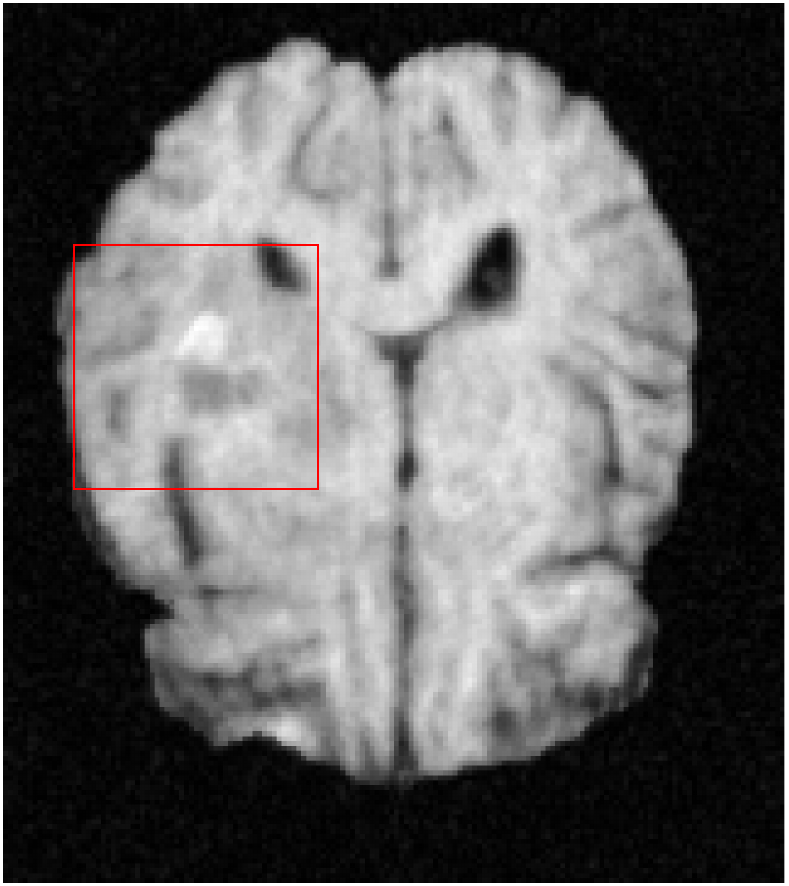}
\includegraphics[width=0.18\linewidth]{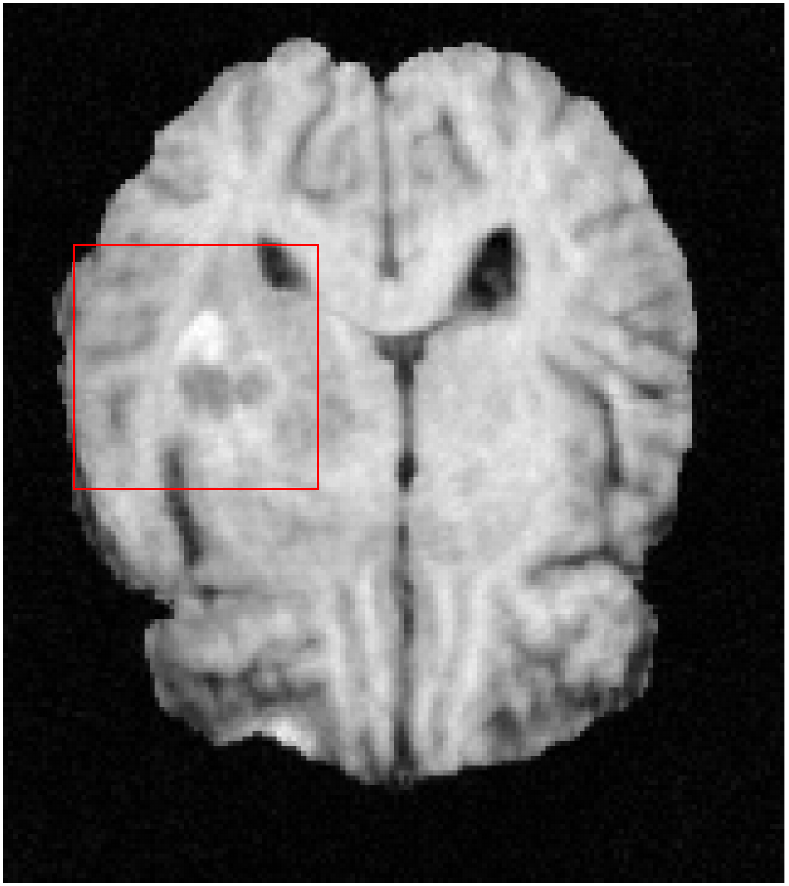}
\includegraphics[width=0.18\linewidth]{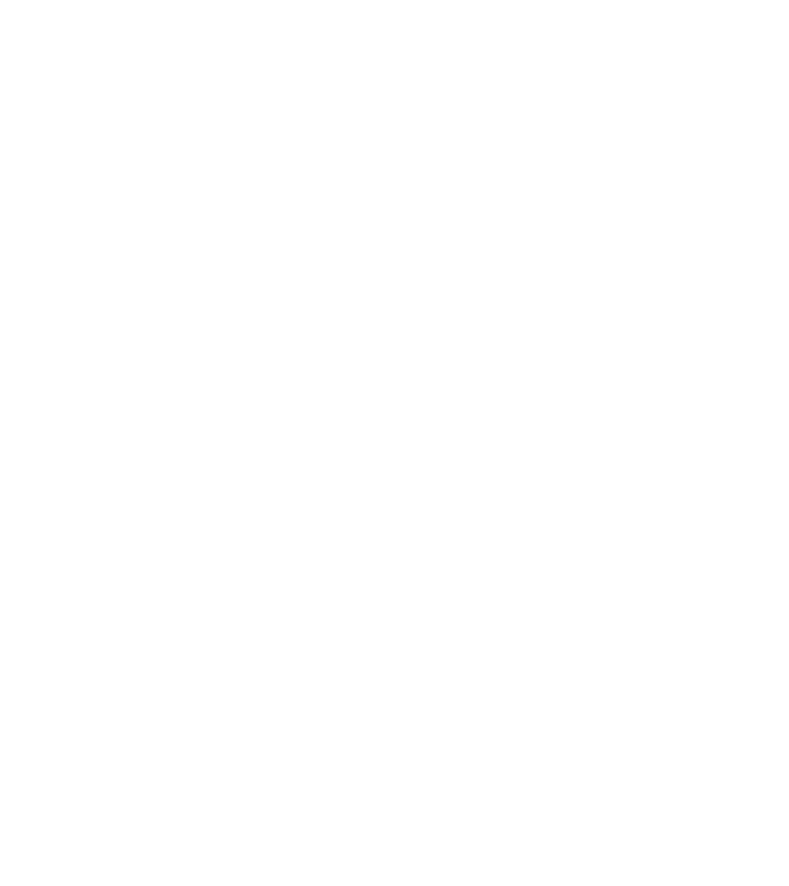}\\
\includegraphics[width=0.2\linewidth, angle=90]{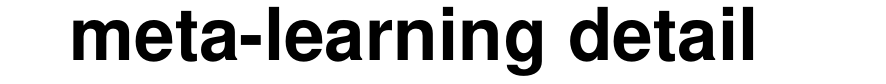}
\includegraphics[width=0.18\linewidth]{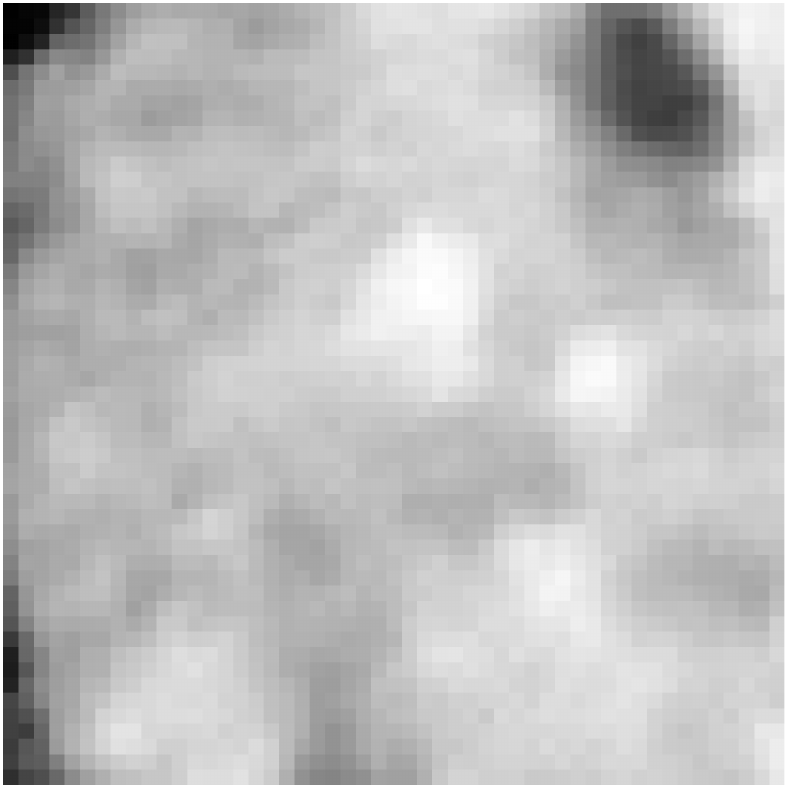}
\includegraphics[width=0.18\linewidth]{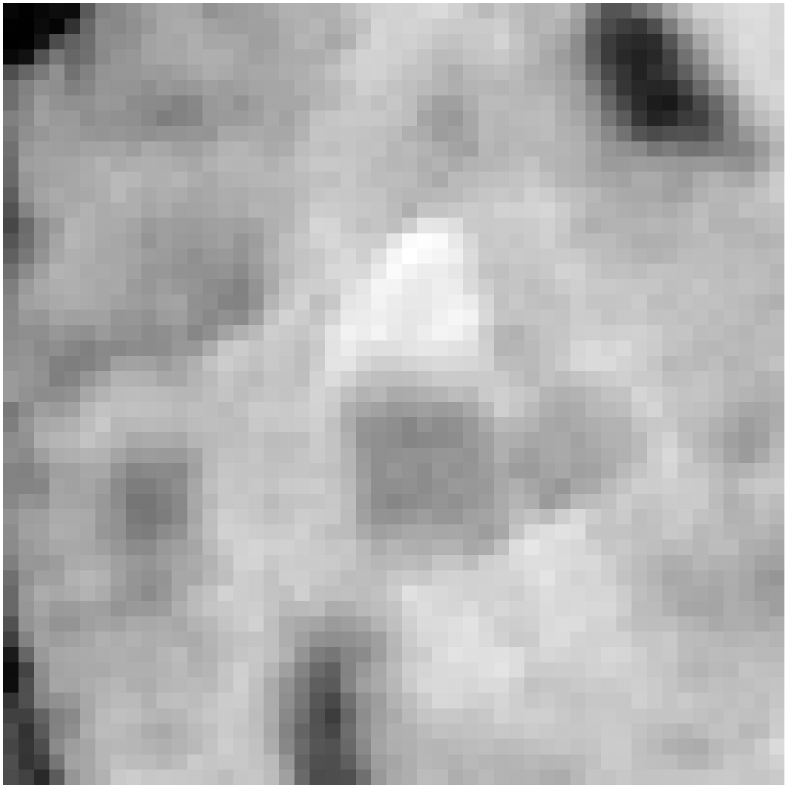}
\includegraphics[width=0.18\linewidth]{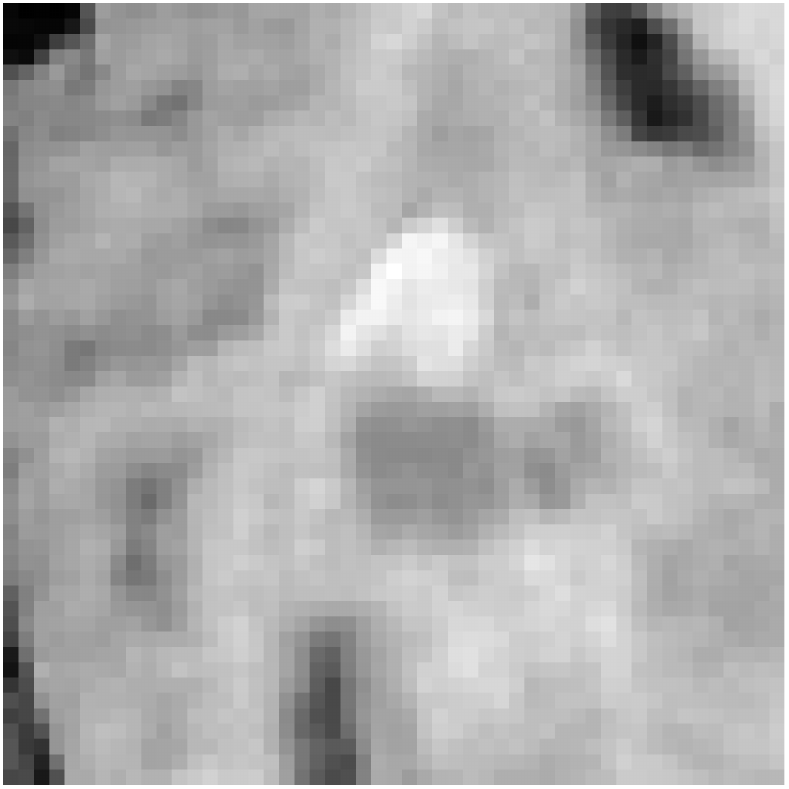}
\includegraphics[width=0.18\linewidth]{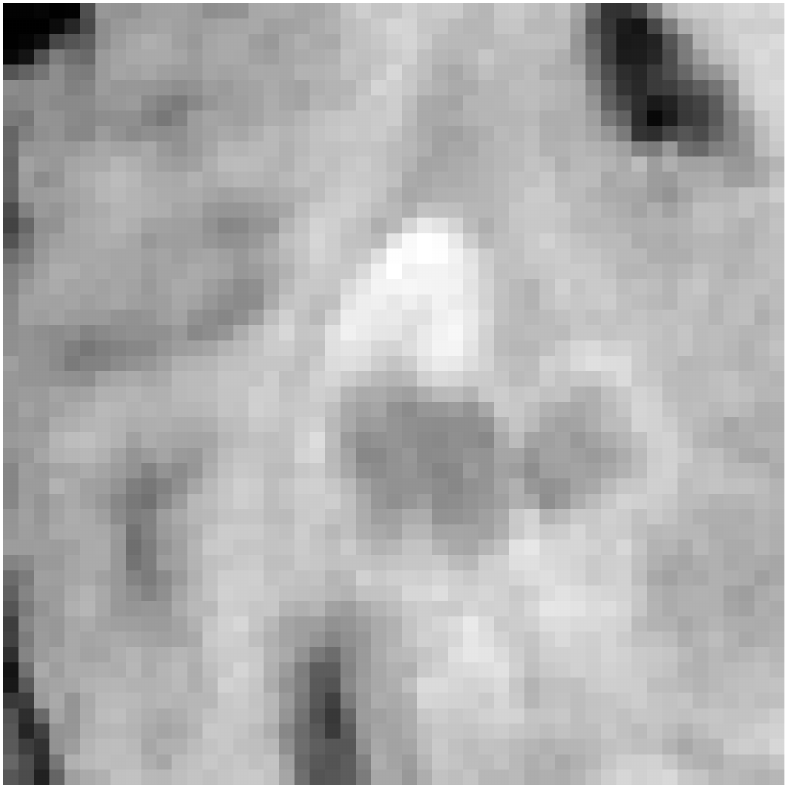}
\includegraphics[width=0.18\linewidth]{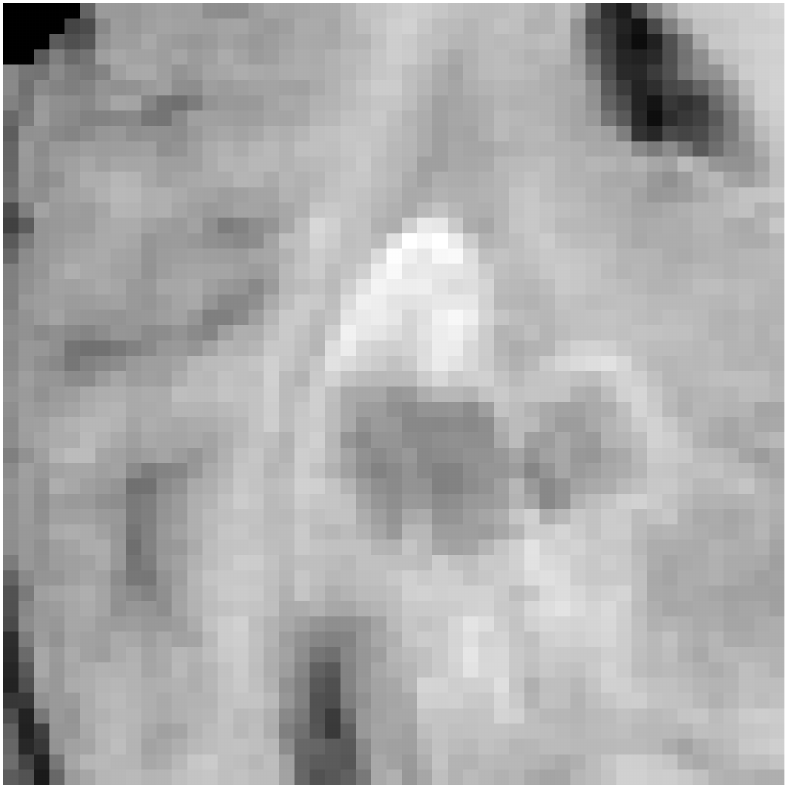}\\
\includegraphics[width=0.2\linewidth, angle=90]{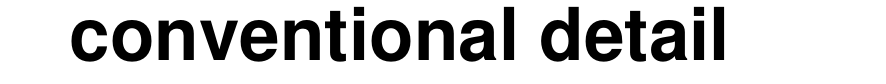}
\includegraphics[width=0.18\linewidth]{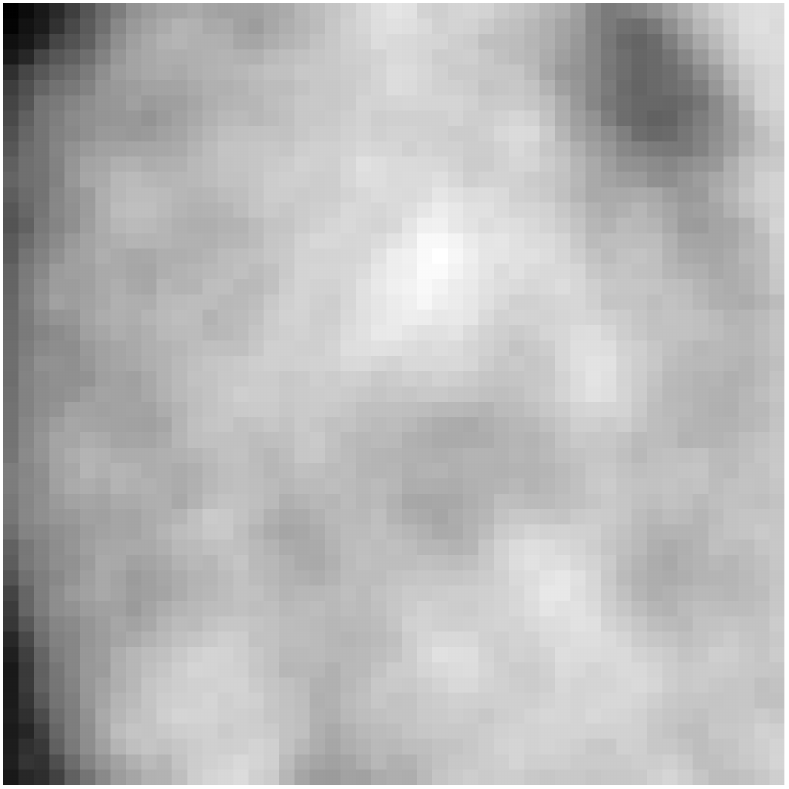}
\includegraphics[width=0.18\linewidth]{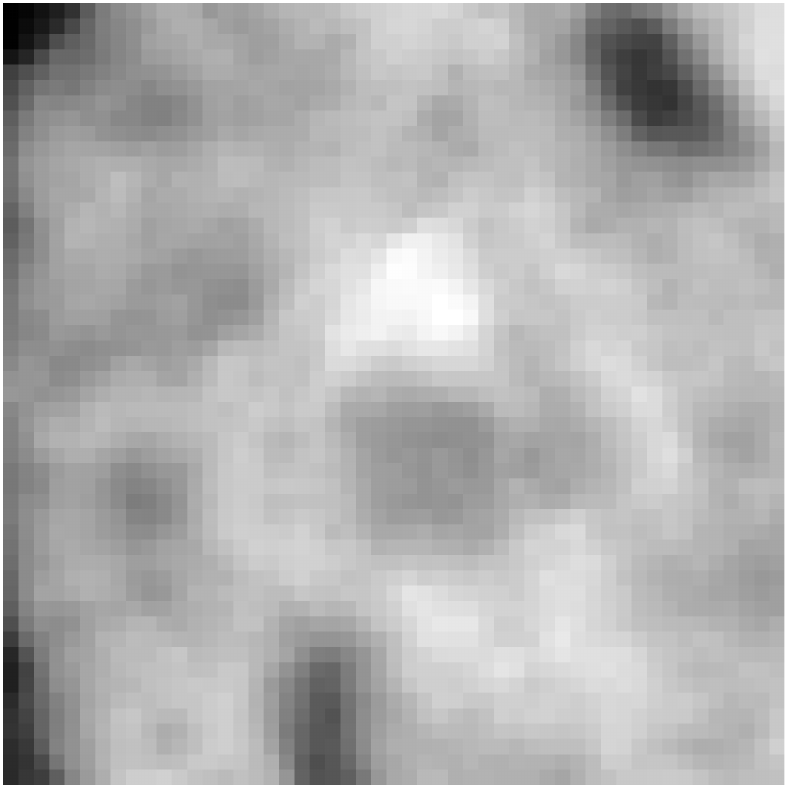}
\includegraphics[width=0.18\linewidth]{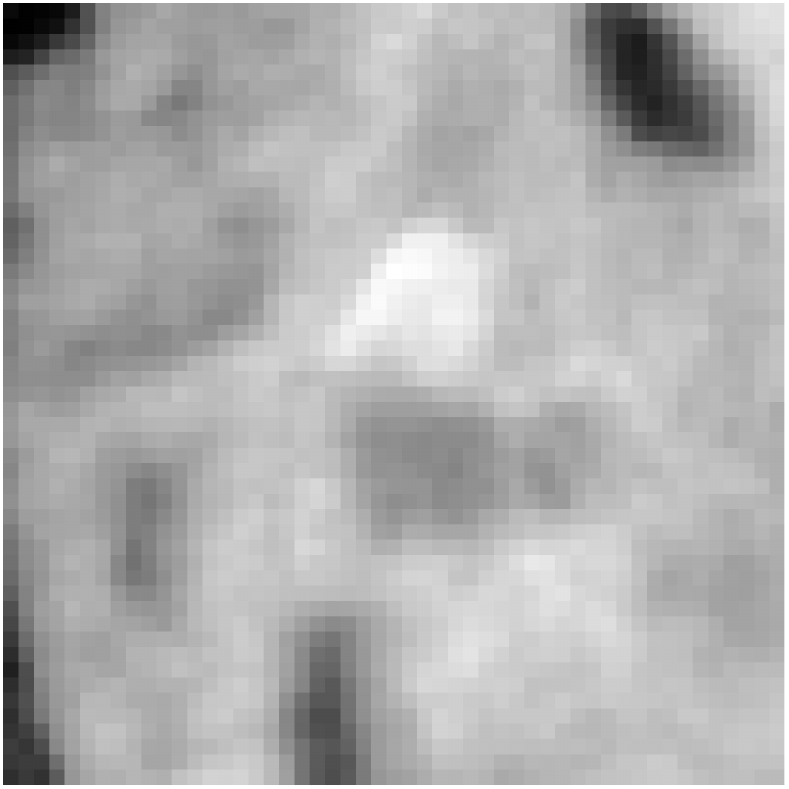}
\includegraphics[width=0.18\linewidth]{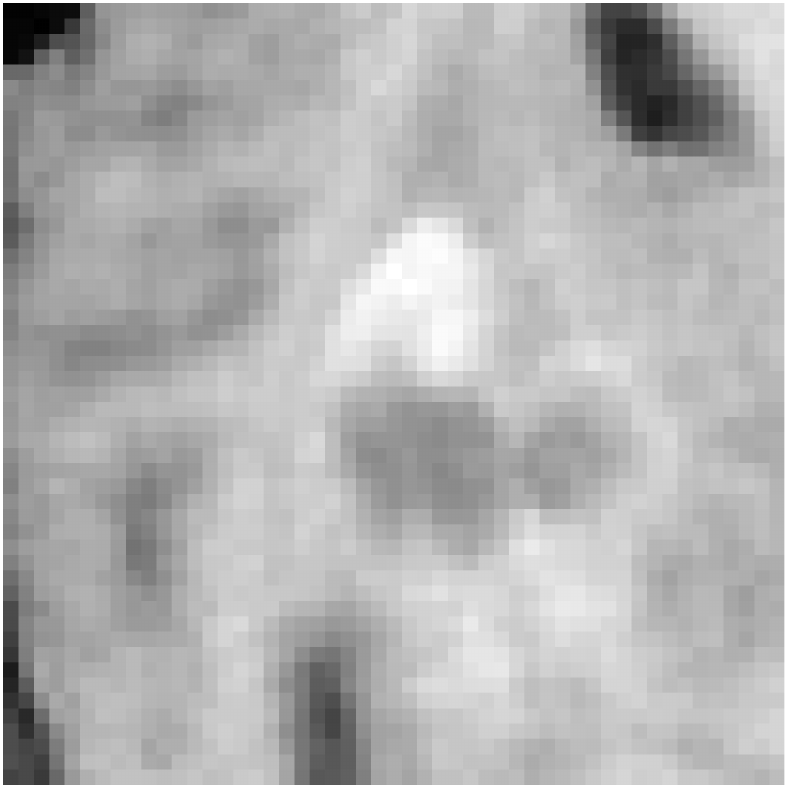}
\includegraphics[width=0.18\linewidth]{fig/white.pdf}\\
\includegraphics[width=0.2\linewidth, angle=90]{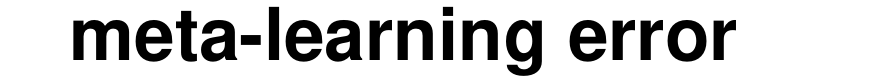}
\includegraphics[width=0.18\linewidth]{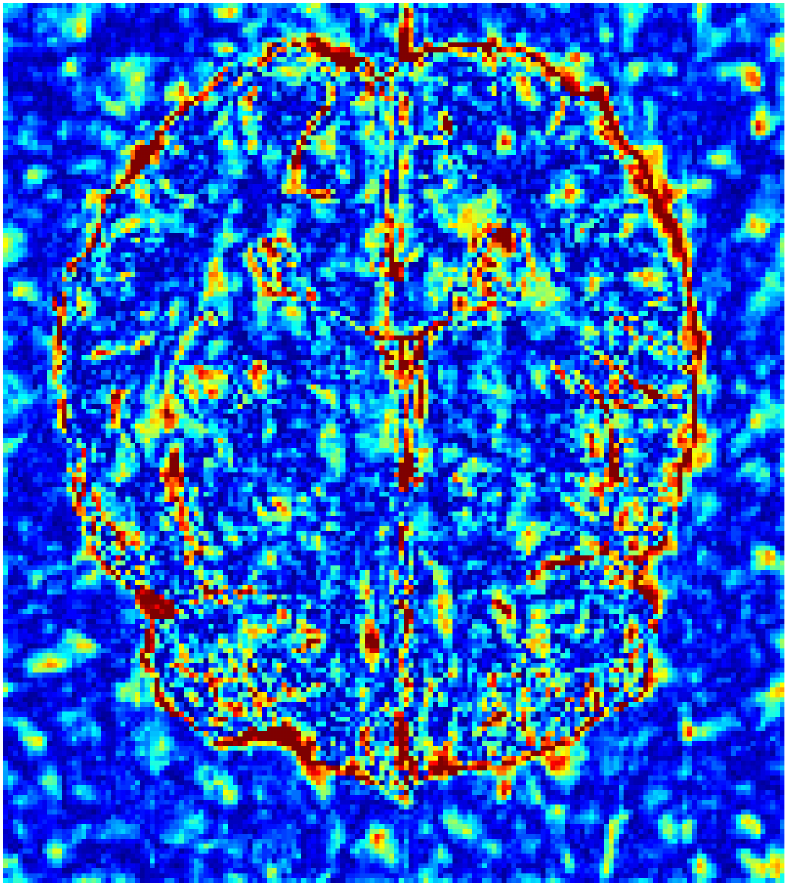}
\includegraphics[width=0.18\linewidth]{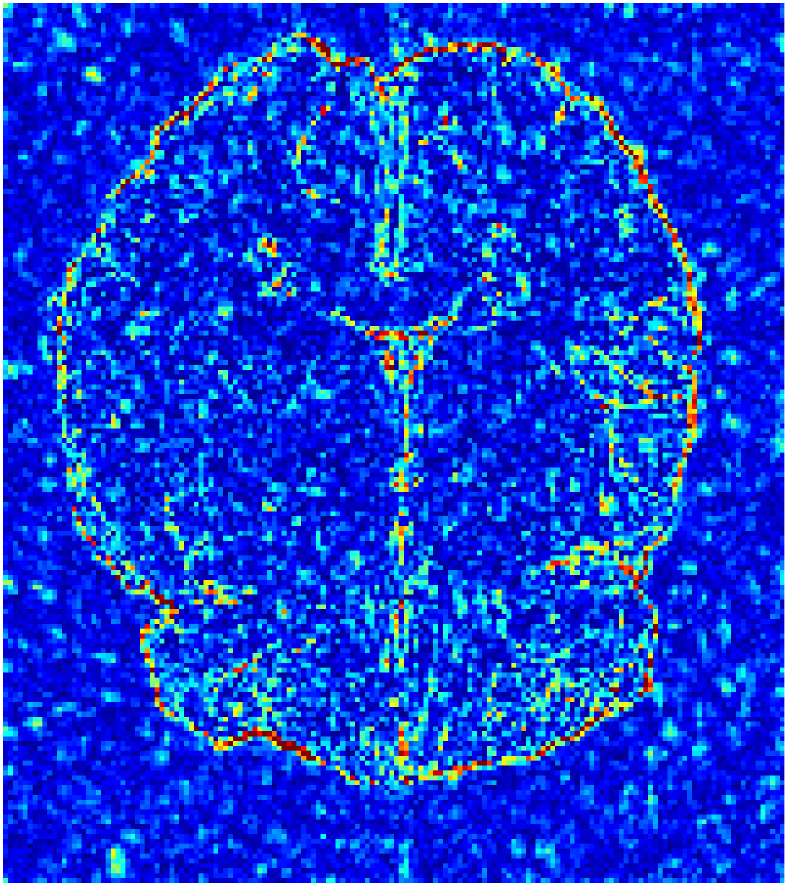}
\includegraphics[width=0.18\linewidth]{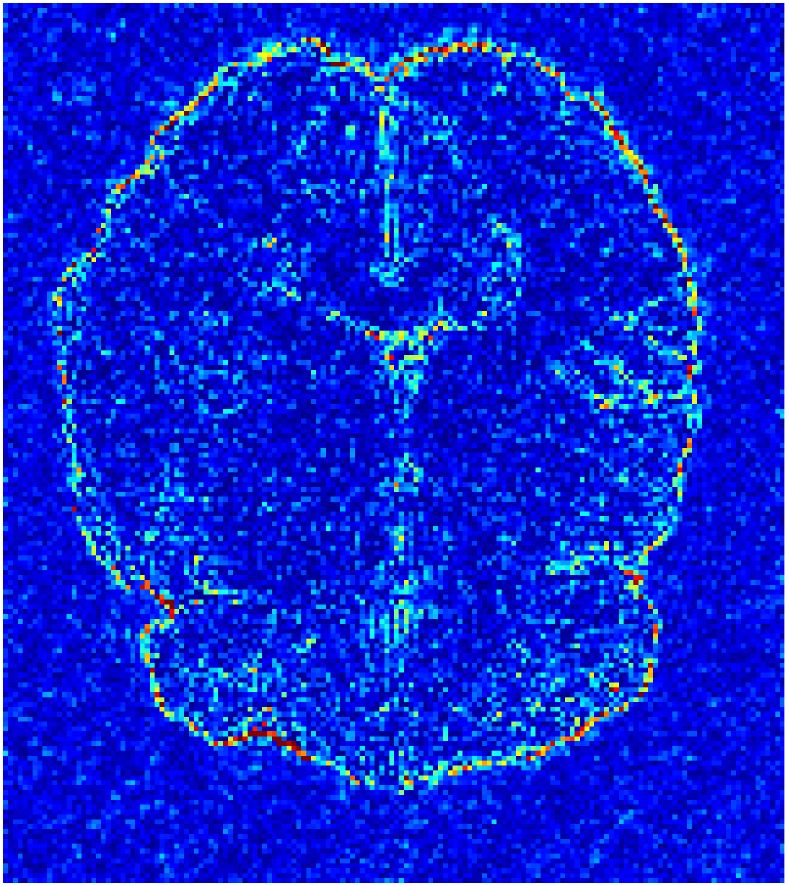}
\includegraphics[width=0.18\linewidth]{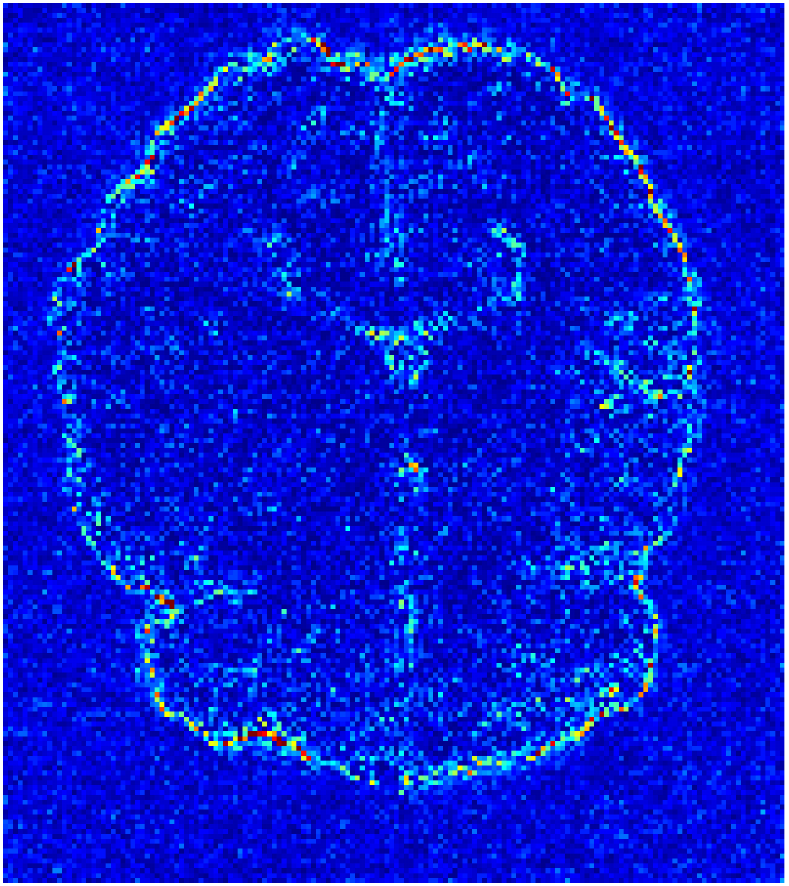}
\includegraphics[width=0.18\linewidth]{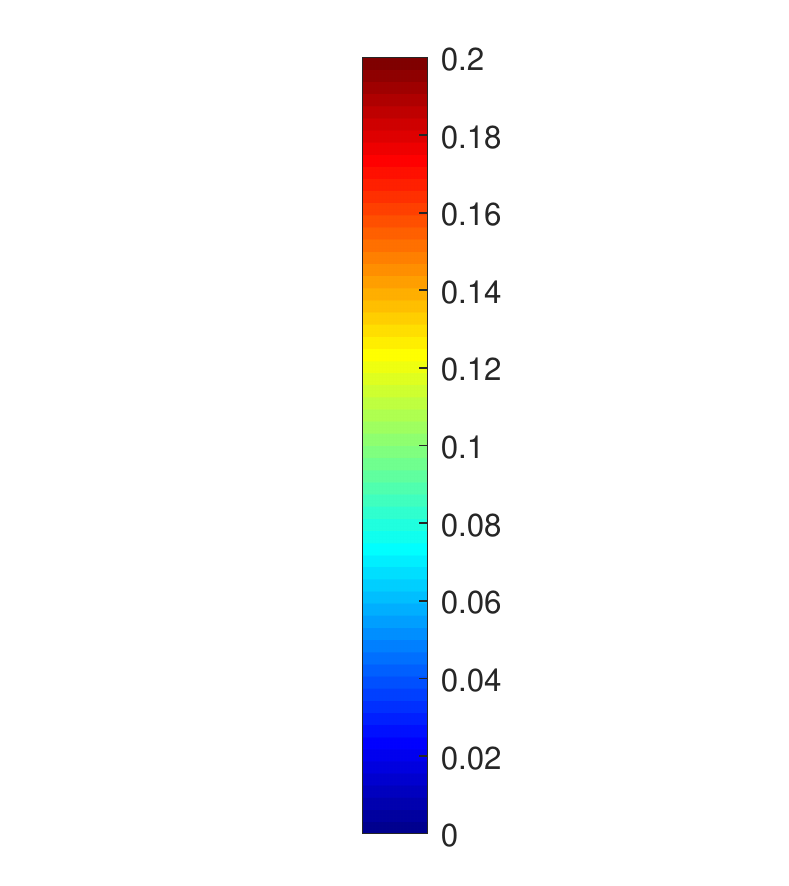}\\
\includegraphics[width=0.2\linewidth, angle=90]{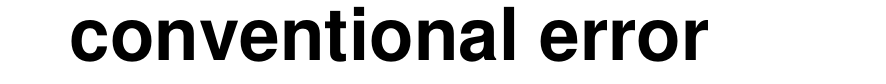}
\includegraphics[width=0.18\linewidth]{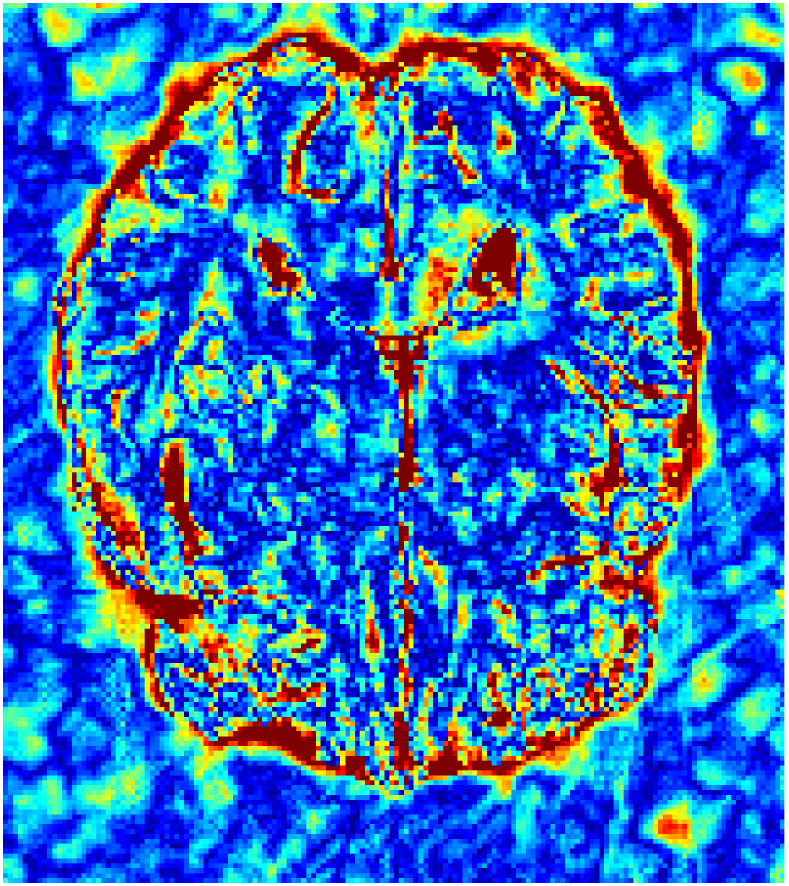}
\includegraphics[width=0.18\linewidth]{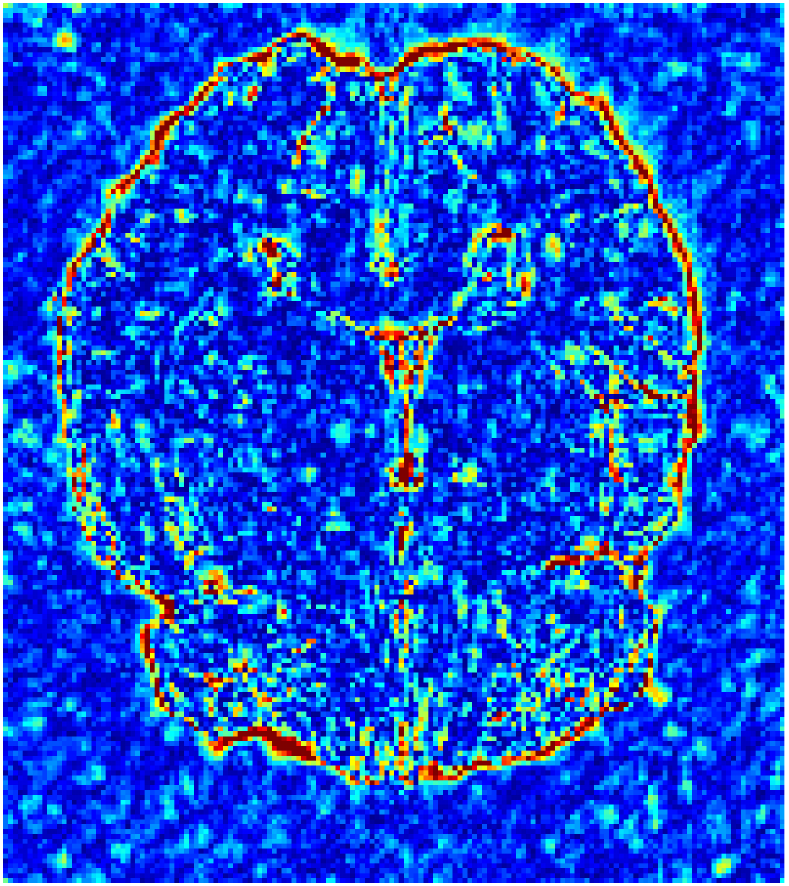}
\includegraphics[width=0.18\linewidth]{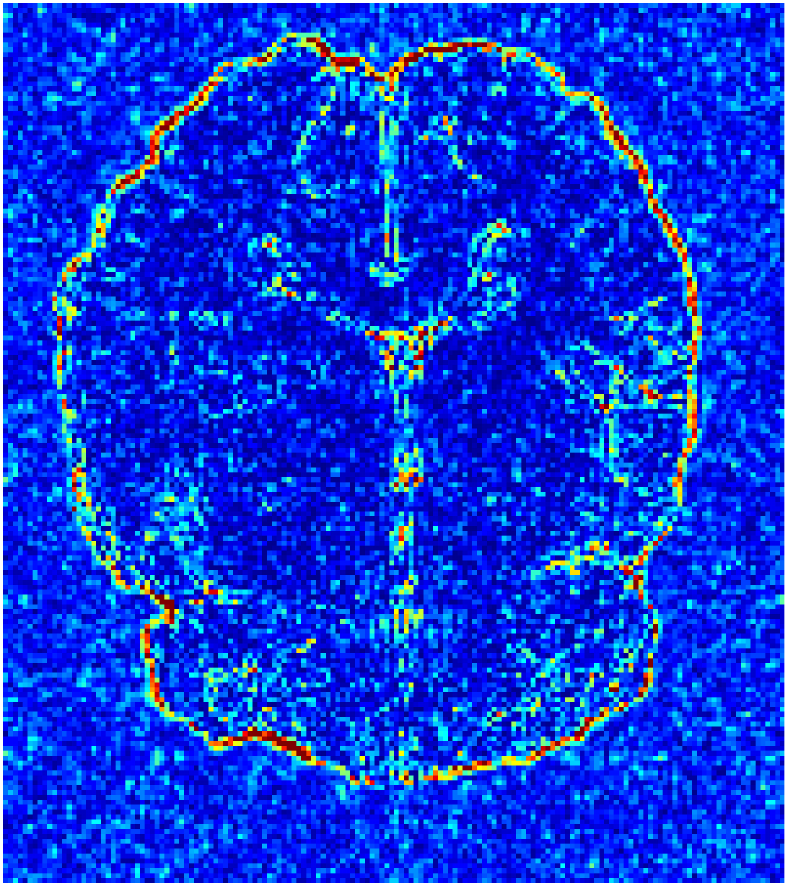}
\includegraphics[width=0.18\linewidth]{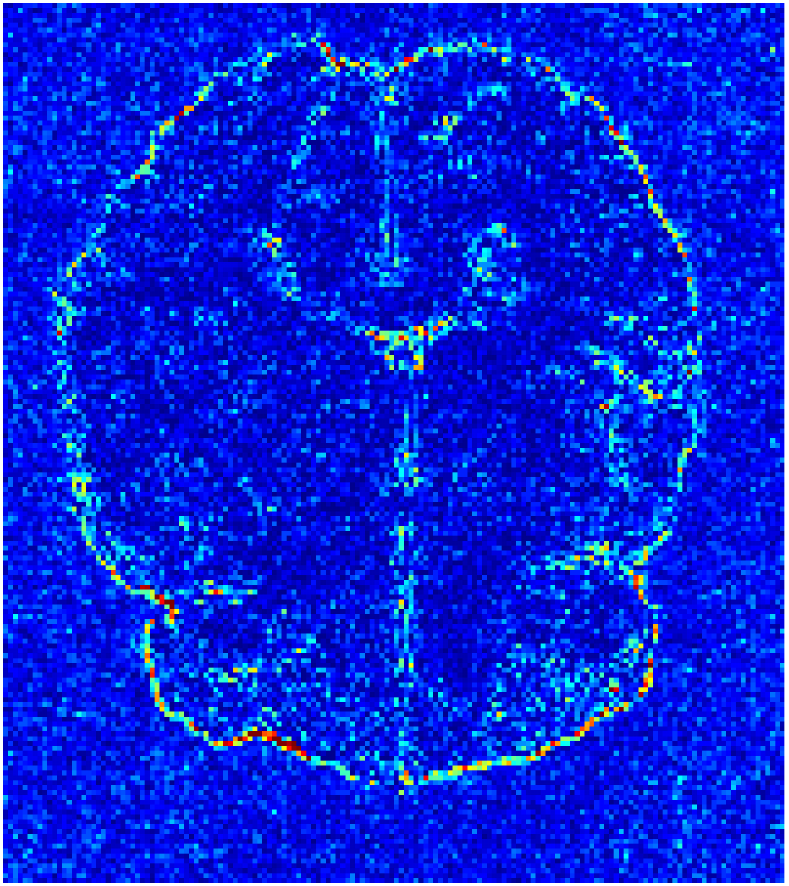}
\includegraphics[width=0.18\linewidth]{fig/white.pdf}\\
\includegraphics[width=0.2\linewidth, angle=90]{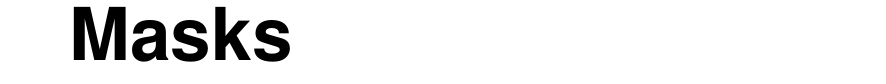}
\includegraphics[width=0.18\linewidth]{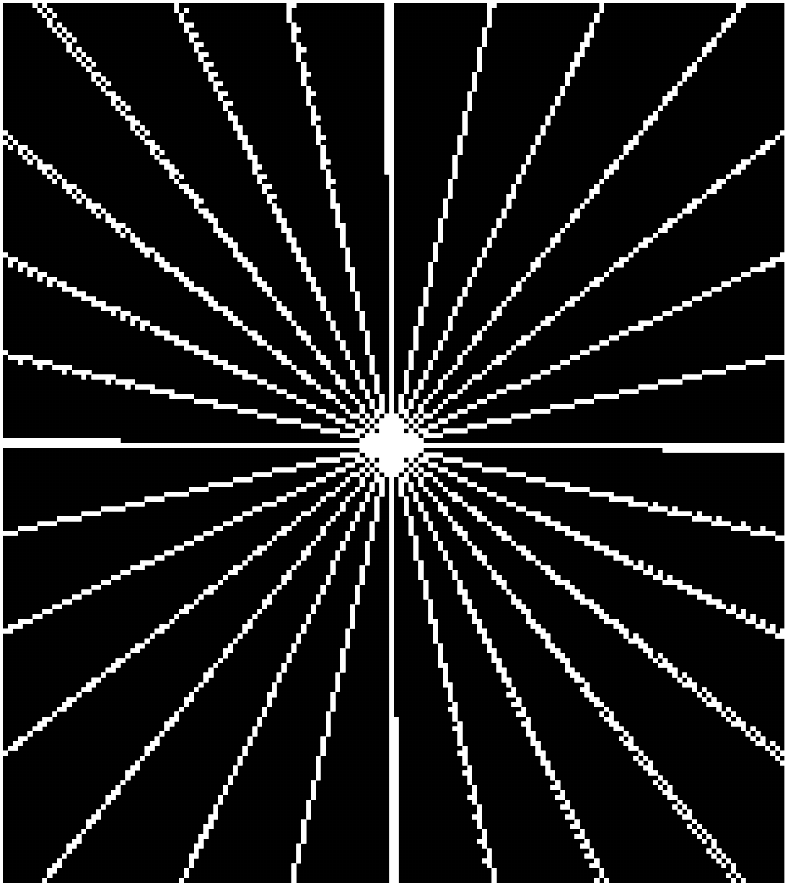}
\includegraphics[width=0.18\linewidth]{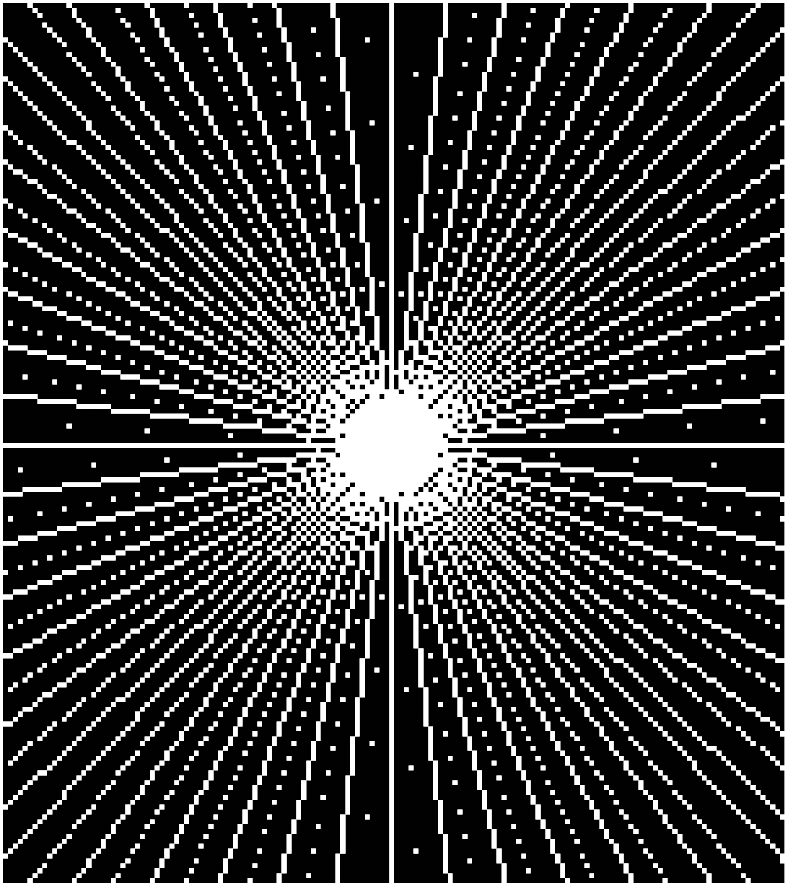}
\includegraphics[width=0.18\linewidth]{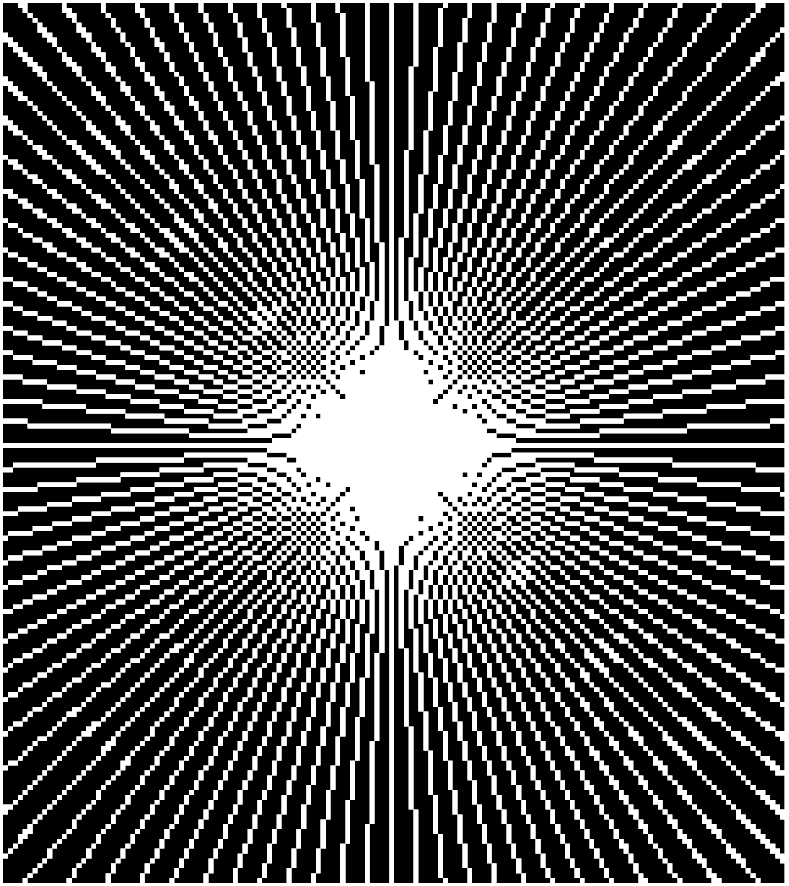}
\includegraphics[width=0.18\linewidth]{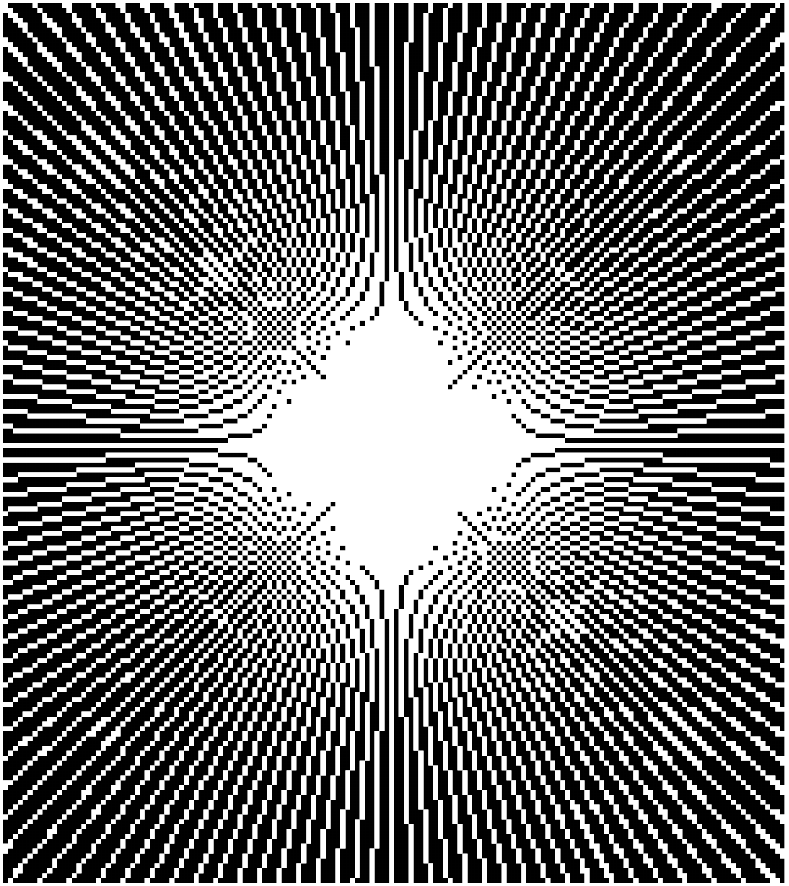}
\includegraphics[width=0.18\linewidth]{fig/white.pdf}
\caption{The pictures (from top to bottom) display the reconstruction results, zoomed in details, pointwise errors with colorbar and associated \textbf{radio} masks for meta-learning and conventional learning with four different CS ratios 10\%, 20\%, 30\%, 40\%（from left to right). The most top right one is ground truth fully-sampled image. }
\label{figure_same_ratio_t1}
\end{figure}

\begin{figure}[H]
\centering
\includegraphics[width=0.2\linewidth, angle=90]{fig/meta_result.pdf}
\includegraphics[width=0.18\linewidth]{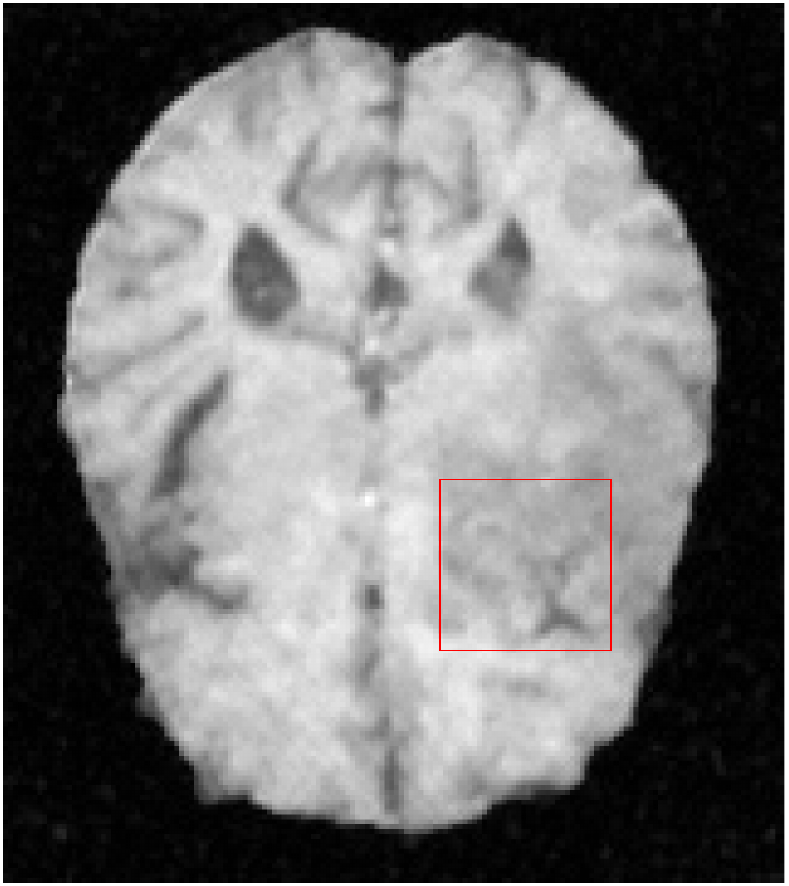}
\includegraphics[width=0.18\linewidth]{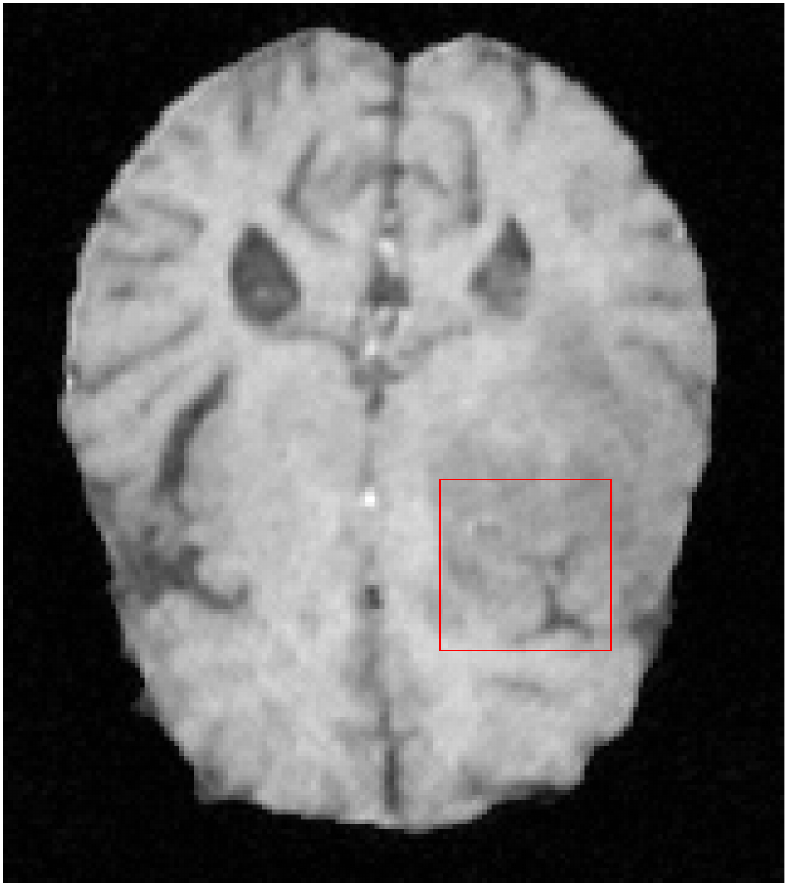}
\includegraphics[width=0.18\linewidth]{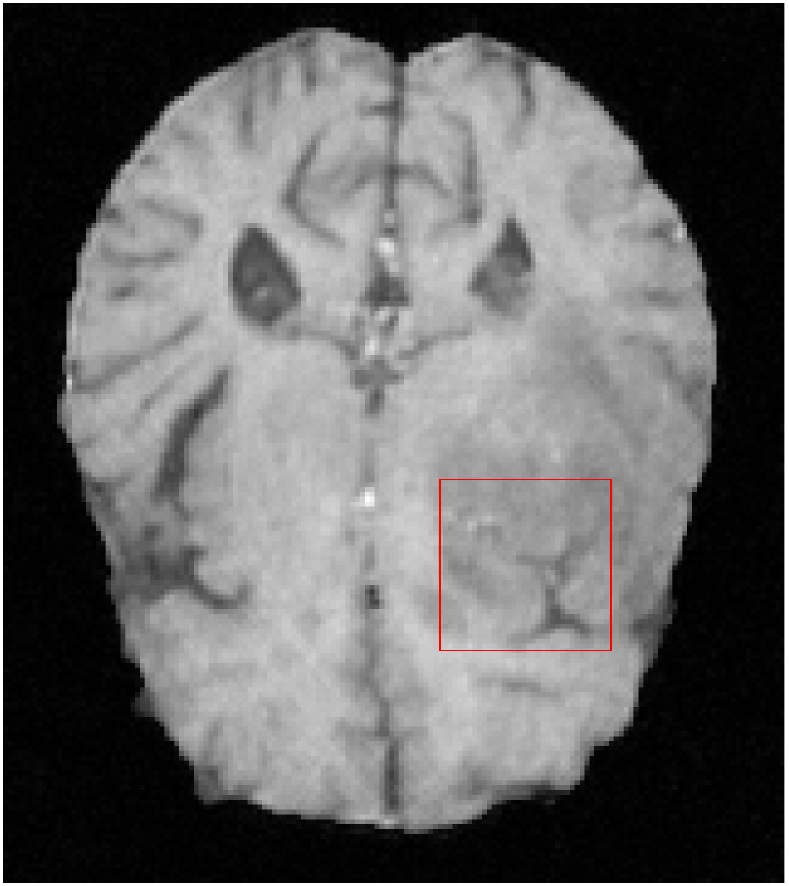}
\includegraphics[width=0.18\linewidth]{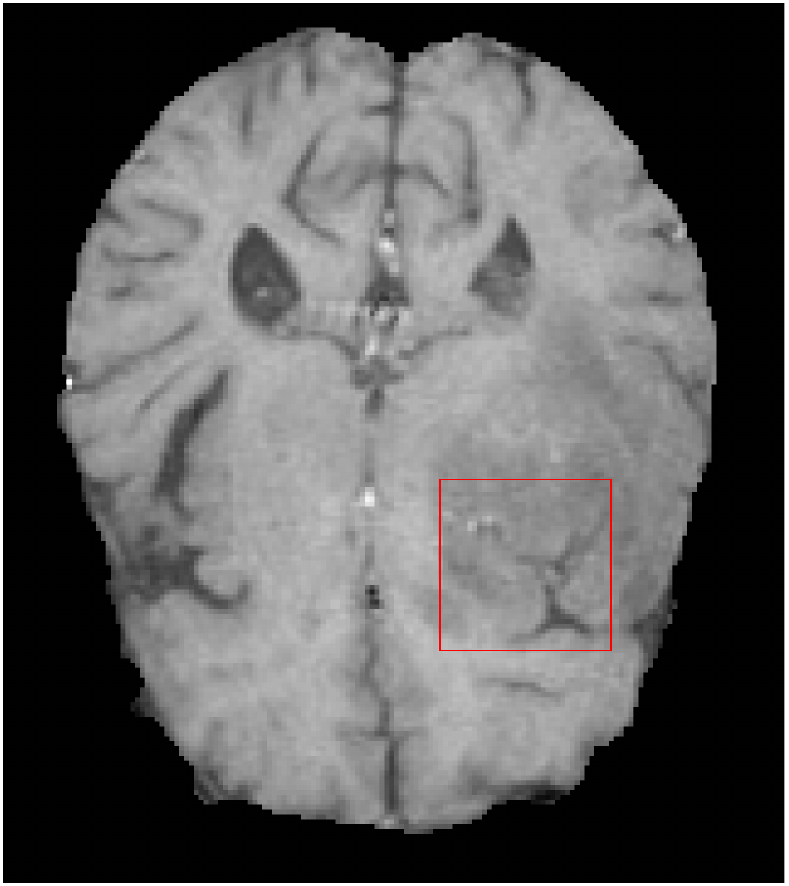}\\
\includegraphics[width=0.2\linewidth, angle=90]{fig/conventional_result.pdf}
\includegraphics[width=0.18\linewidth]{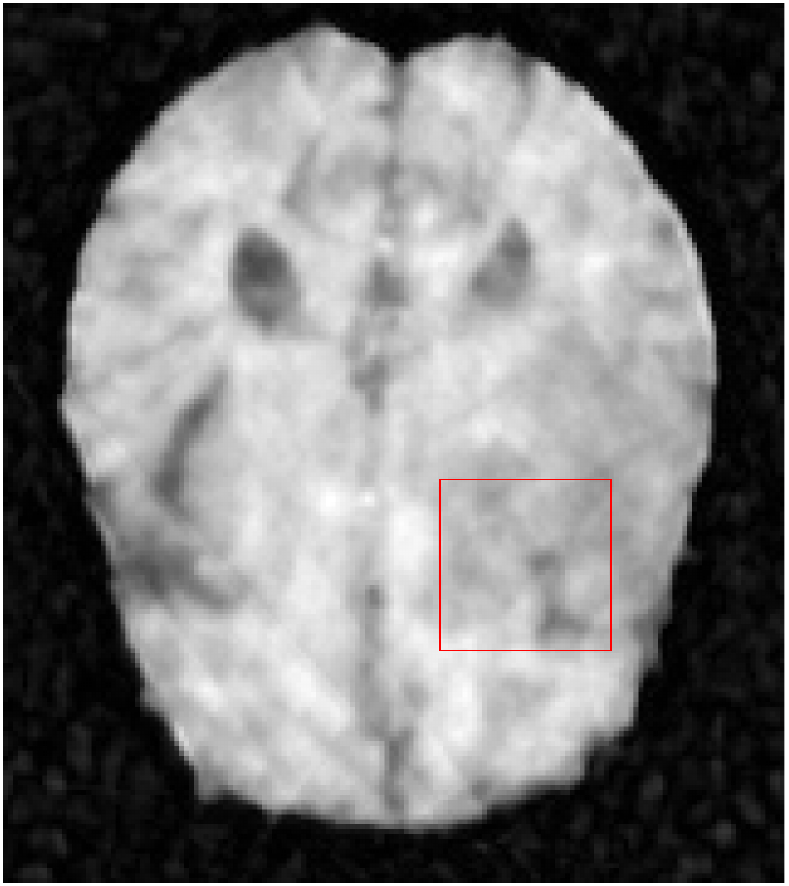}
\includegraphics[width=0.18\linewidth]{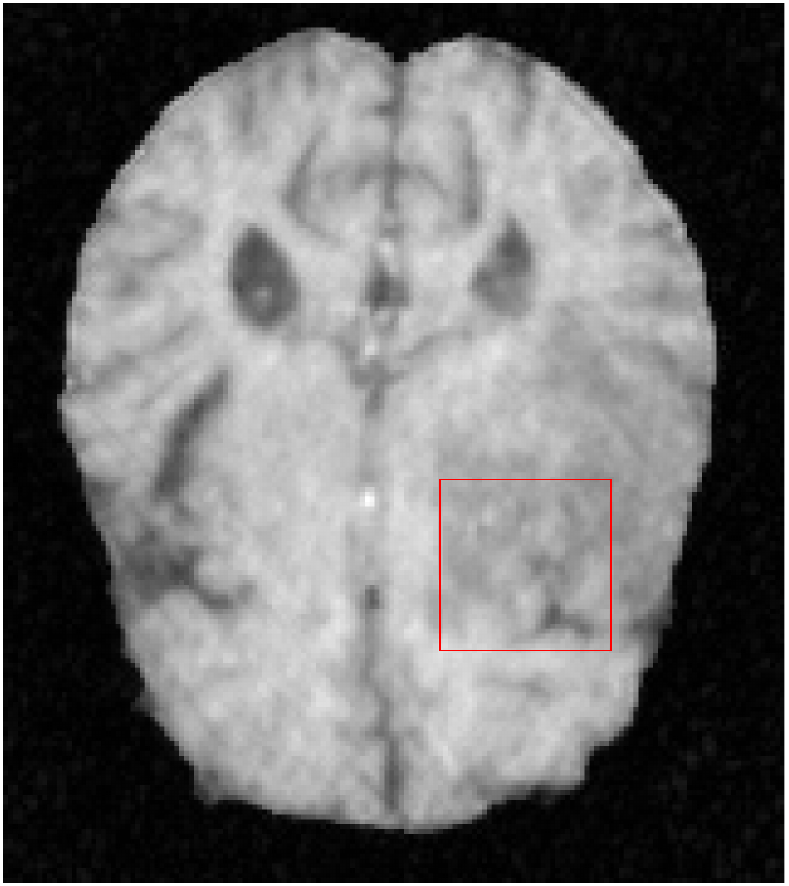}
\includegraphics[width=0.18\linewidth]{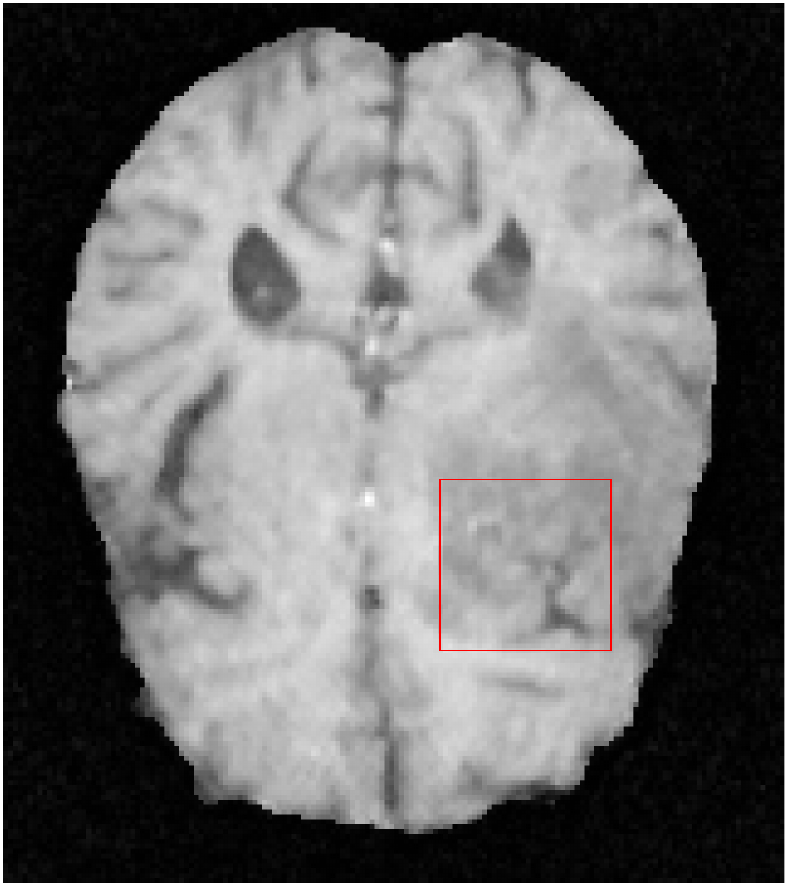}
\includegraphics[width=0.18\linewidth]{fig/white.pdf}\\
\includegraphics[width=0.2\linewidth, angle=90]{fig/meta_detail.pdf}
\includegraphics[width=0.18\linewidth]{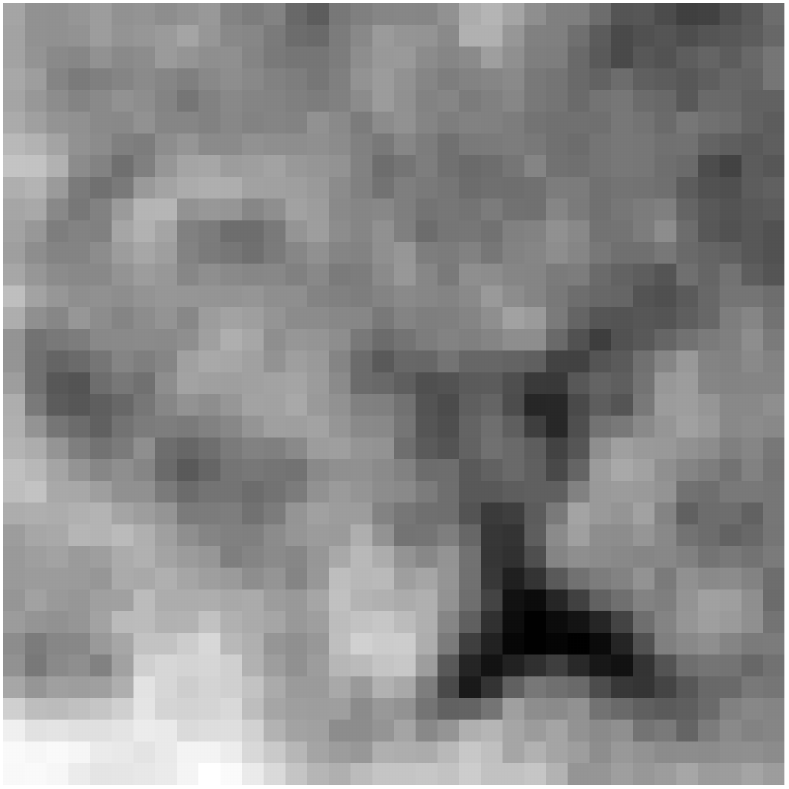}
\includegraphics[width=0.18\linewidth]{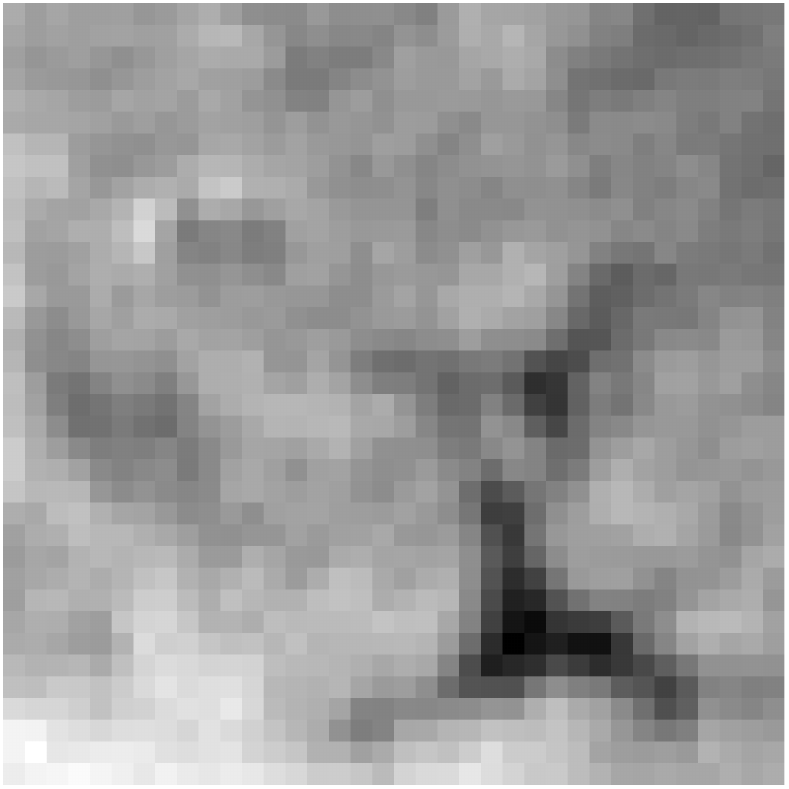}
\includegraphics[width=0.18\linewidth]{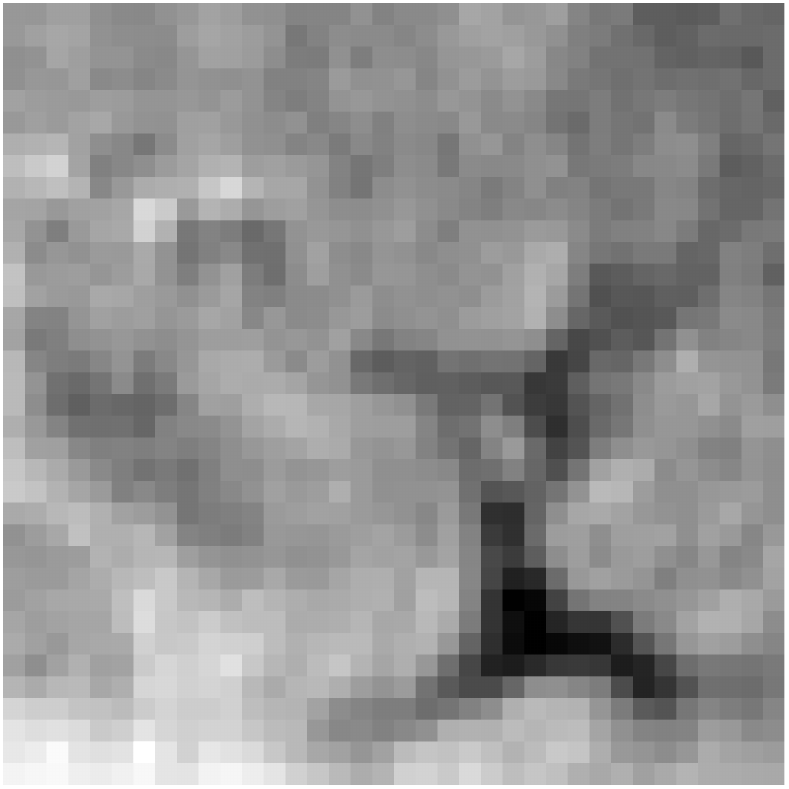}
\includegraphics[width=0.18\linewidth]{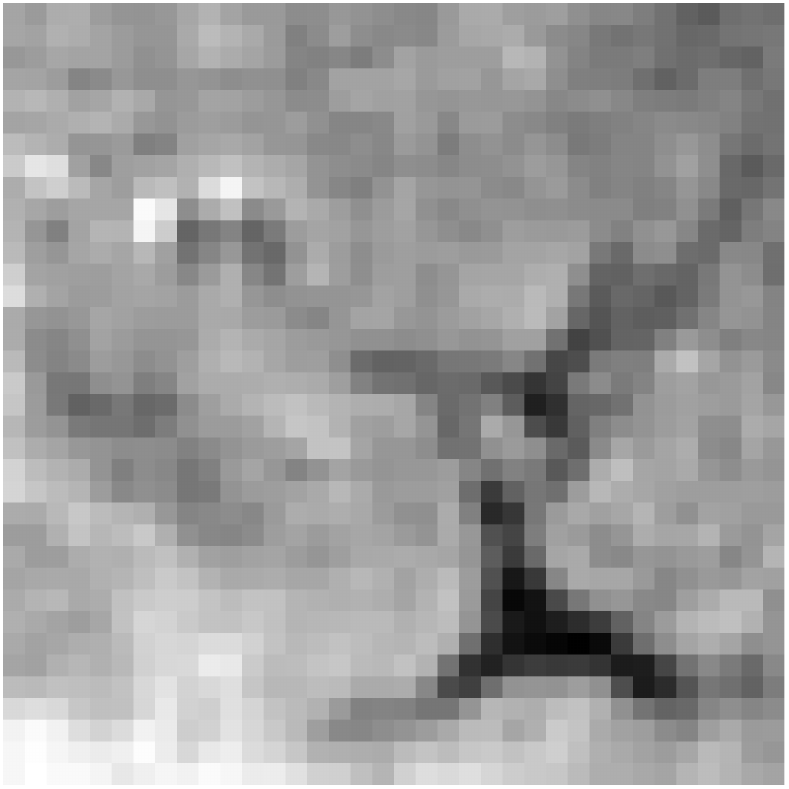}\\
\includegraphics[width=0.2\linewidth, angle=90]{fig/conventional_detail.pdf}
\includegraphics[width=0.18\linewidth]{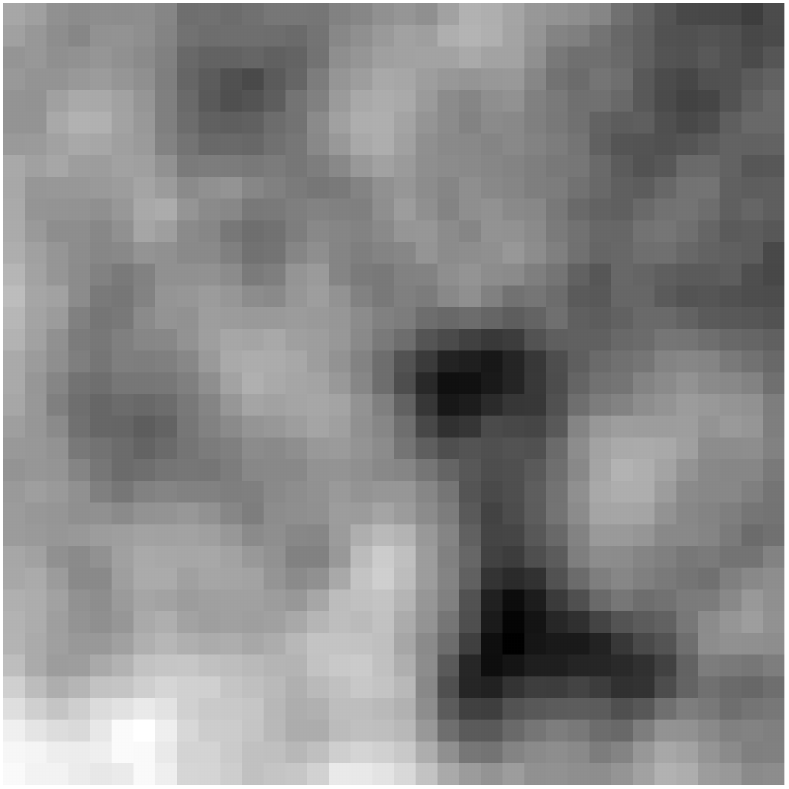}
\includegraphics[width=0.18\linewidth]{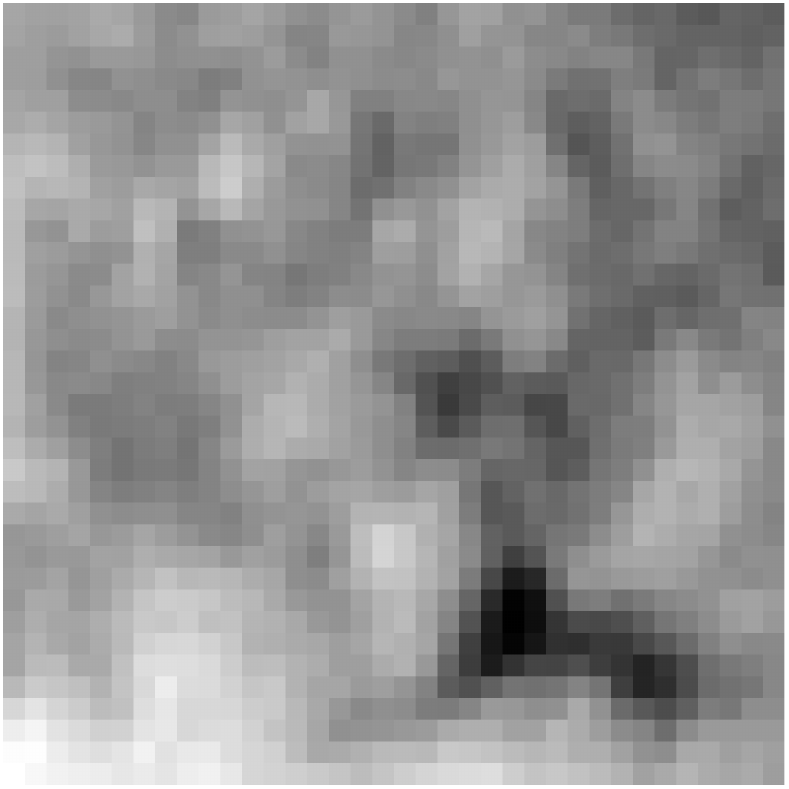}
\includegraphics[width=0.18\linewidth]{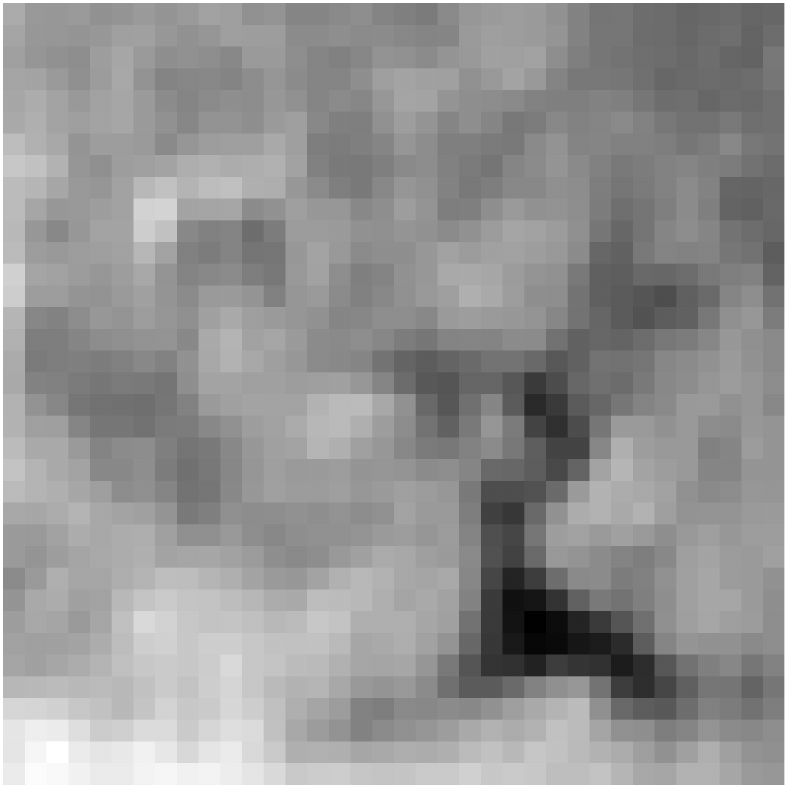}
\includegraphics[width=0.18\linewidth]{fig/white.pdf}\\
\includegraphics[width=0.2\linewidth, angle=90]{fig/meta_error.pdf}
\includegraphics[width=0.18\linewidth]{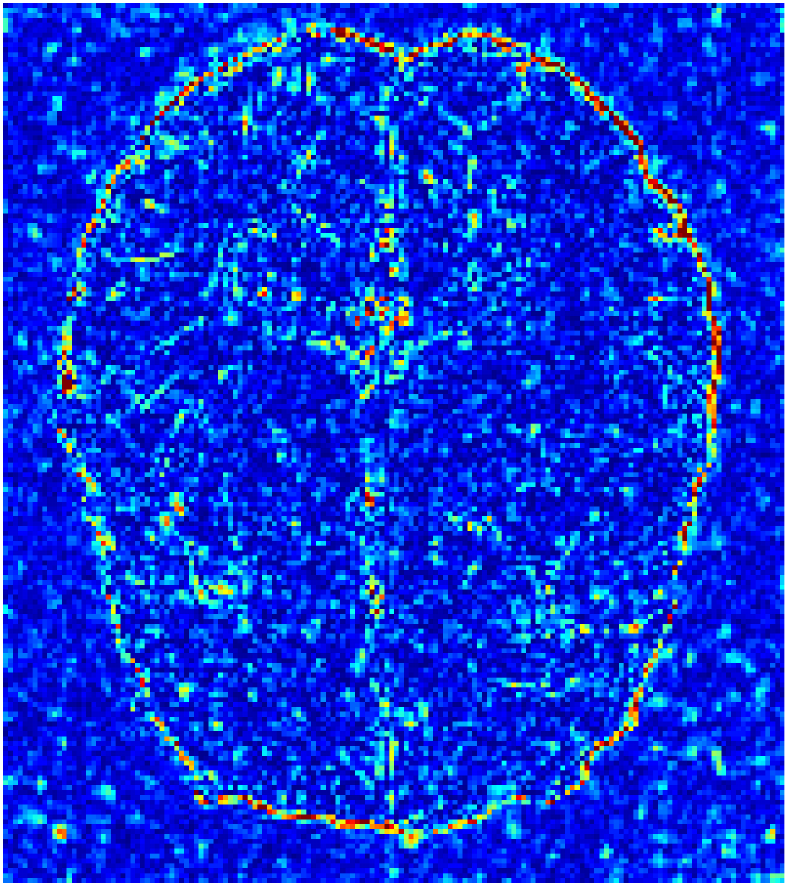}
\includegraphics[width=0.18\linewidth]{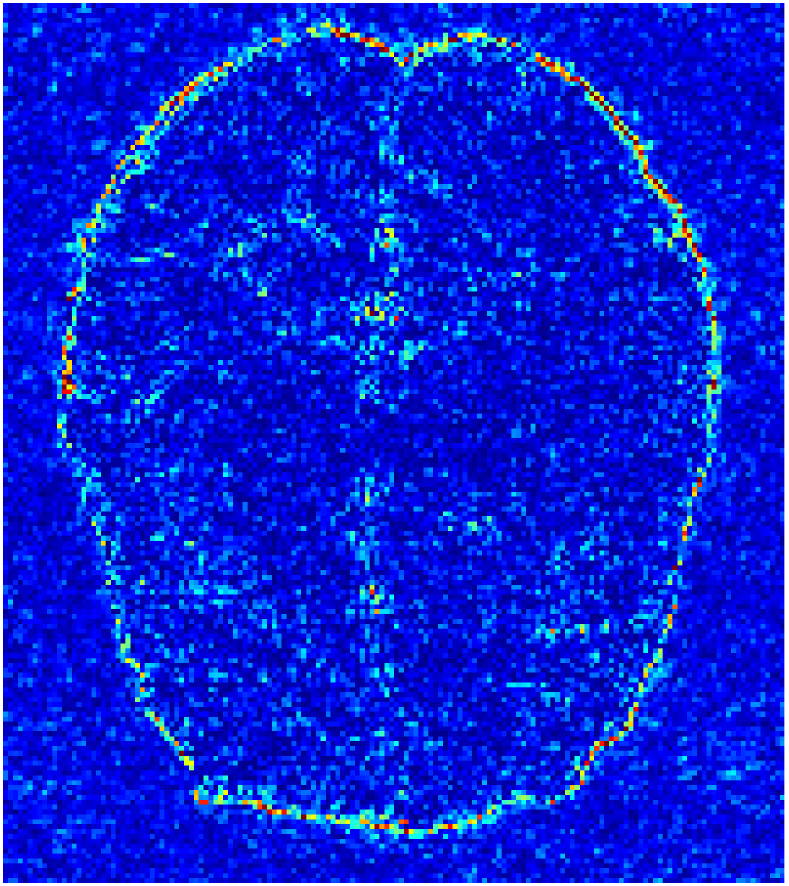}
\includegraphics[width=0.18\linewidth]{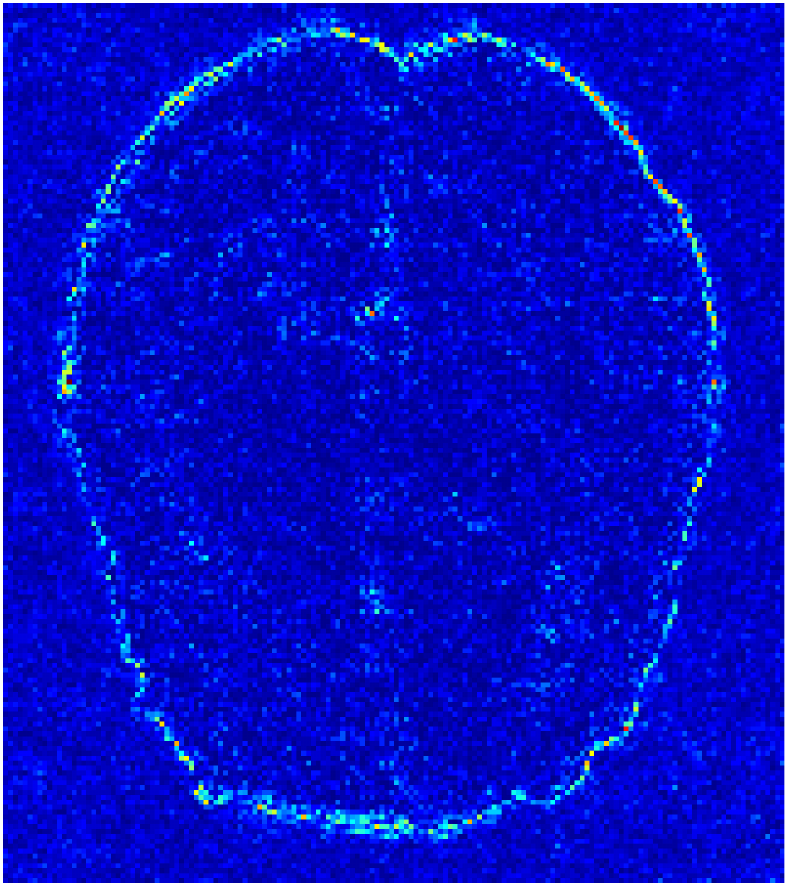}
\includegraphics[width=0.18\linewidth]{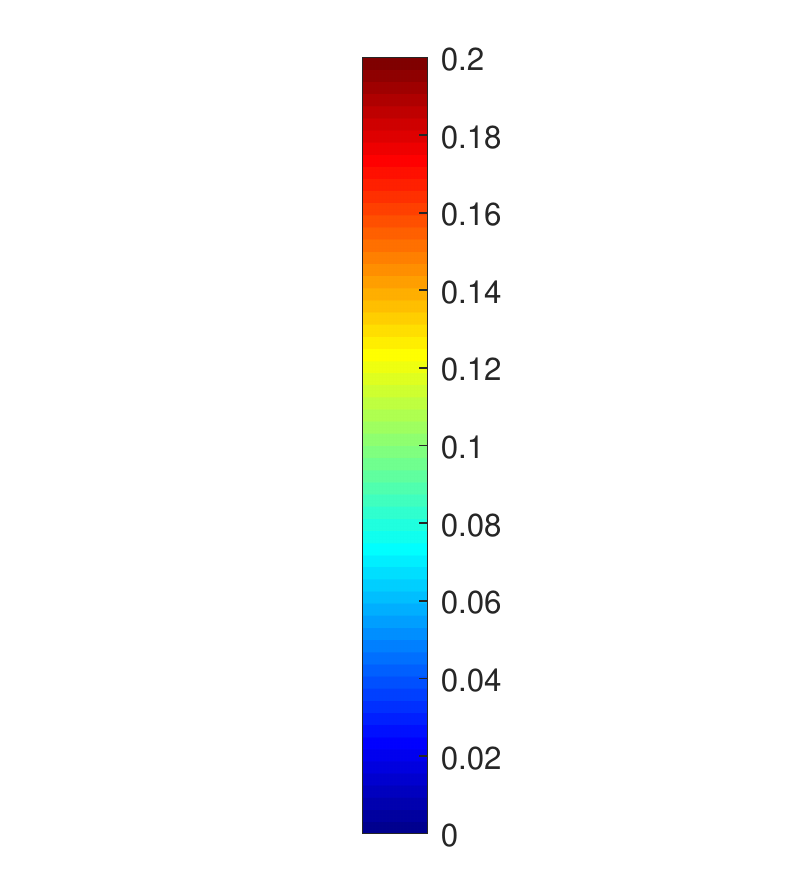}\\
\includegraphics[width=0.2\linewidth, angle=90]{fig/conventional_error.pdf}
\includegraphics[width=0.18\linewidth]{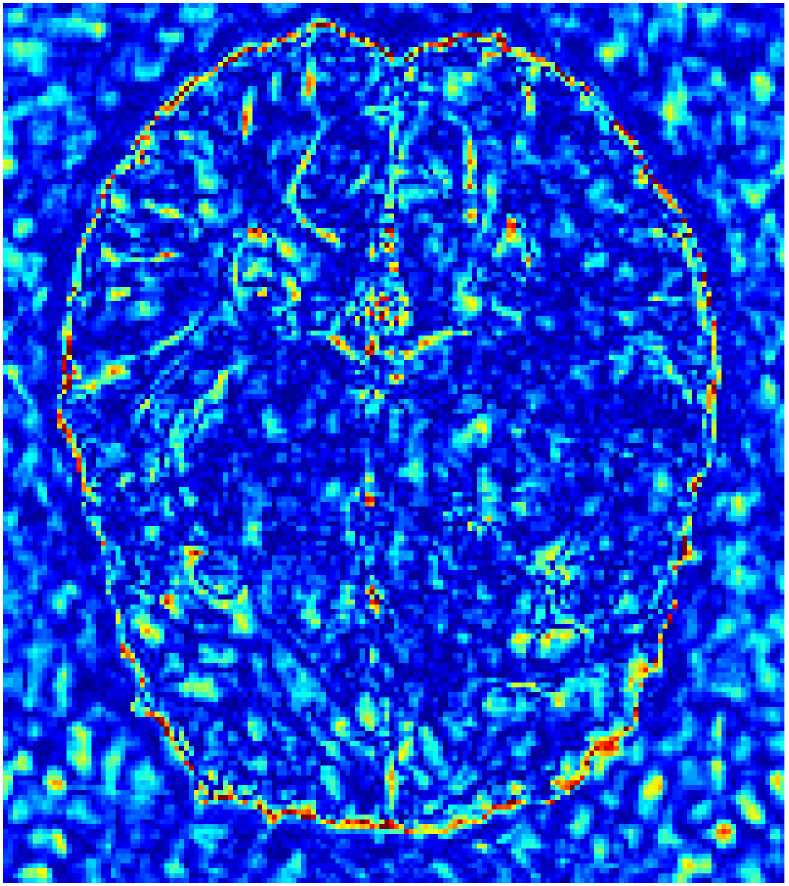}
\includegraphics[width=0.18\linewidth]{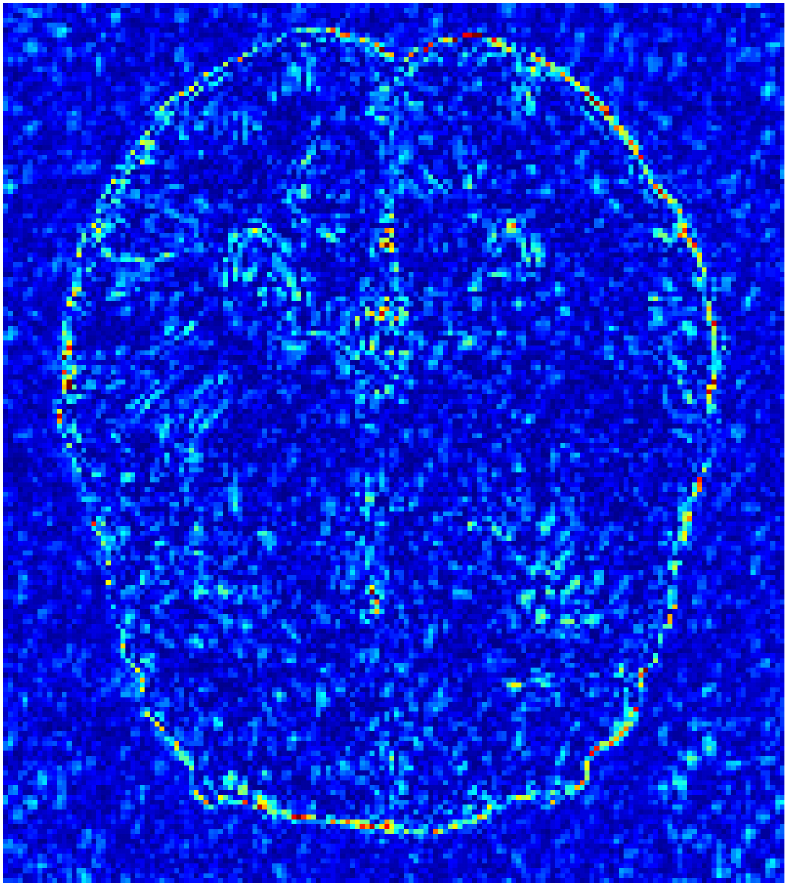}
\includegraphics[width=0.18\linewidth]{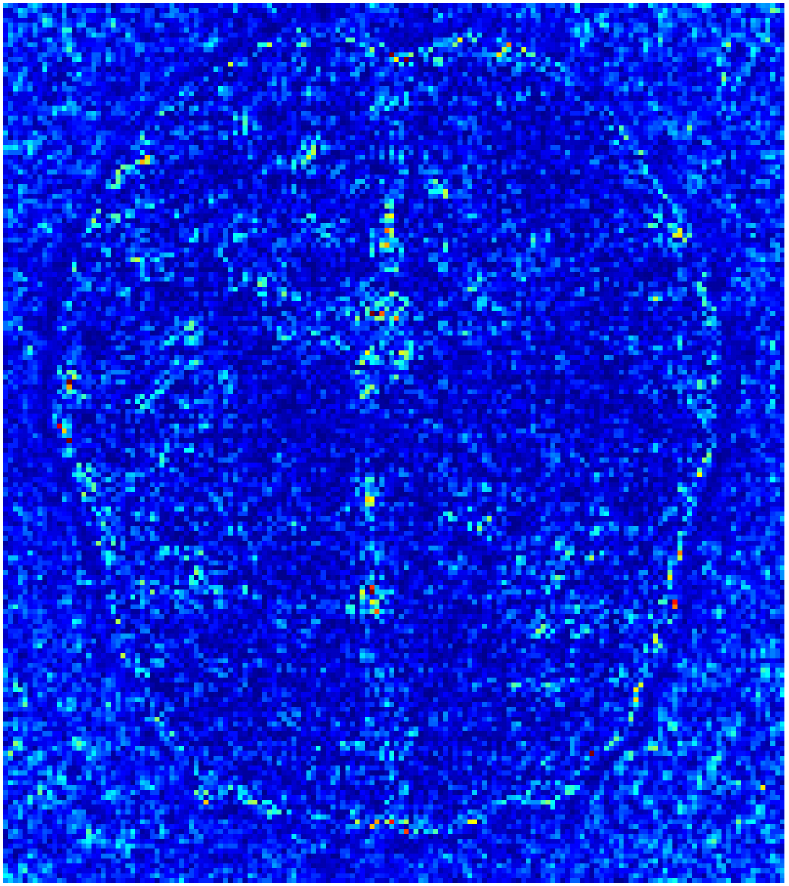}
\includegraphics[width=0.18\linewidth]{fig/white.pdf}\\
\includegraphics[width=0.2\linewidth, angle=90]{fig/masks.pdf}
\includegraphics[width=0.18\linewidth]{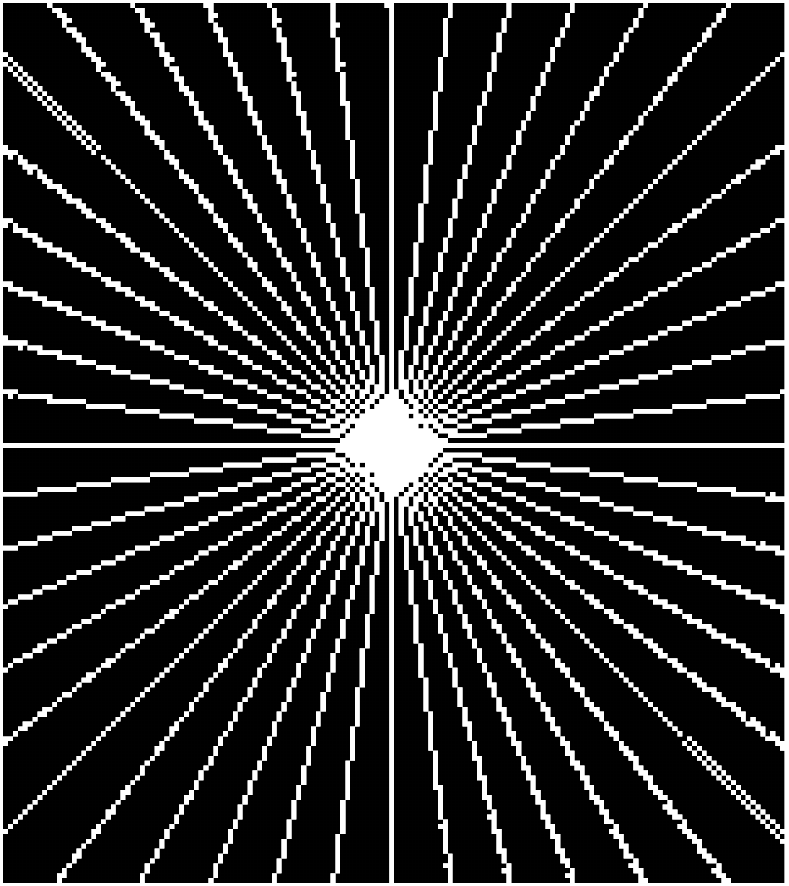}
\includegraphics[width=0.18\linewidth]{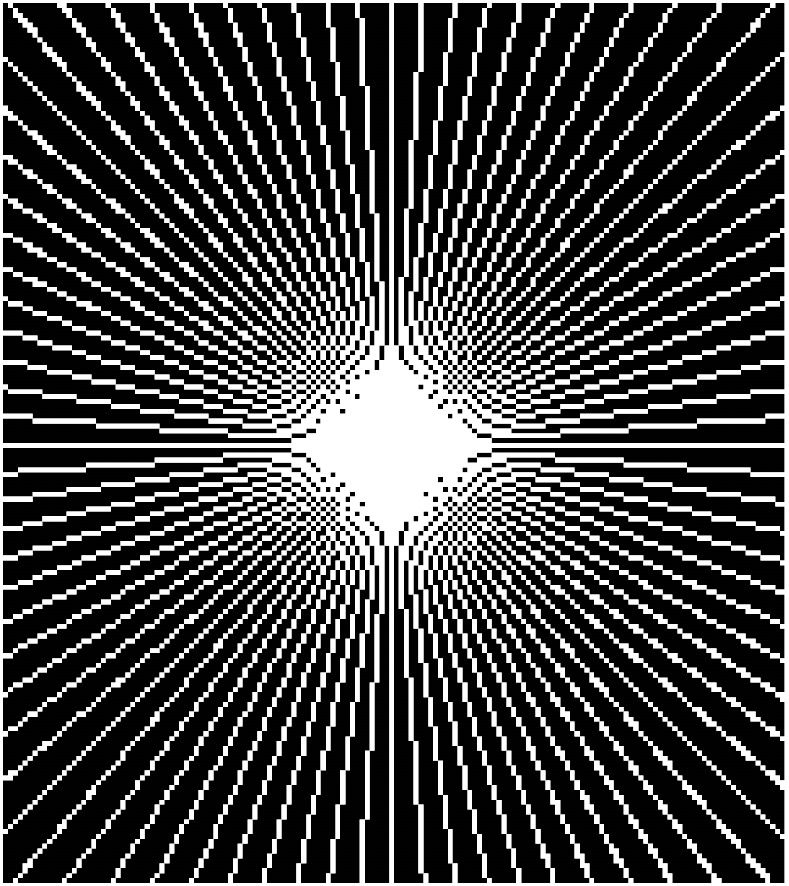}
\includegraphics[width=0.18\linewidth]{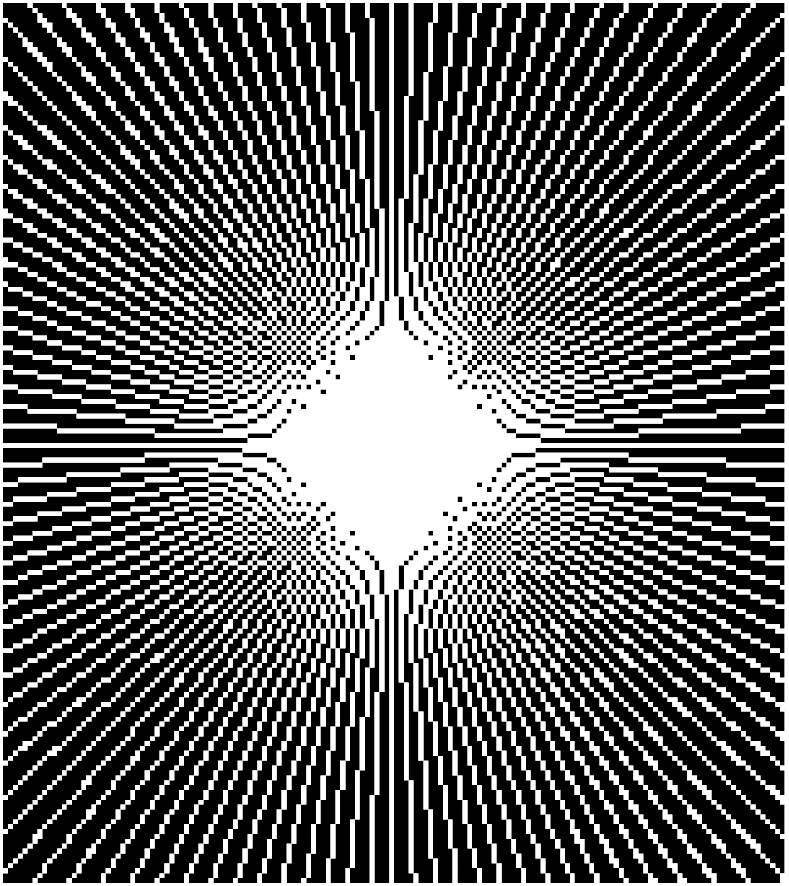}
\includegraphics[width=0.18\linewidth]{fig/white.pdf}
\caption{The pictures (from top to bottom) display the T1 Brain image reconstruction results, zoomed in details, pointwise errors with colorbar and associated \textbf{radio} masks for meta-learning and conventional learning. Meta-learning was trained with CS ratios 10\%, 20\%, 30\% 40\% and test with three different CS ratios 15\%, 25\% and 35\%（from left to right). Conventional learning was trained and test with same CS ratios 15\%, 25\% and 35\%. The most top right one is ground truth fully-sampled image.}
\label{figure_dif_ratio_t1}
\end{figure}

\begin{figure}[H]
\centering
\includegraphics[width=0.2\linewidth, angle=90]{fig/meta_result.pdf}
\includegraphics[width=0.18\linewidth]{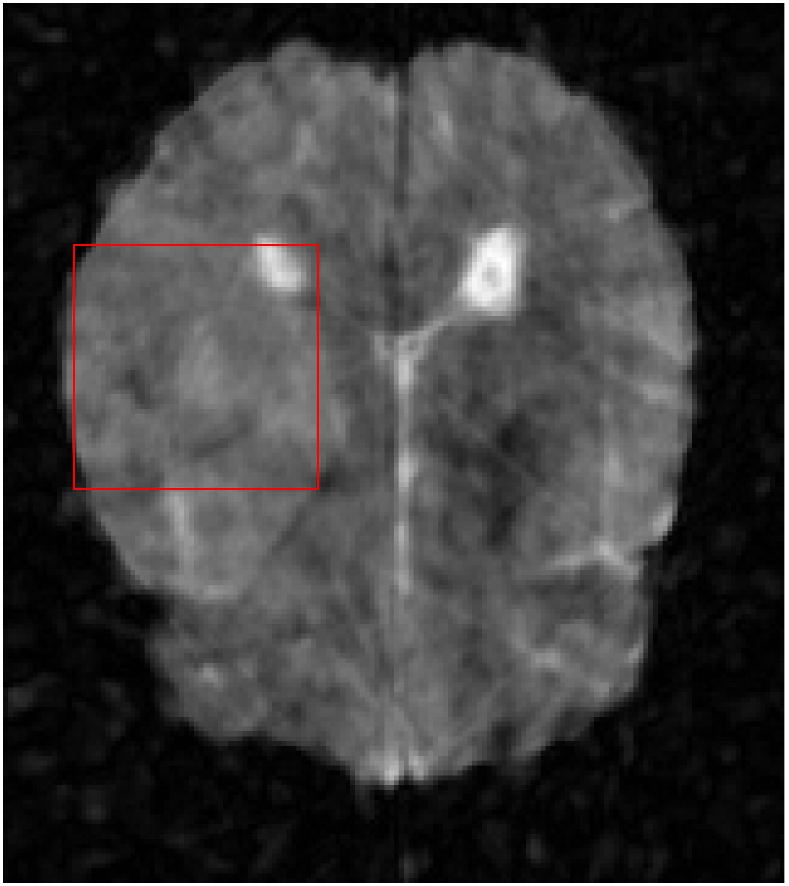}
\includegraphics[width=0.18\linewidth]{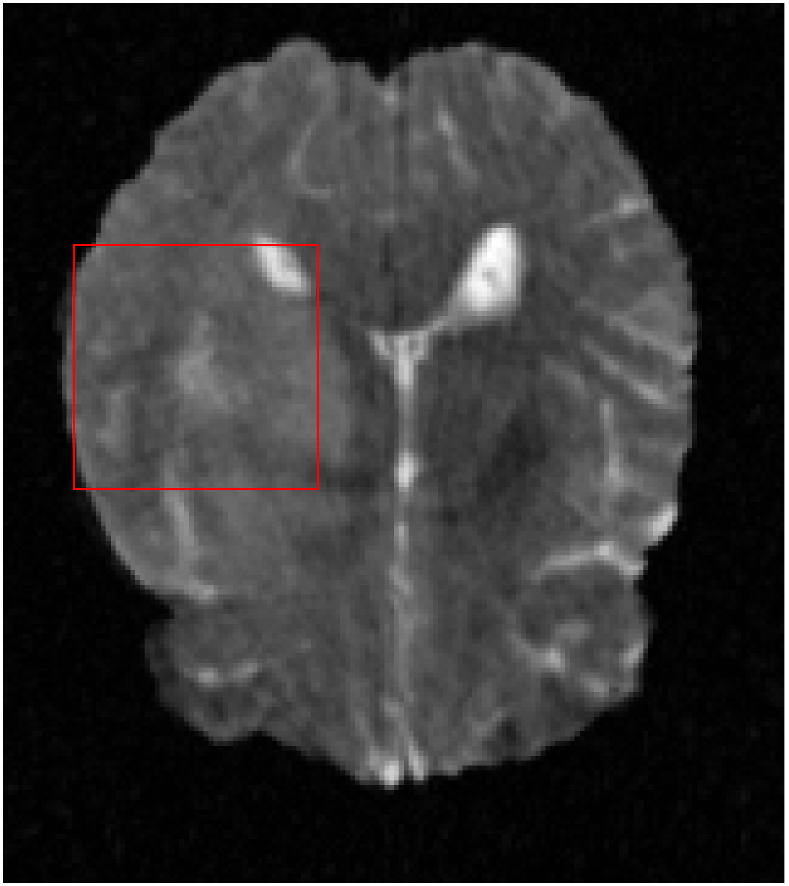}
\includegraphics[width=0.18\linewidth]{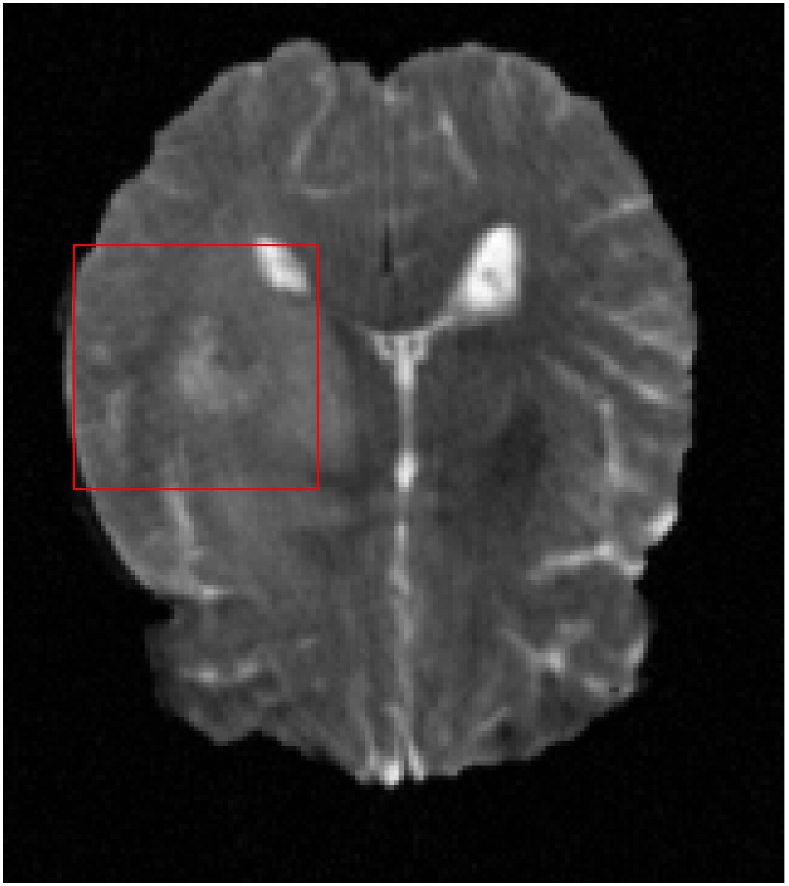}
\includegraphics[width=0.18\linewidth]{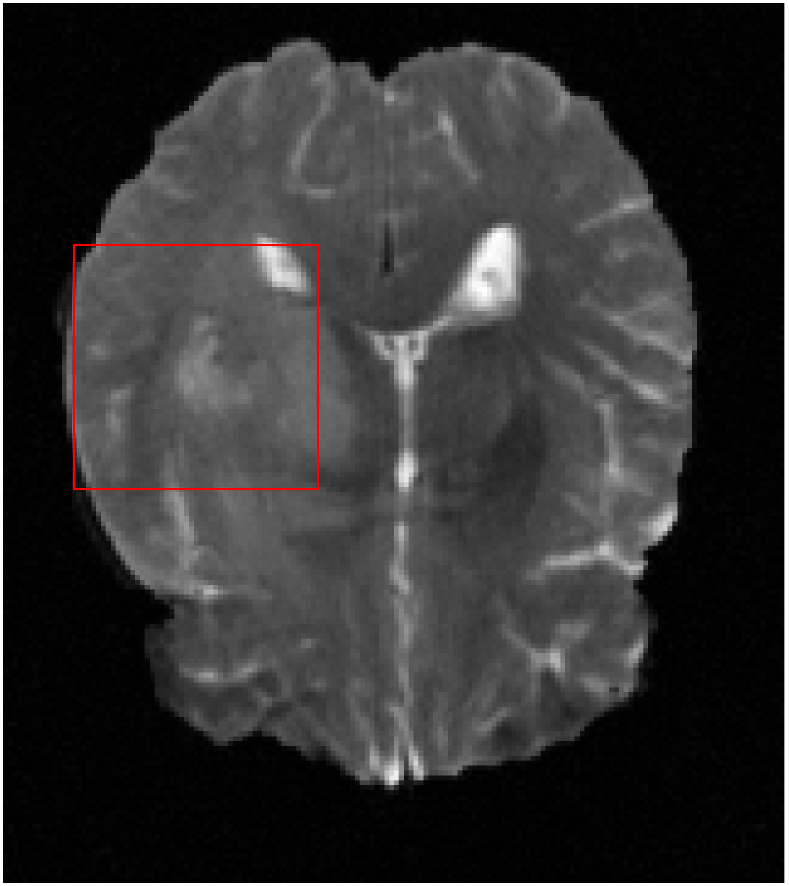}
\includegraphics[width=0.18\linewidth]{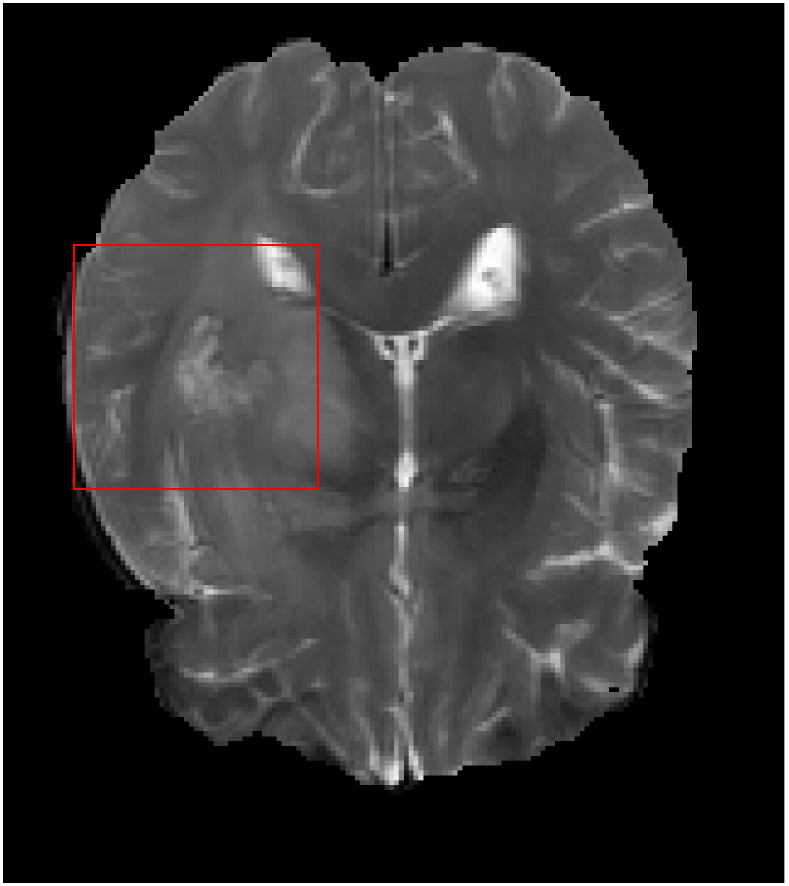}\\
\includegraphics[width=0.2\linewidth, angle=90]{fig/conventional_result.pdf}
\includegraphics[width=0.18\linewidth]{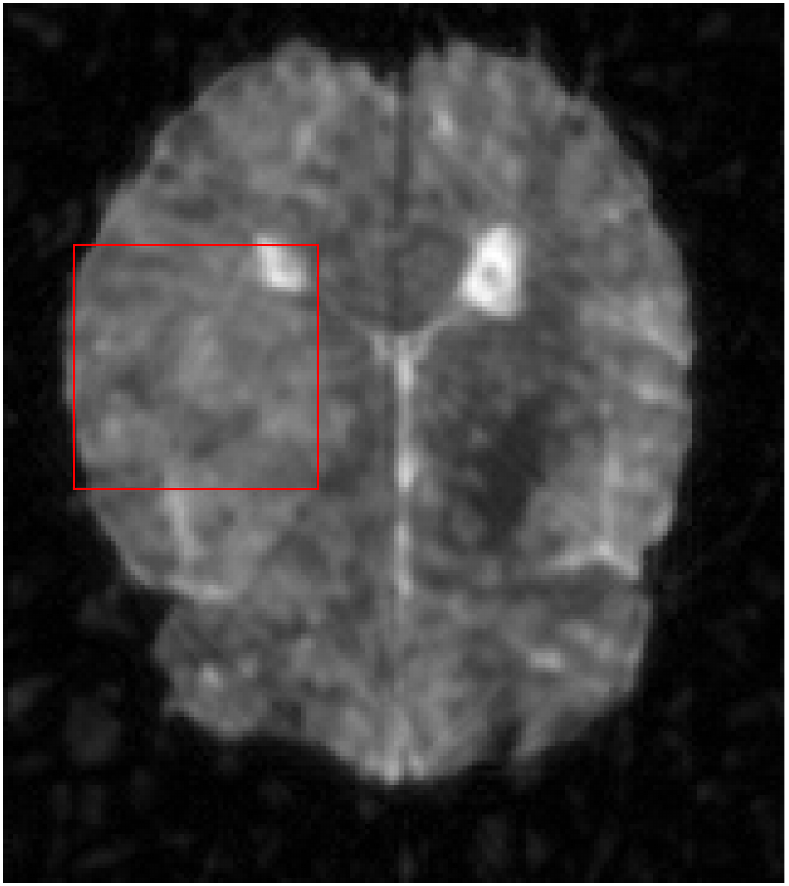}
\includegraphics[width=0.18\linewidth]{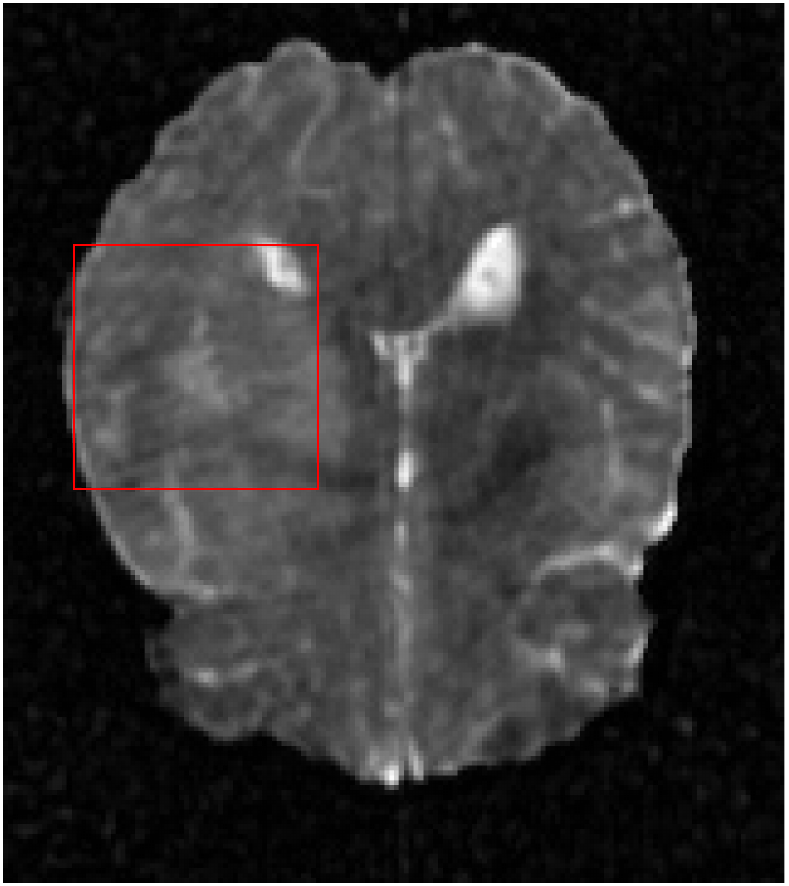}
\includegraphics[width=0.18\linewidth]{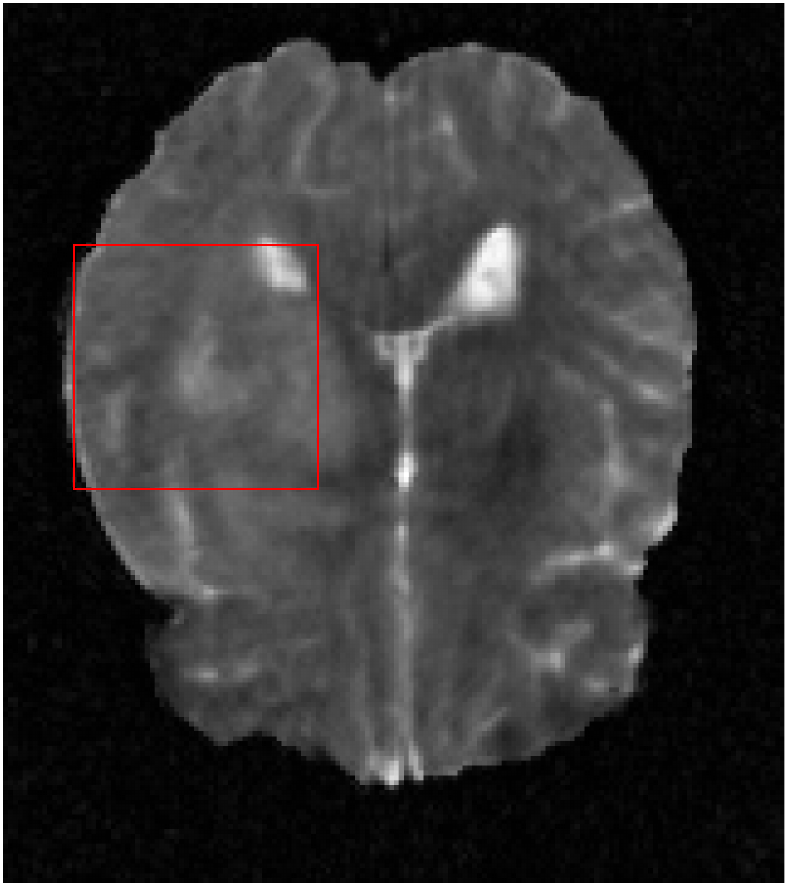}
\includegraphics[width=0.18\linewidth]{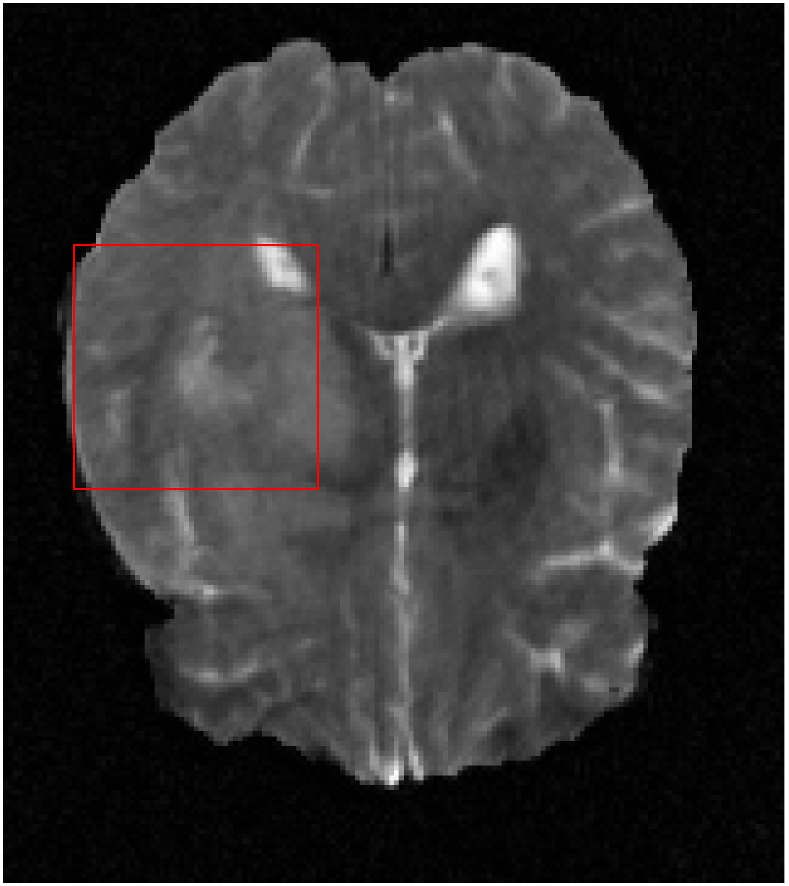}
\includegraphics[width=0.18\linewidth]{fig/white.pdf}\\
\includegraphics[width=0.2\linewidth, angle=90]{fig/meta_detail.pdf}
\includegraphics[width=0.18\linewidth]{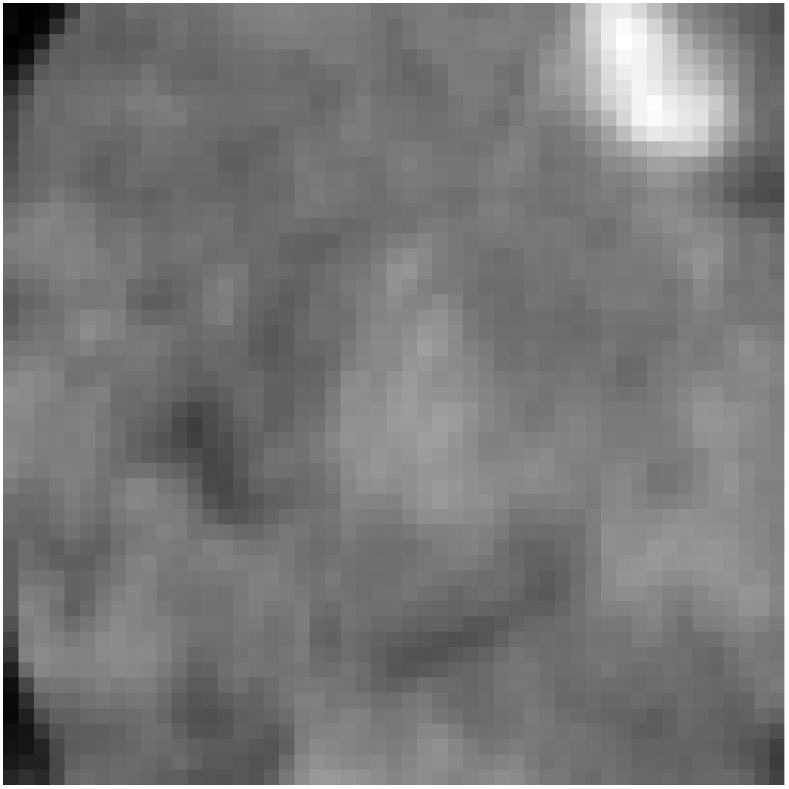}
\includegraphics[width=0.18\linewidth]{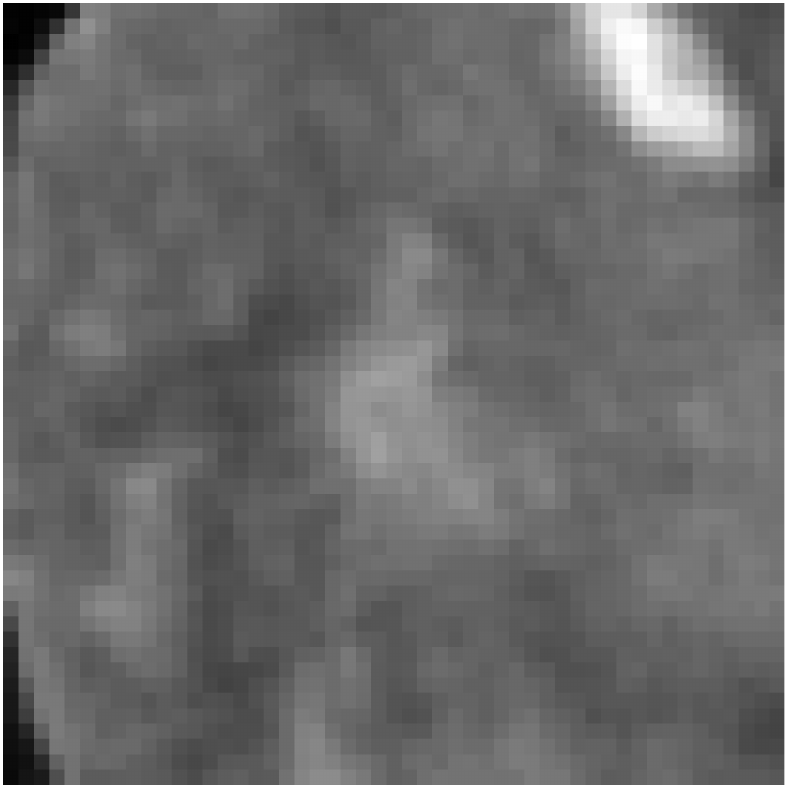}
\includegraphics[width=0.18\linewidth]{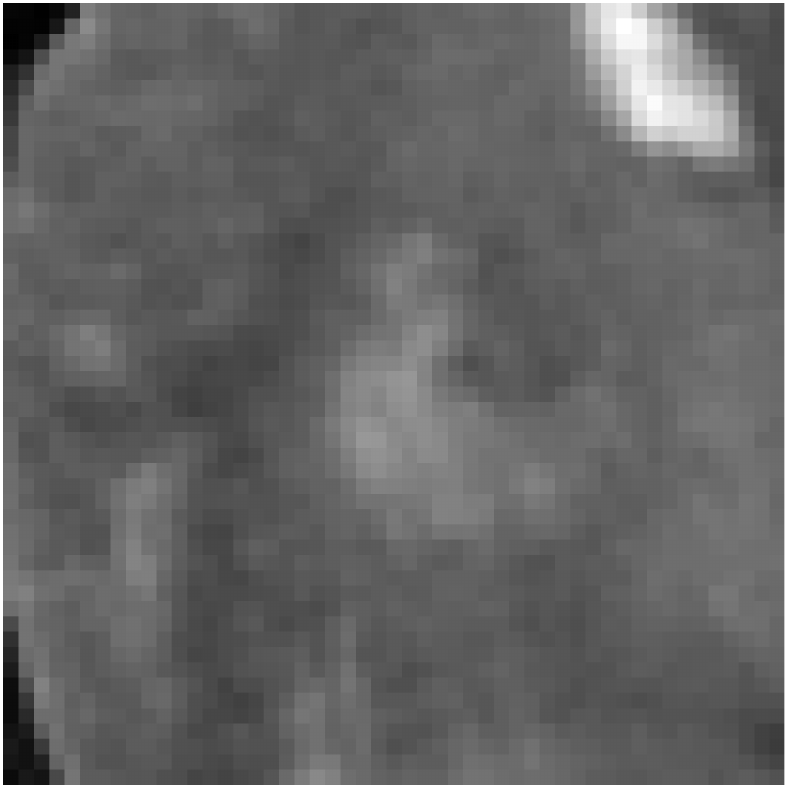}
\includegraphics[width=0.18\linewidth]{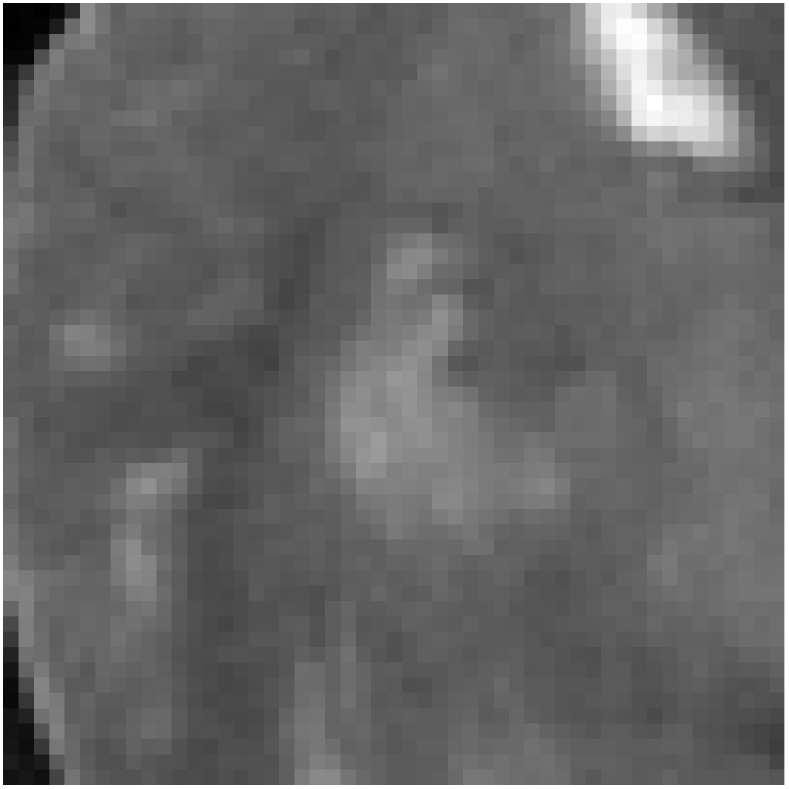}
\includegraphics[width=0.18\linewidth]{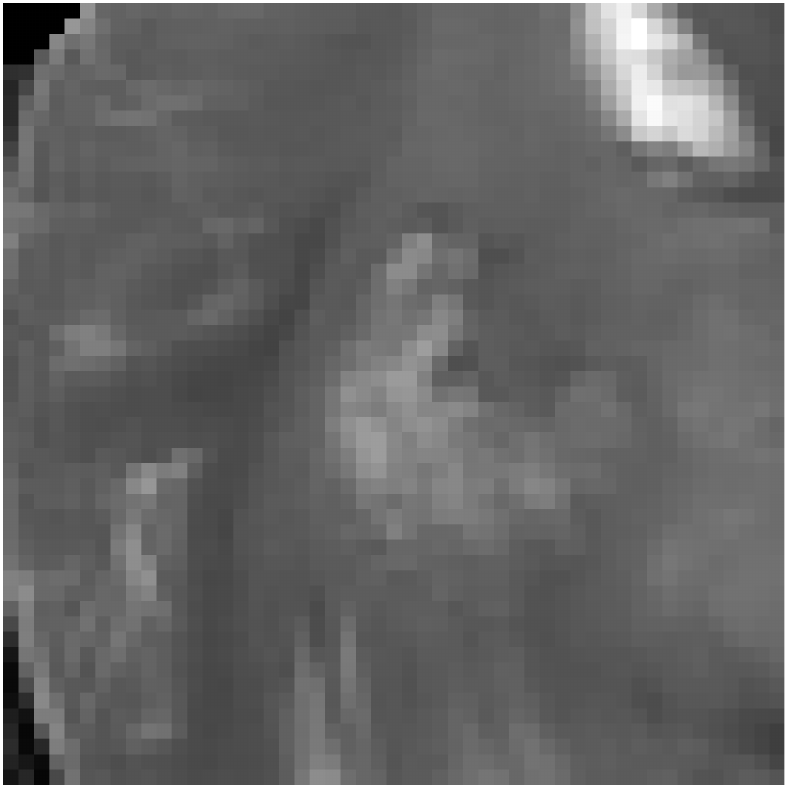}\\
\includegraphics[width=0.2\linewidth, angle=90]{fig/conventional_detail.pdf}
\includegraphics[width=0.18\linewidth]{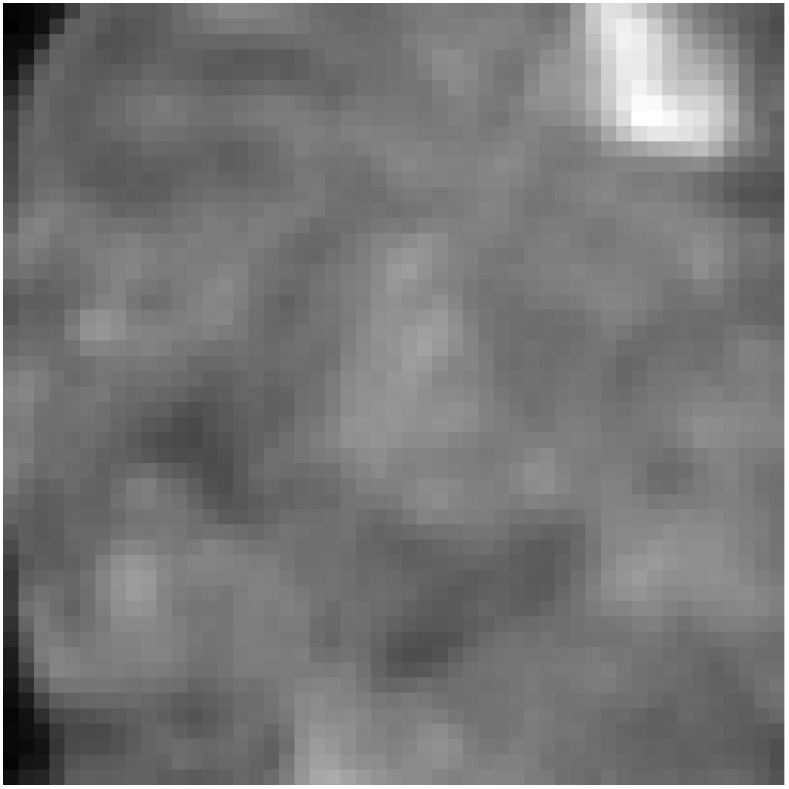}
\includegraphics[width=0.18\linewidth]{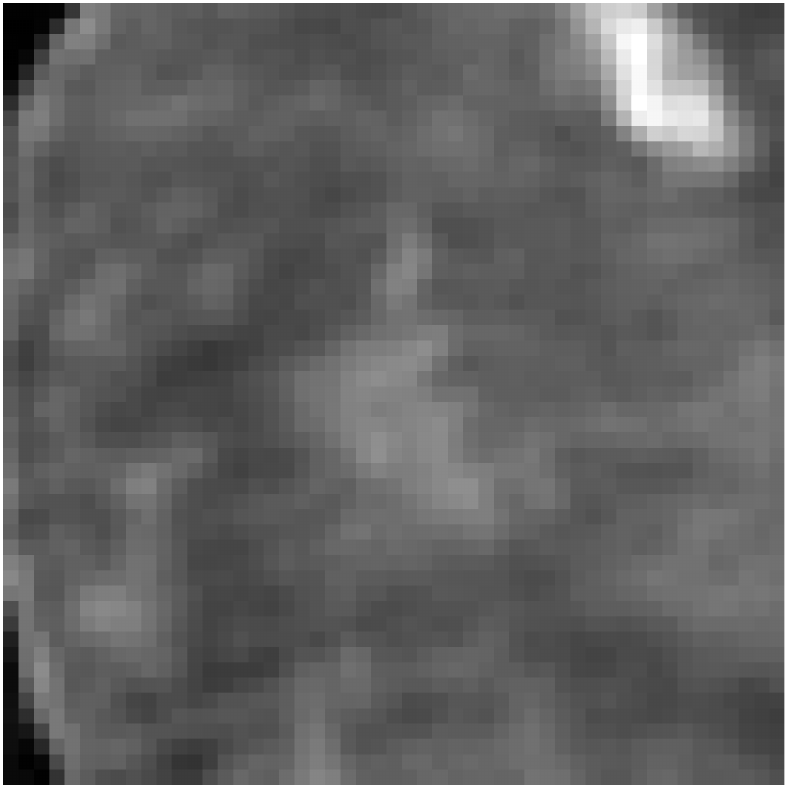}
\includegraphics[width=0.18\linewidth]{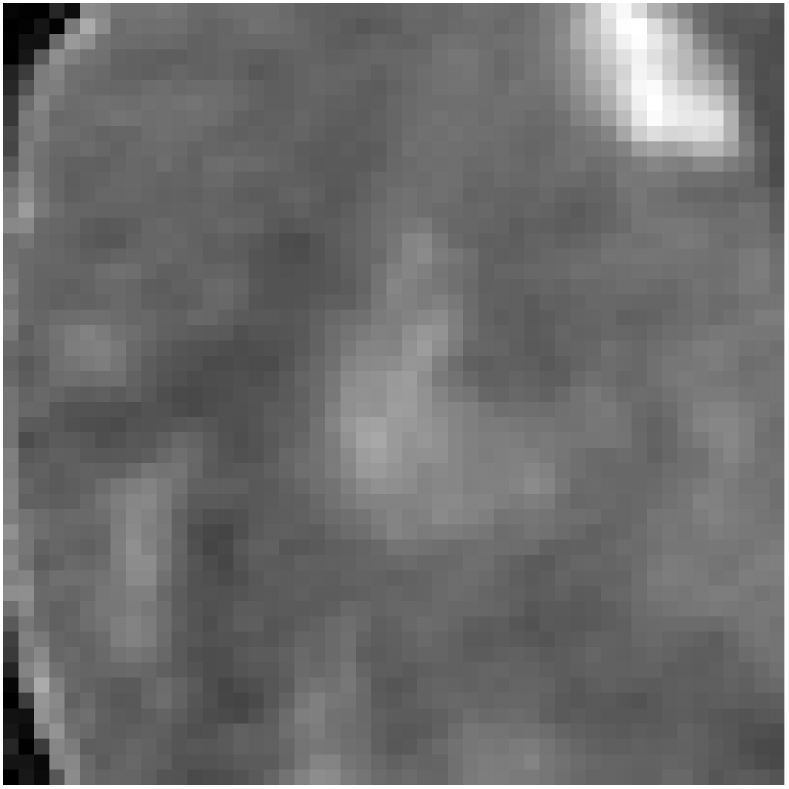}
\includegraphics[width=0.18\linewidth]{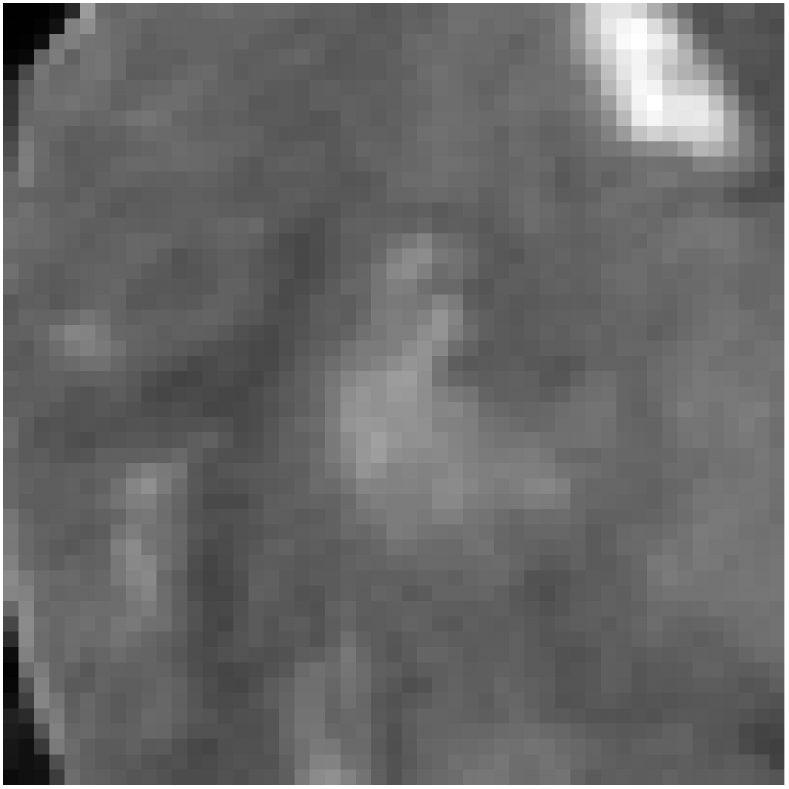}
\includegraphics[width=0.18\linewidth]{fig/white.pdf}\\
\includegraphics[width=0.2\linewidth, angle=90]{fig/meta_error.pdf}
\includegraphics[width=0.18\linewidth]{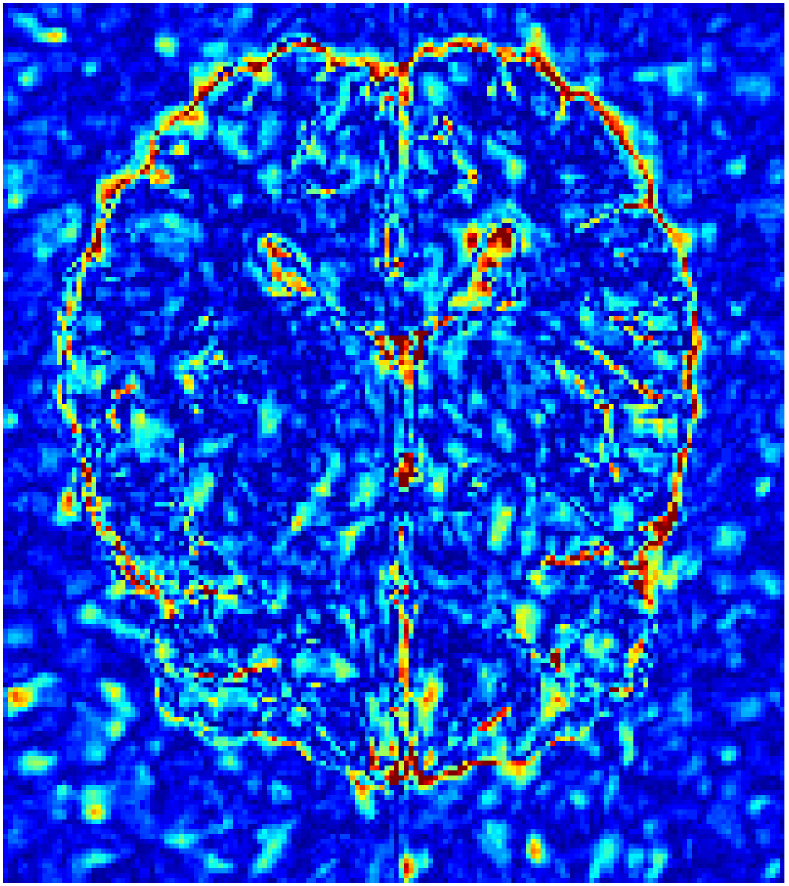}
\includegraphics[width=0.18\linewidth]{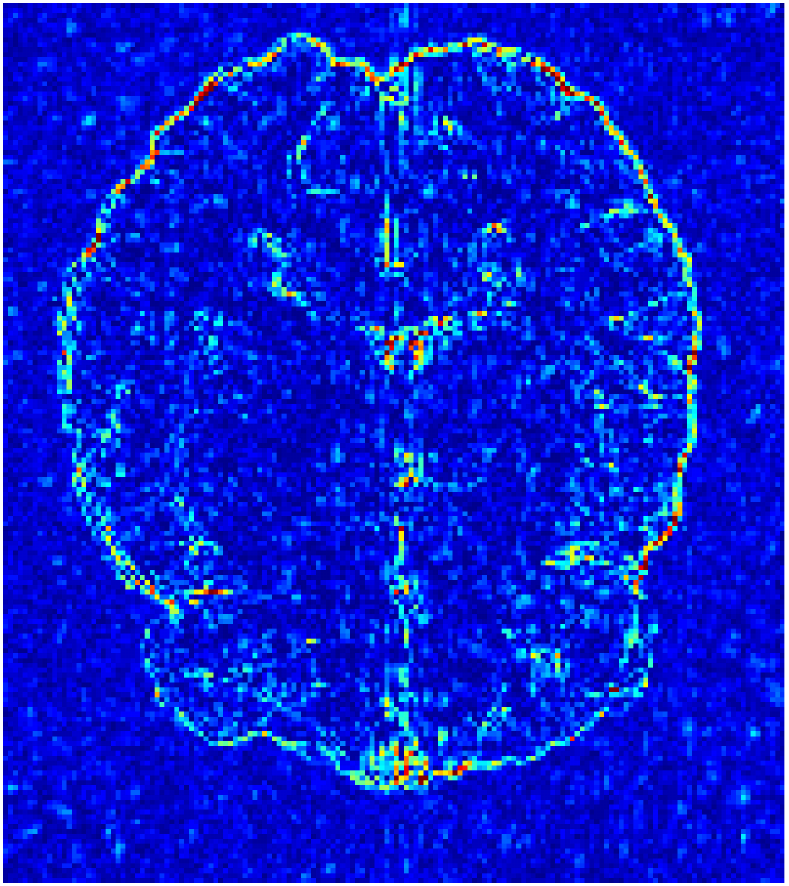}
\includegraphics[width=0.18\linewidth]{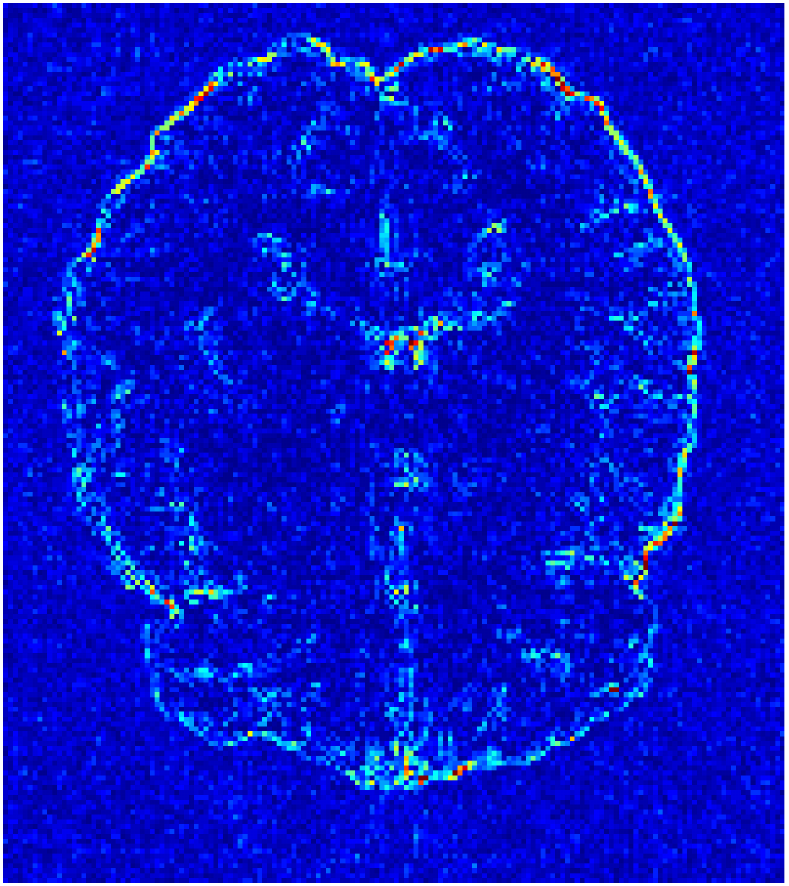}
\includegraphics[width=0.18\linewidth]{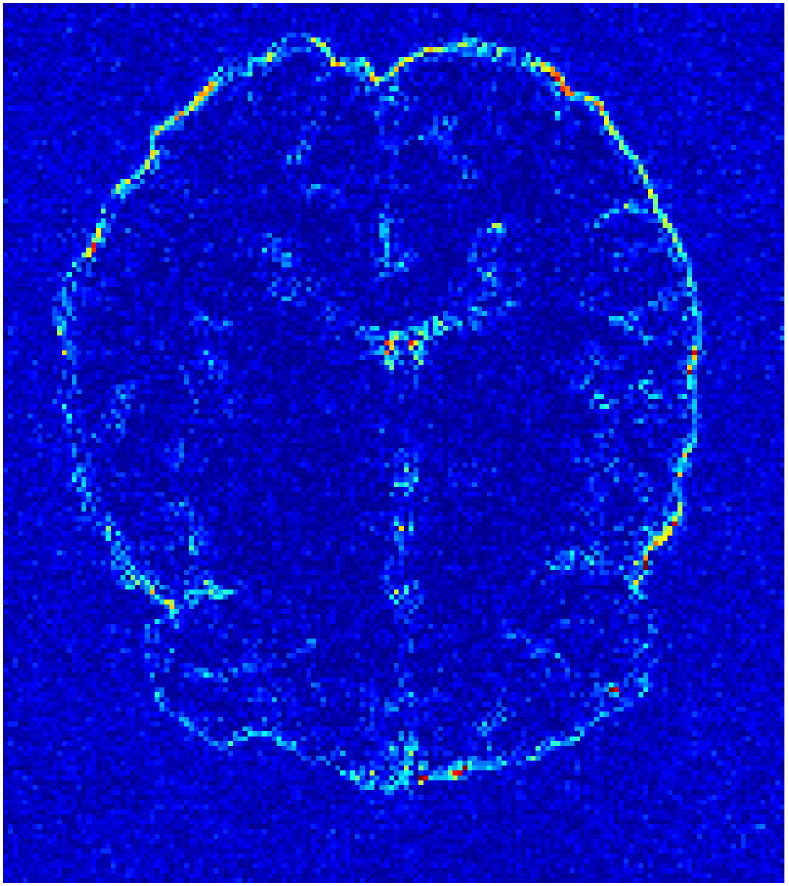}
\includegraphics[width=0.18\linewidth]{fig/colorbar.pdf}\\
\includegraphics[width=0.2\linewidth, angle=90]{fig/conventional_error.pdf}
\includegraphics[width=0.18\linewidth]{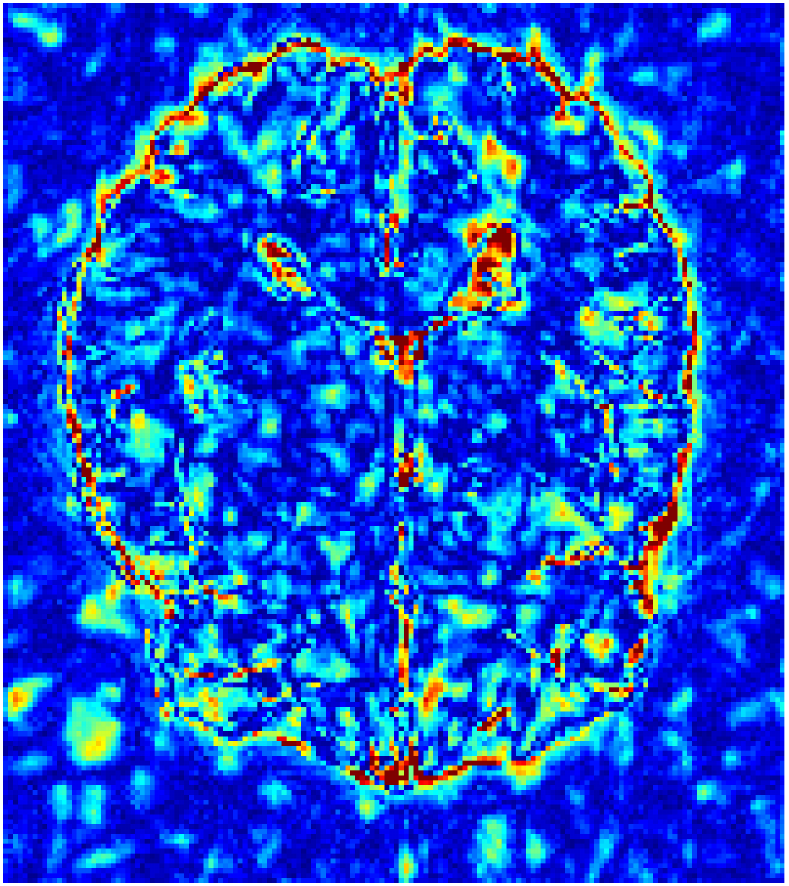}
\includegraphics[width=0.18\linewidth]{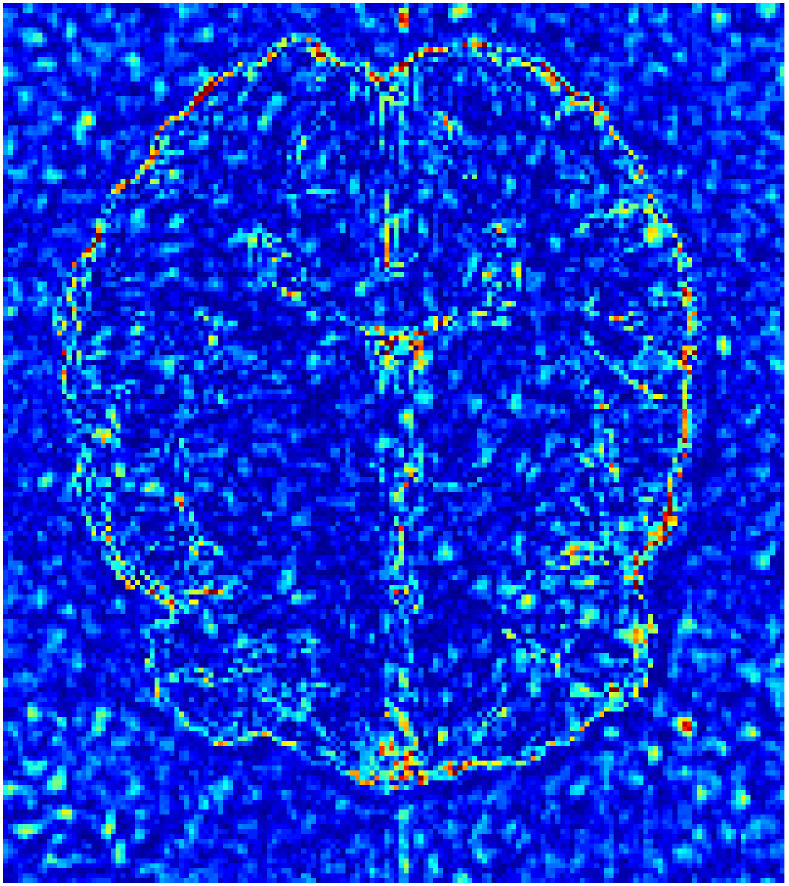}
\includegraphics[width=0.18\linewidth]{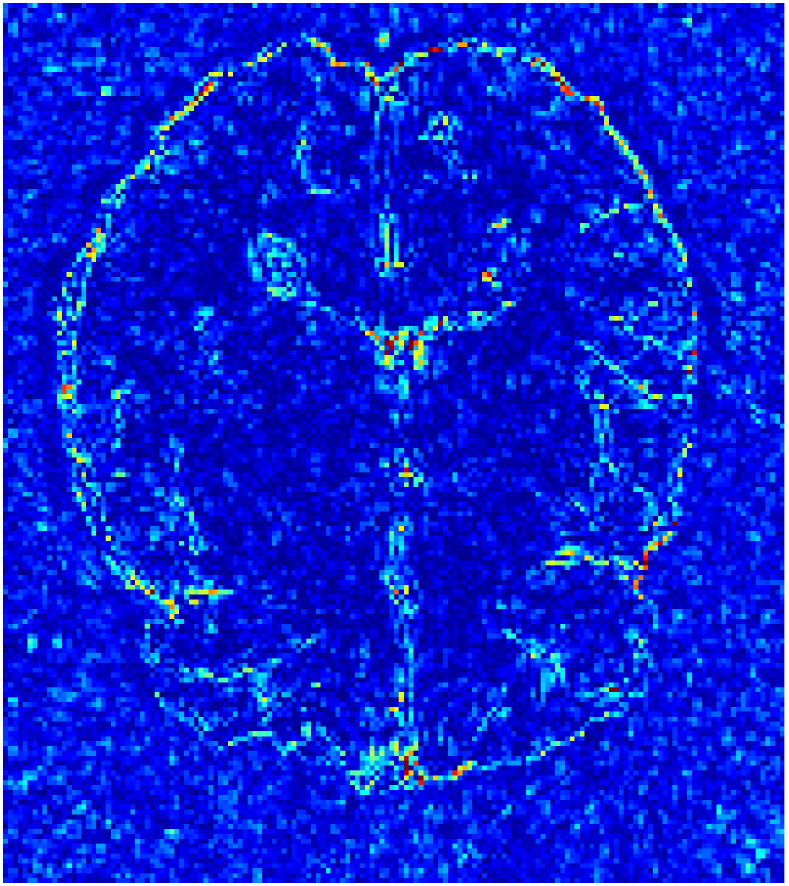}
\includegraphics[width=0.18\linewidth]{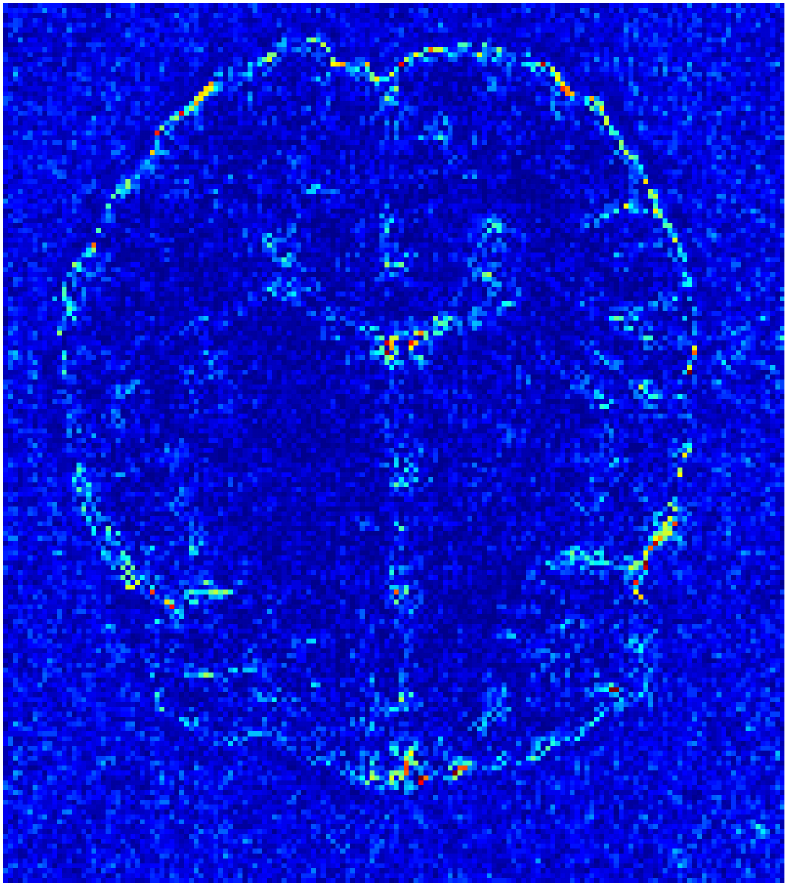}
\includegraphics[width=0.18\linewidth]{fig/white.pdf}\\
\includegraphics[width=0.2\linewidth, angle=90]{fig/masks.pdf}
\includegraphics[width=0.18\linewidth]{fig/mask10_t1.pdf}
\includegraphics[width=0.18\linewidth]{fig/mask20_t1.pdf}
\includegraphics[width=0.18\linewidth]{fig/mask30_t1.pdf}
\includegraphics[width=0.18\linewidth]{fig/mask40_t1.pdf}
\includegraphics[width=0.18\linewidth]{fig/white.pdf}
\caption{The pictures (from top to bottom) display the T2 Brain image reconstruction results, zoomed in details, pointwise errors with colorbar and associated \textbf{radio} masks for both these  two compared methods with four different CS ratios 10\%, 20\%, 30\%, 40\%（from left to right). The most top right one is ground truth fully-sampled image. }
\label{figure_same_ratio_t2}
\end{figure}

\begin{figure}[H]
\centering
\includegraphics[width=0.2\linewidth, angle=90]{fig/meta_result.pdf}
\includegraphics[width=0.18\linewidth]{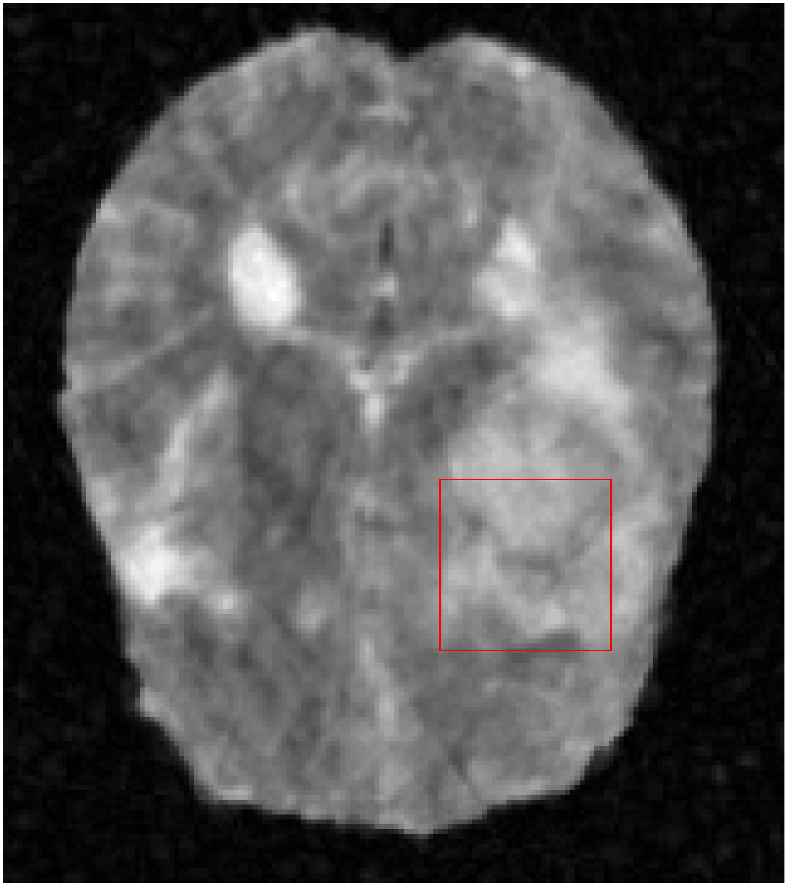}
\includegraphics[width=0.18\linewidth]{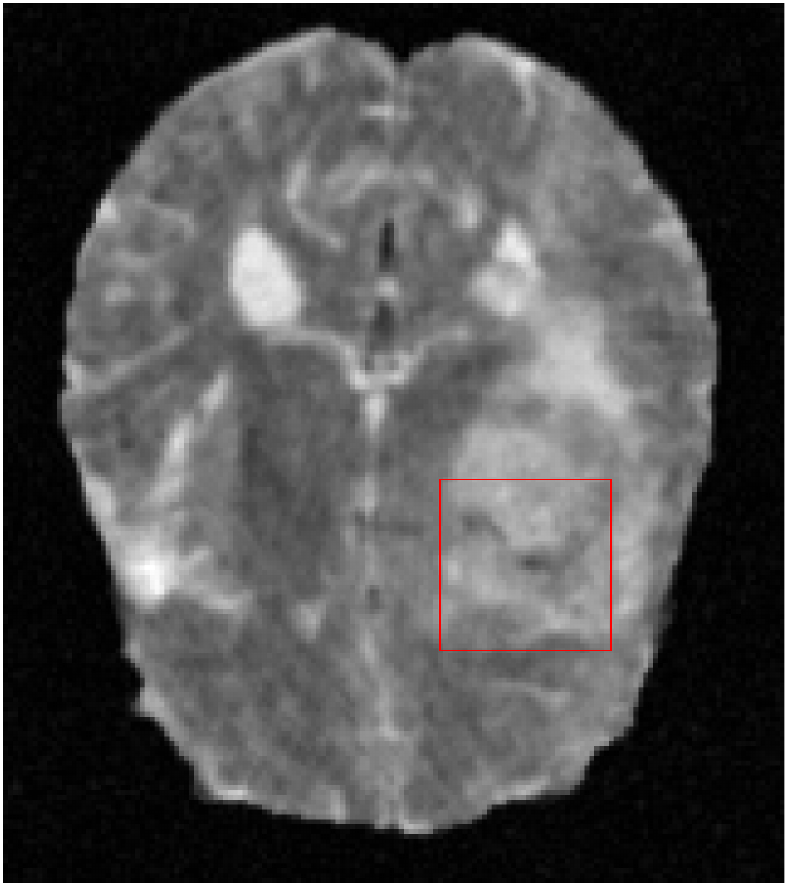}
\includegraphics[width=0.18\linewidth]{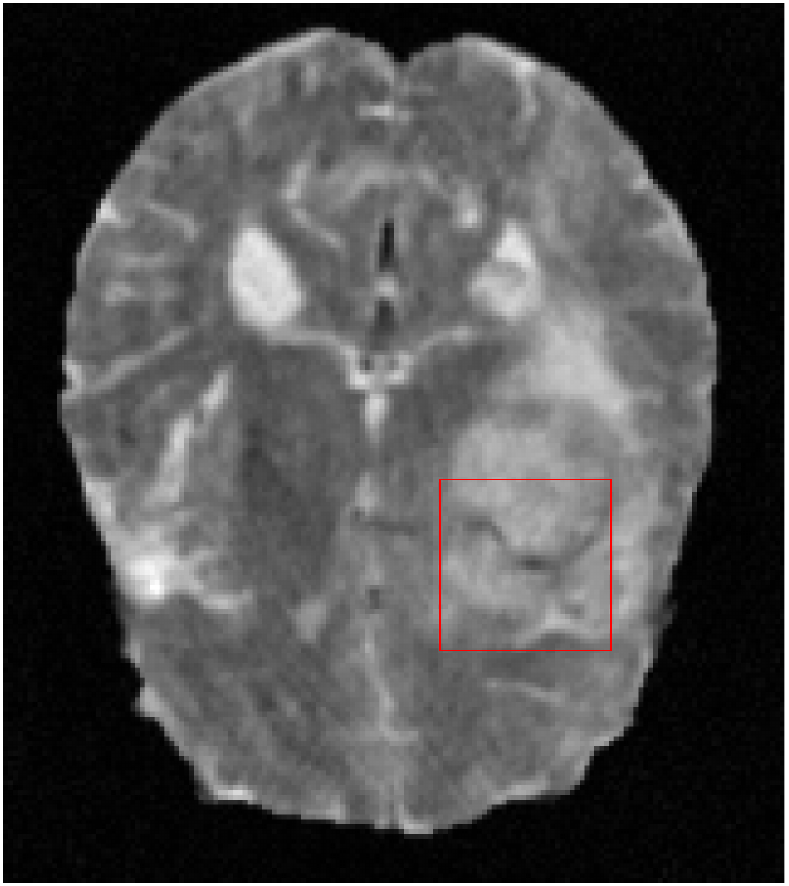}
\includegraphics[width=0.18\linewidth]{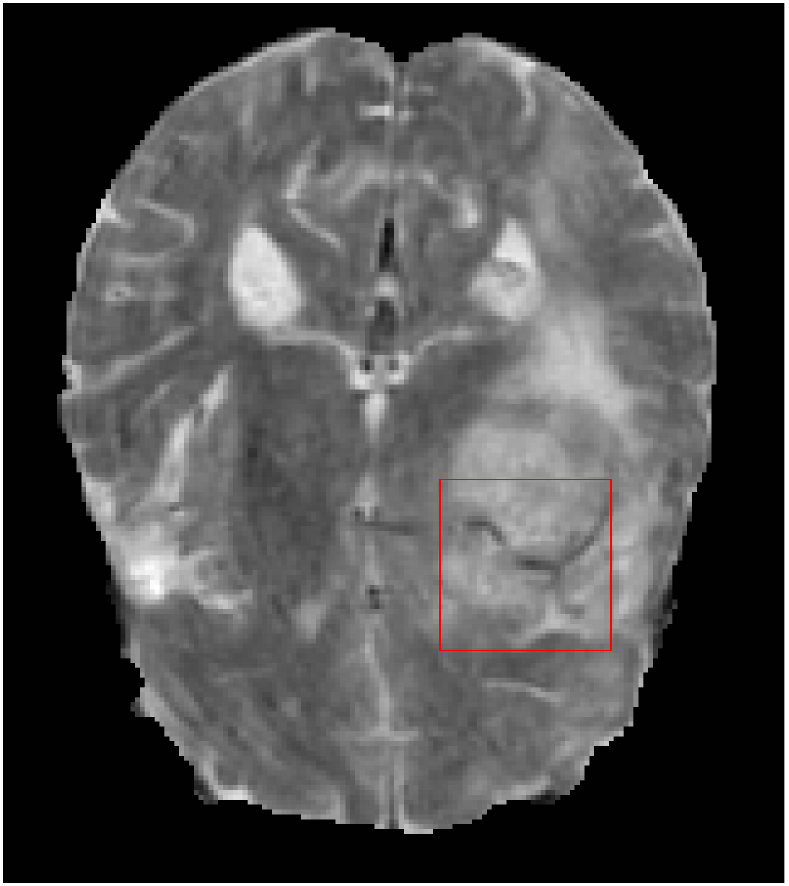}\\
\includegraphics[width=0.2\linewidth, angle=90]{fig/conventional_result.pdf}
\includegraphics[width=0.18\linewidth]{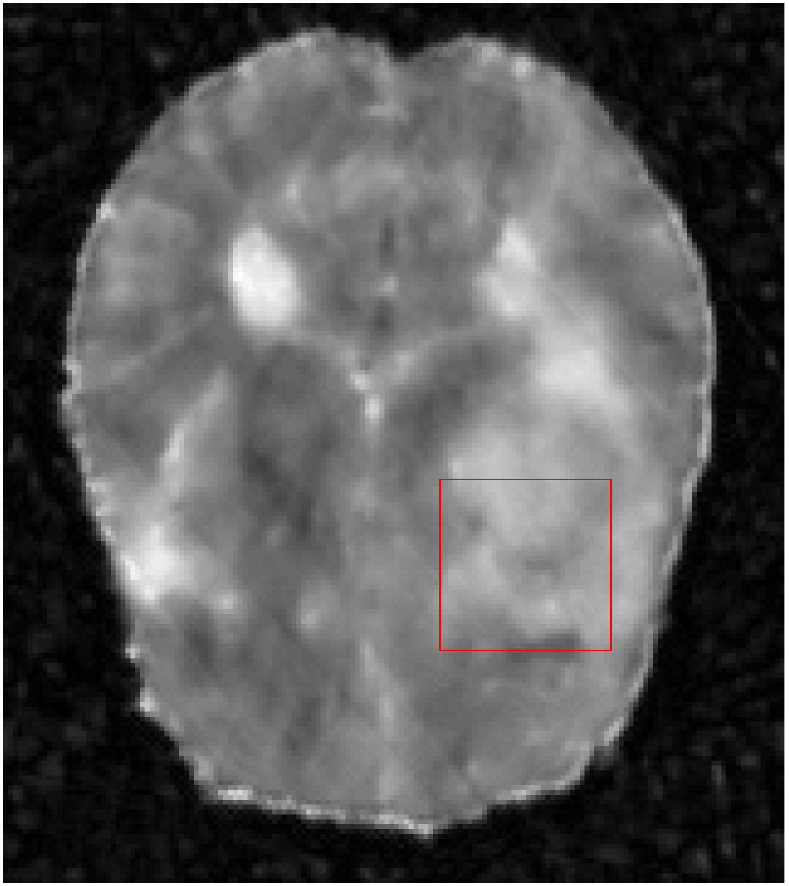}
\includegraphics[width=0.18\linewidth]{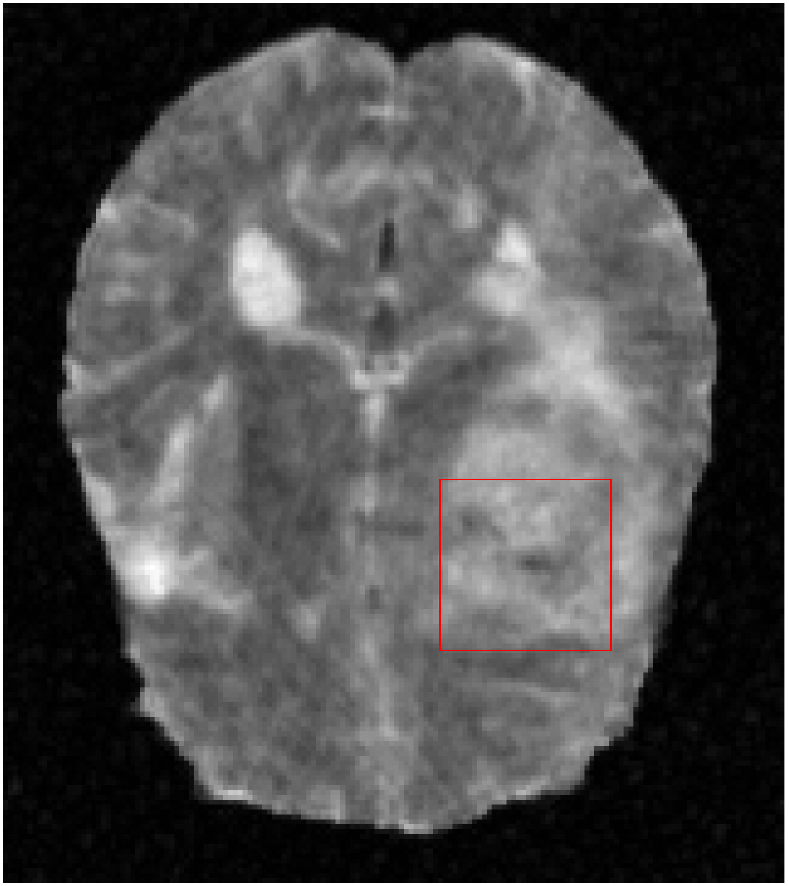}
\includegraphics[width=0.18\linewidth]{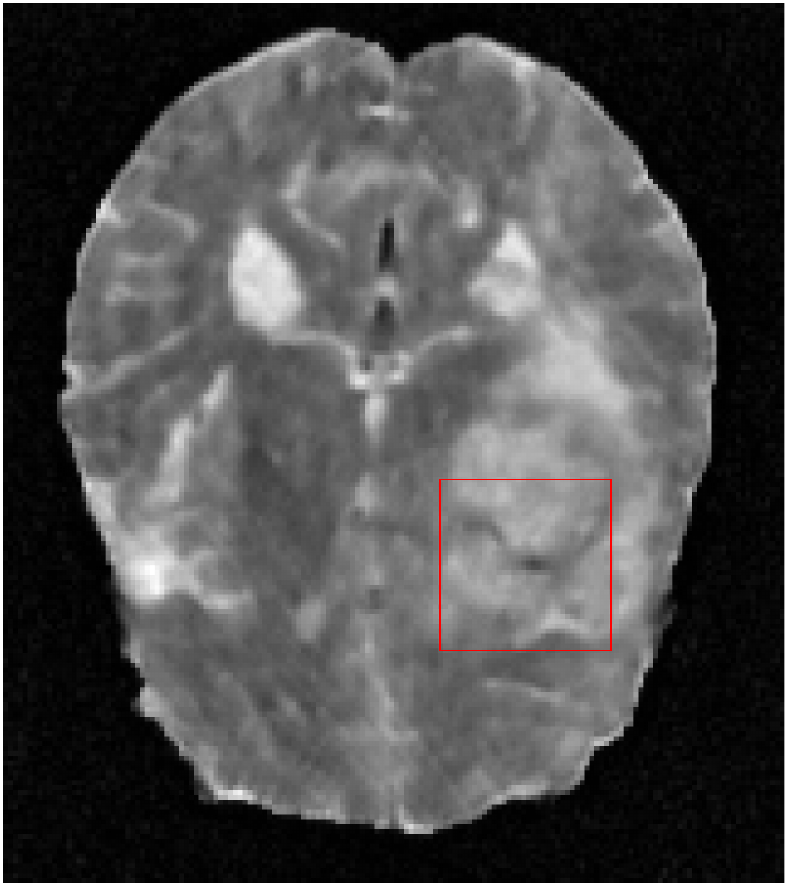}
\includegraphics[width=0.18\linewidth]{fig/white.pdf}\\
\includegraphics[width=0.2\linewidth, angle=90]{fig/meta_detail.pdf}
\includegraphics[width=0.18\linewidth]{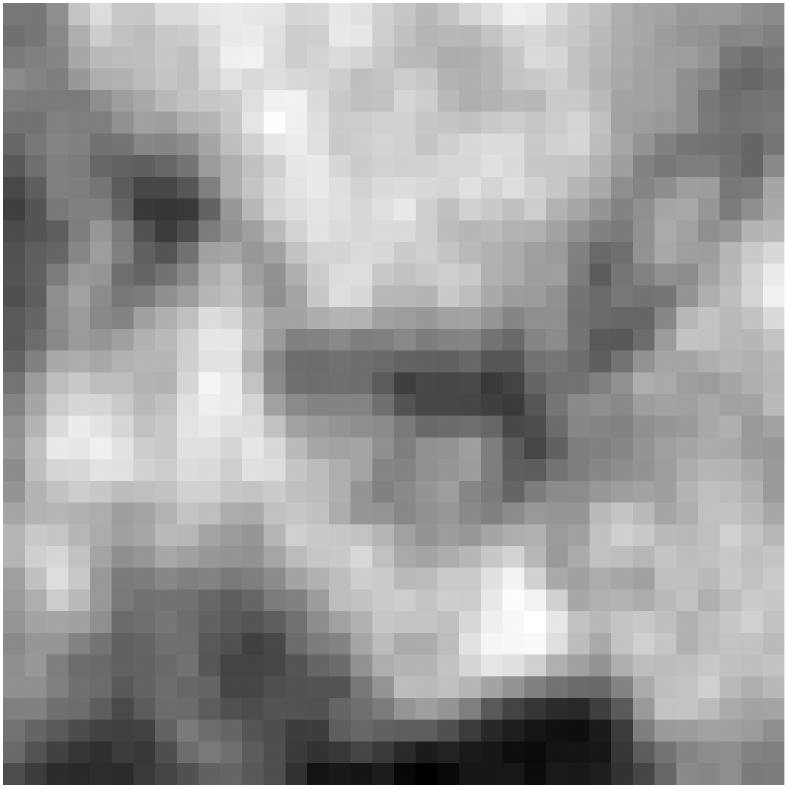}
\includegraphics[width=0.18\linewidth]{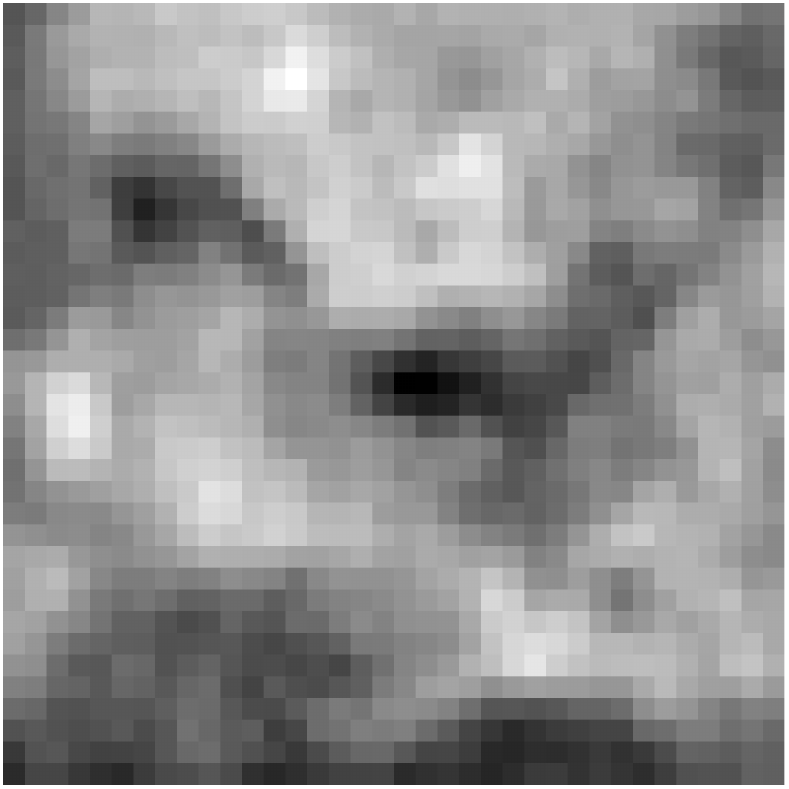}
\includegraphics[width=0.18\linewidth]{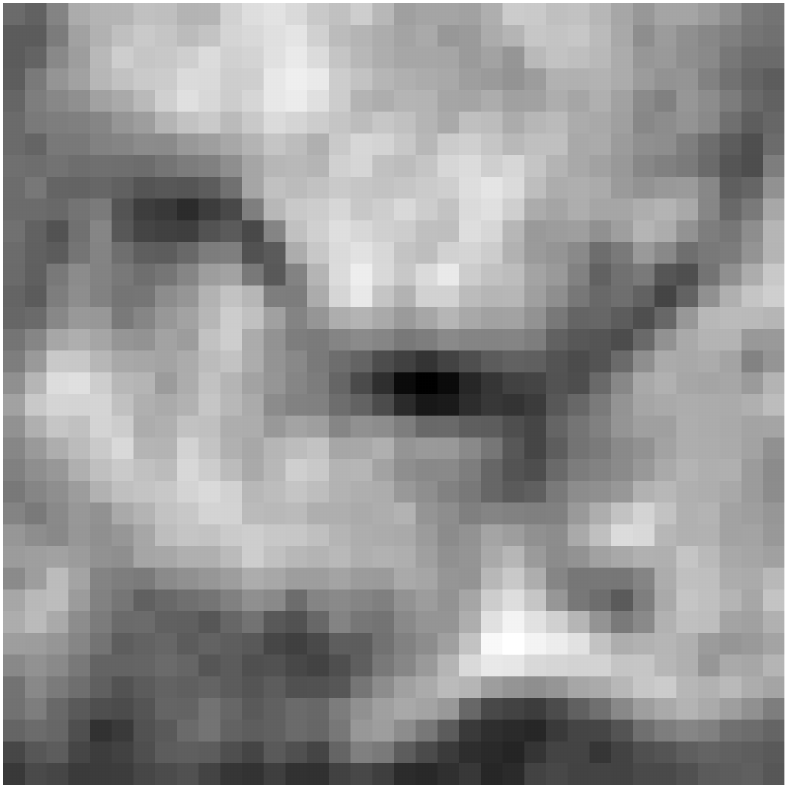}
\includegraphics[width=0.18\linewidth]{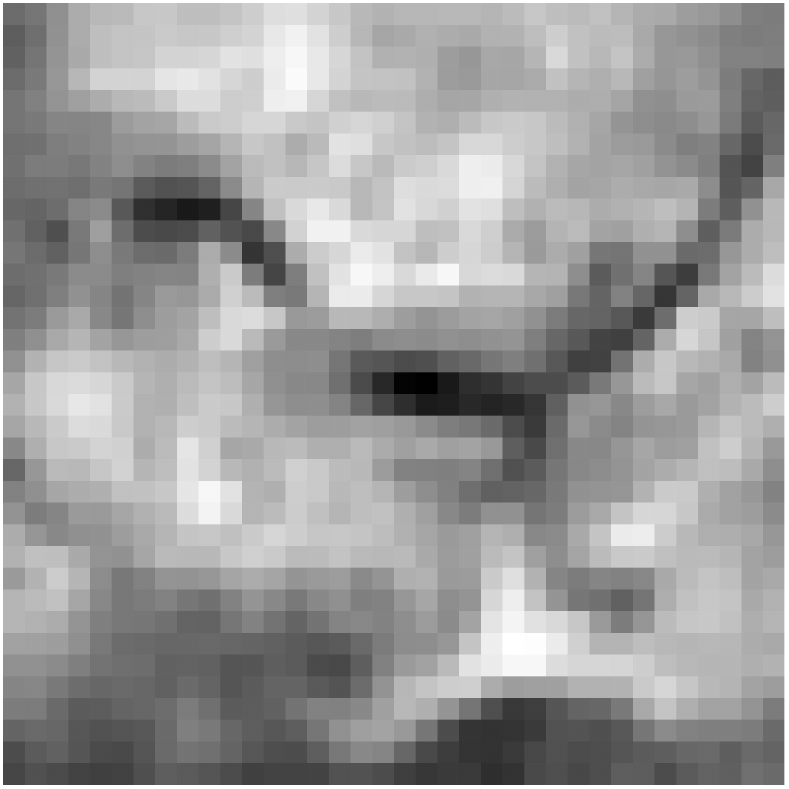}\\
\includegraphics[width=0.2\linewidth, angle=90]{fig/conventional_detail.pdf}
\includegraphics[width=0.18\linewidth]{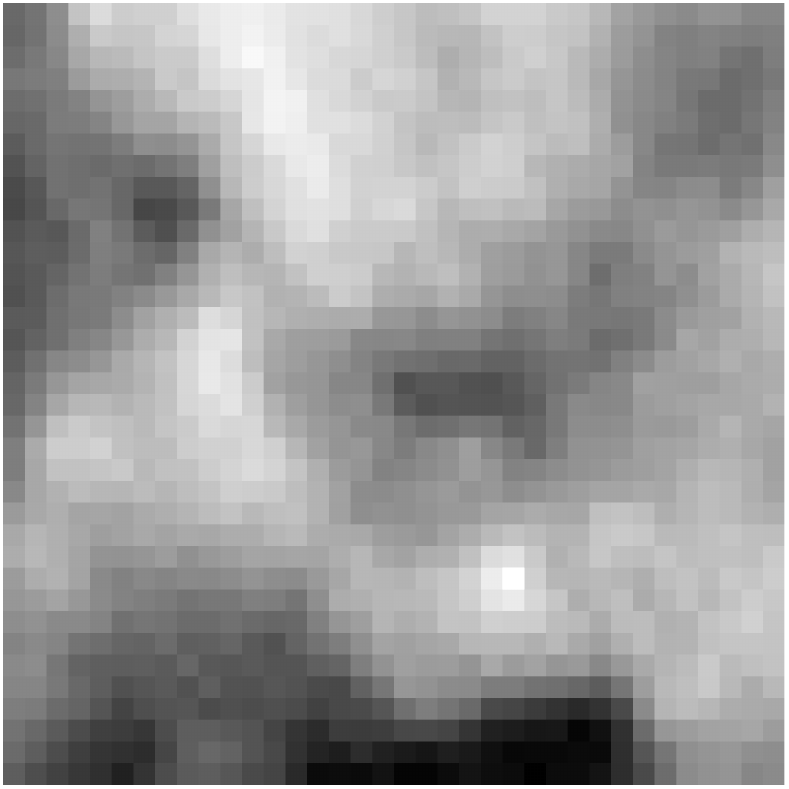}
\includegraphics[width=0.18\linewidth]{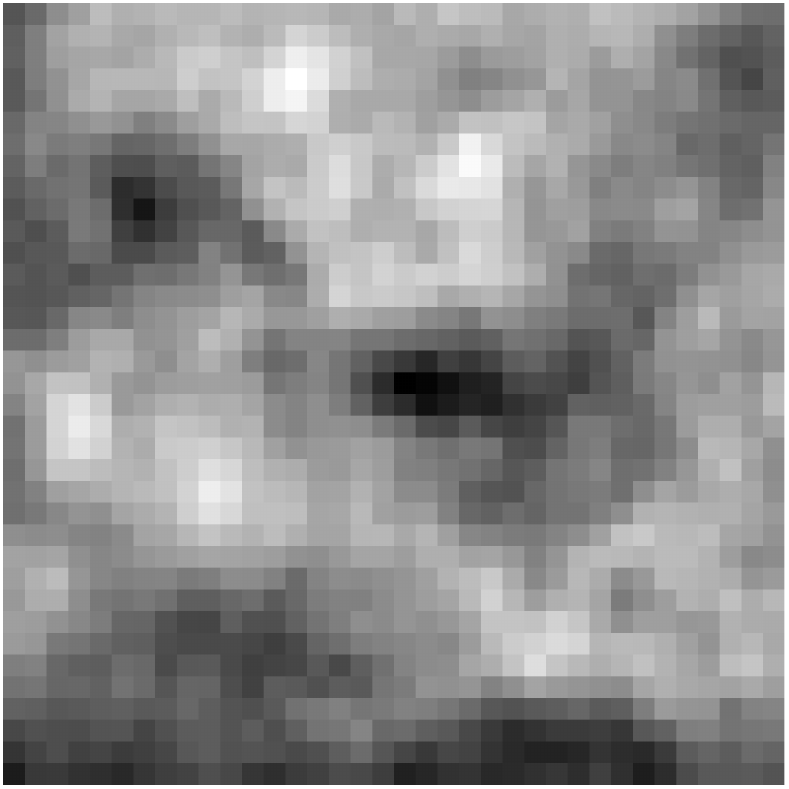}
\includegraphics[width=0.18\linewidth]{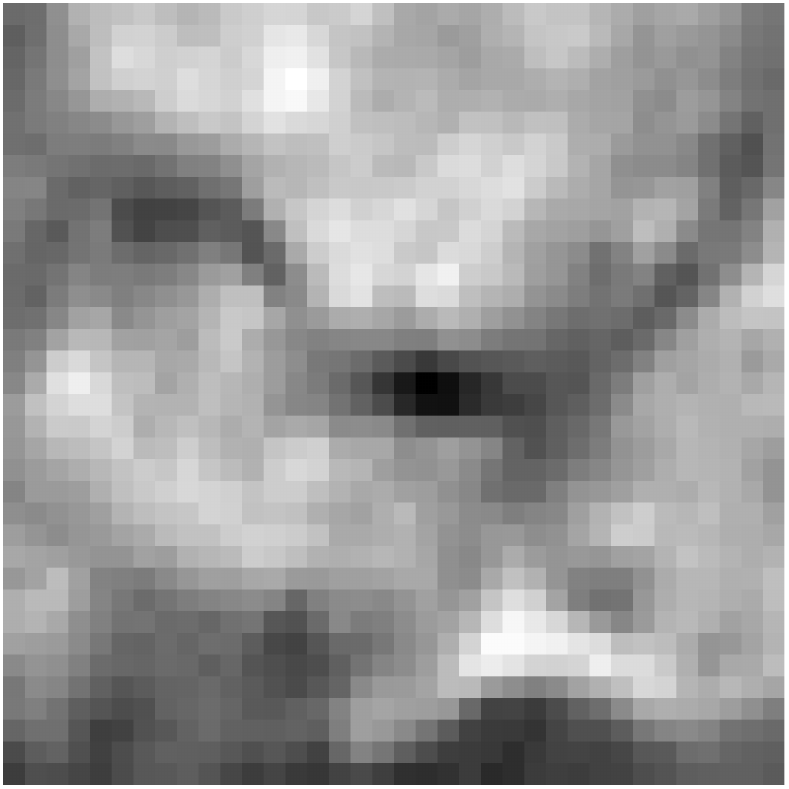}
\includegraphics[width=0.18\linewidth]{fig/white.pdf}\\
\includegraphics[width=0.2\linewidth, angle=90]{fig/meta_error.pdf}
\includegraphics[width=0.18\linewidth]{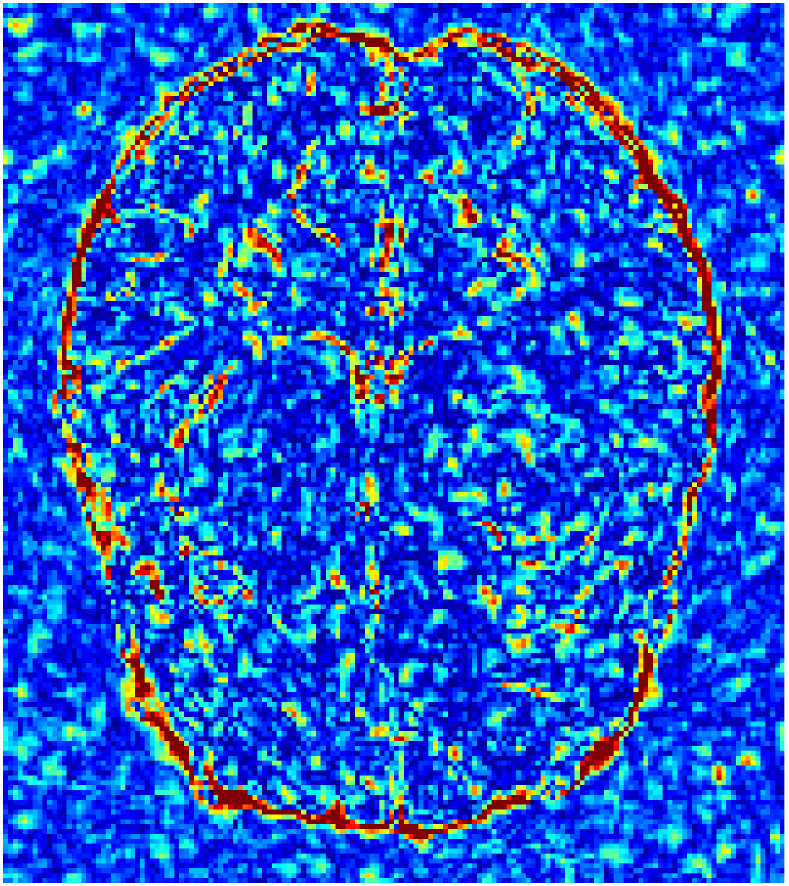}
\includegraphics[width=0.18\linewidth]{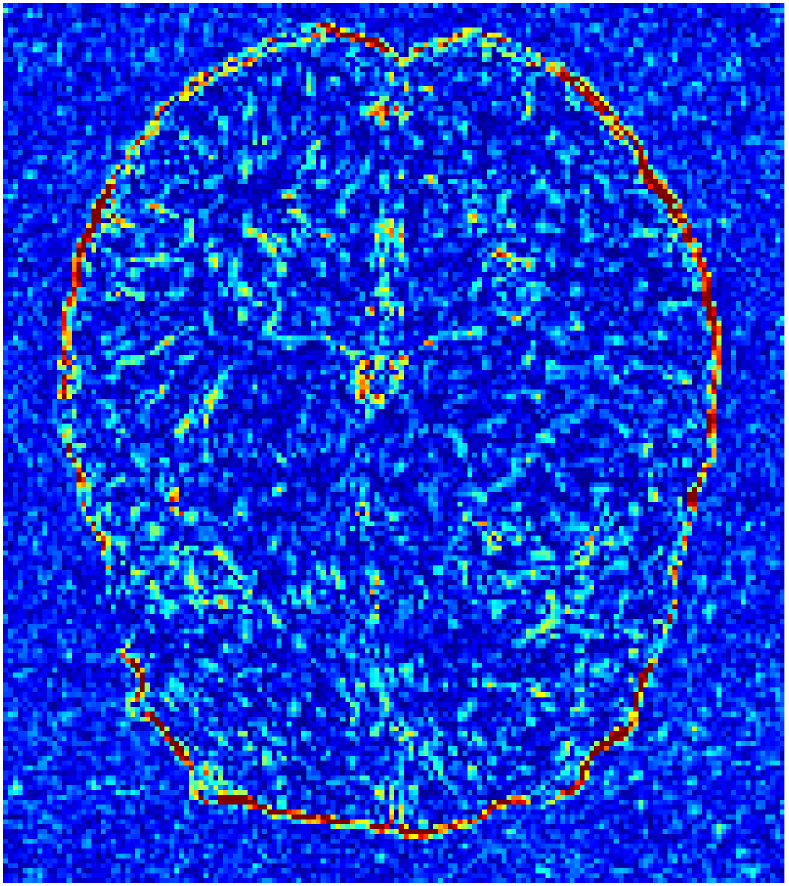}
\includegraphics[width=0.18\linewidth]{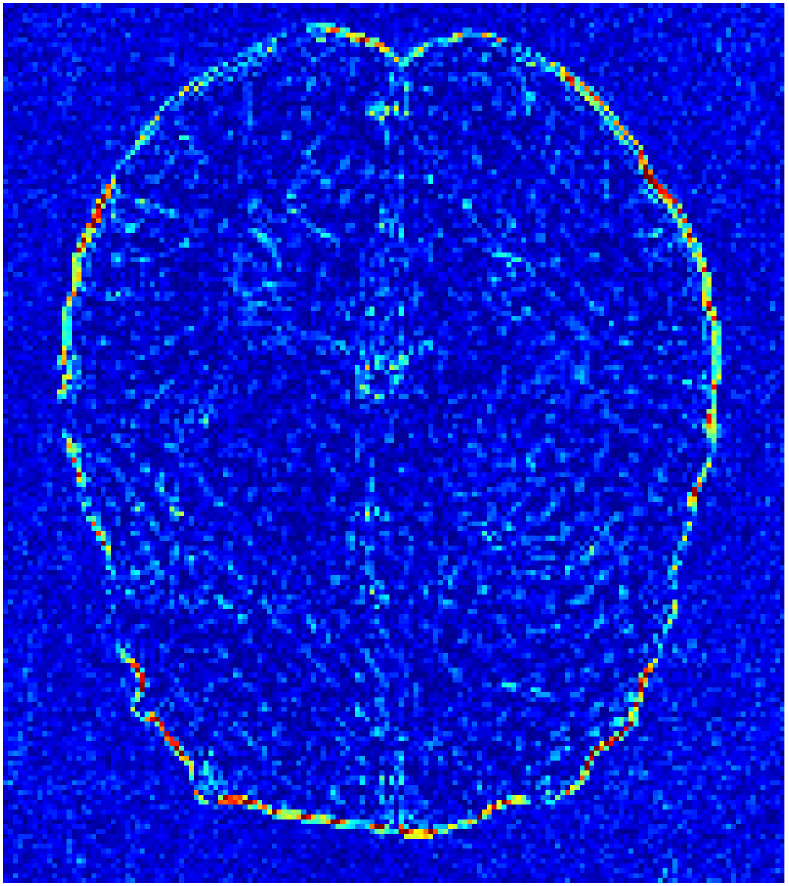}
\includegraphics[width=0.18\linewidth]{fig/colorbar.pdf}\\
\includegraphics[width=0.2\linewidth, angle=90]{fig/conventional_error.pdf}
\includegraphics[width=0.18\linewidth]{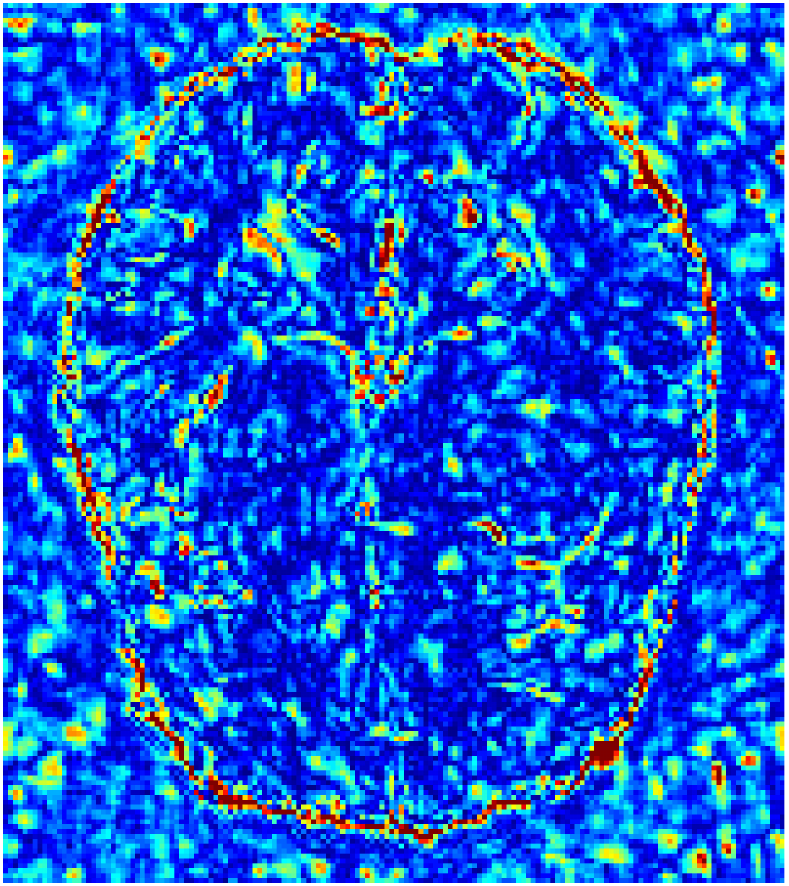}
\includegraphics[width=0.18\linewidth]{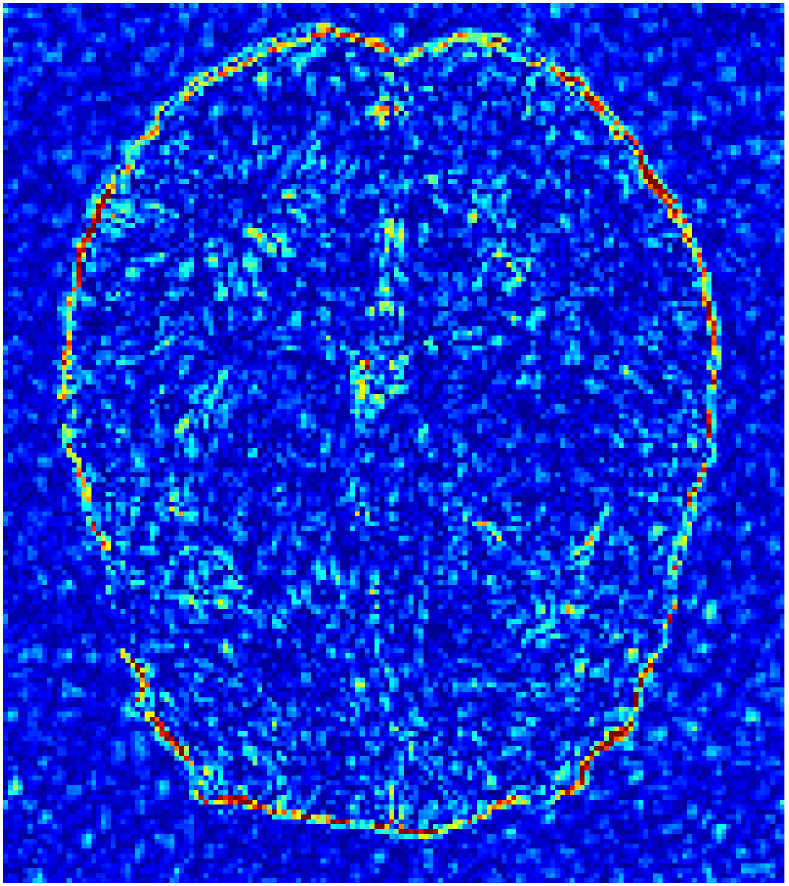}
\includegraphics[width=0.18\linewidth]{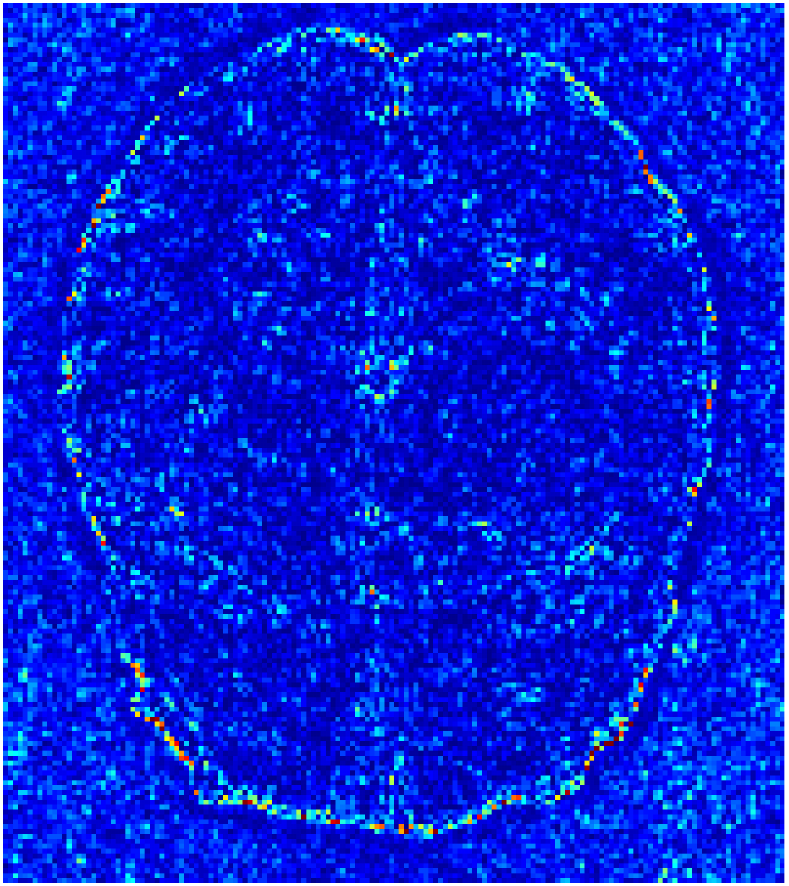}
\includegraphics[width=0.18\linewidth]{fig/white.pdf}\\
\includegraphics[width=0.2\linewidth, angle=90]{fig/masks.pdf}
\includegraphics[width=0.18\linewidth]{fig/mask15_t1.pdf}
\includegraphics[width=0.18\linewidth]{fig/mask25_t1.pdf}
\includegraphics[width=0.18\linewidth]{fig/mask35_t1.pdf}
\includegraphics[width=0.18\linewidth]{fig/white.pdf}
\caption{The pictures (from top to bottom) display the  T2 Brain image reconstruction results, zoomed in details, pointwise errors with colorbar and associated \textbf{radio} masks for meta-learning and conventional learning. Meta-learning was trained with CS ratios 10\%, 20\%, 30\% 40\% and test with three different CS ratios 15\%, 25\% and 35\%（from left to right). Conventional learning was trained and test with same CS ratios 15\%, 25\% and 35\%. The most top right one is ground truth fully-sampled image. }
\label{figure_dif_ratio_t2}
\end{figure}

\begin{figure}[H]
\centering
\includegraphics[width=0.2\linewidth, angle=90]{fig/meta_result.pdf}
\includegraphics[width=0.18\linewidth]{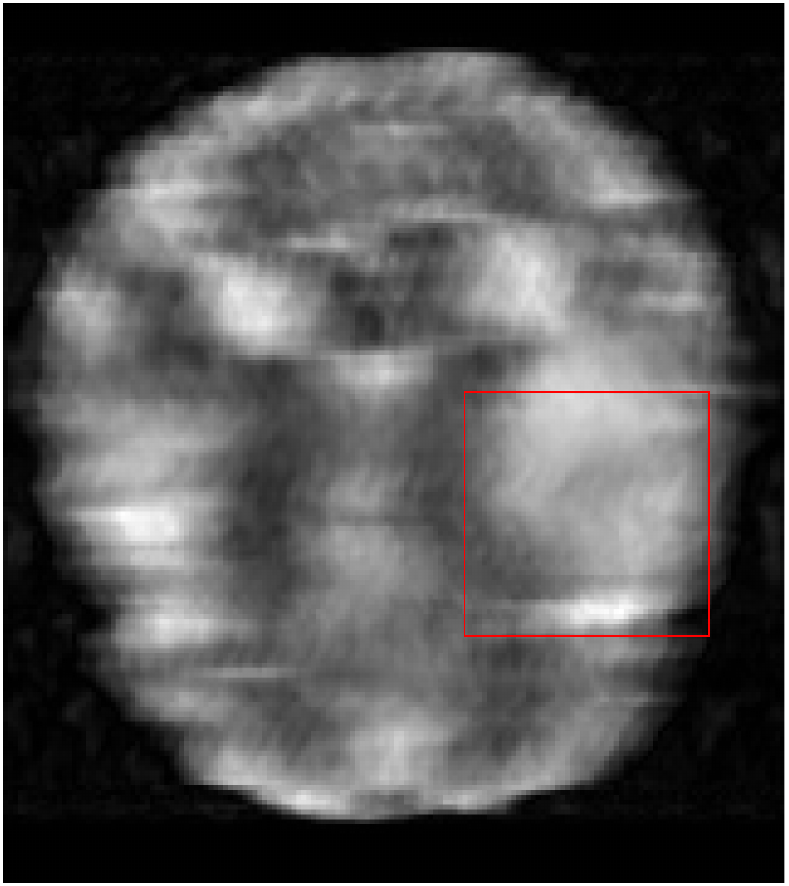}
\includegraphics[width=0.18\linewidth]{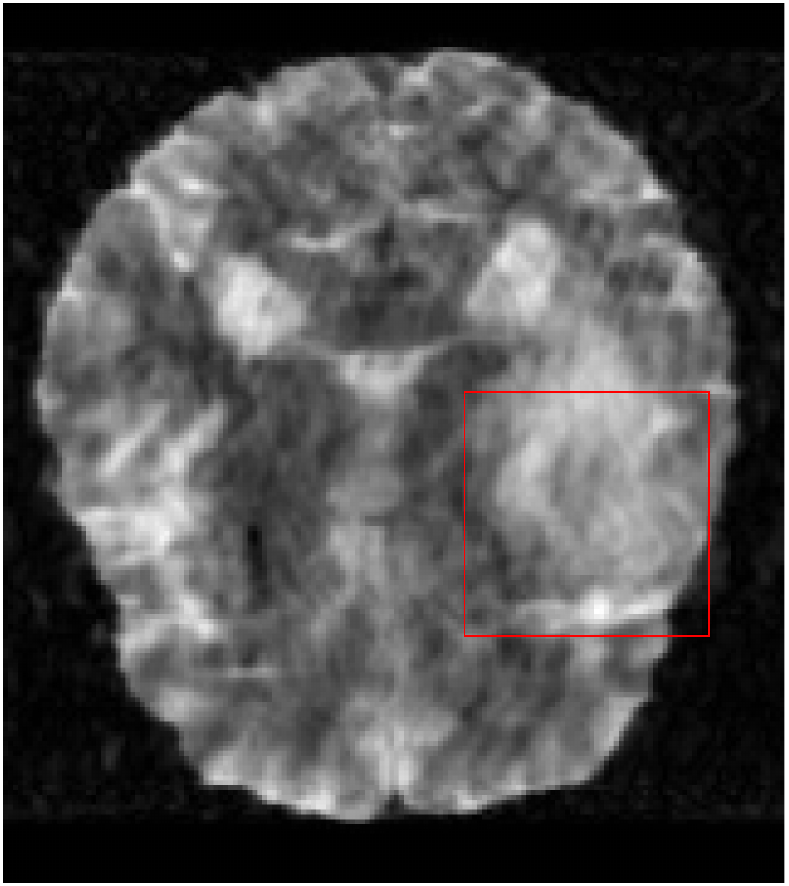}
\includegraphics[width=0.18\linewidth]{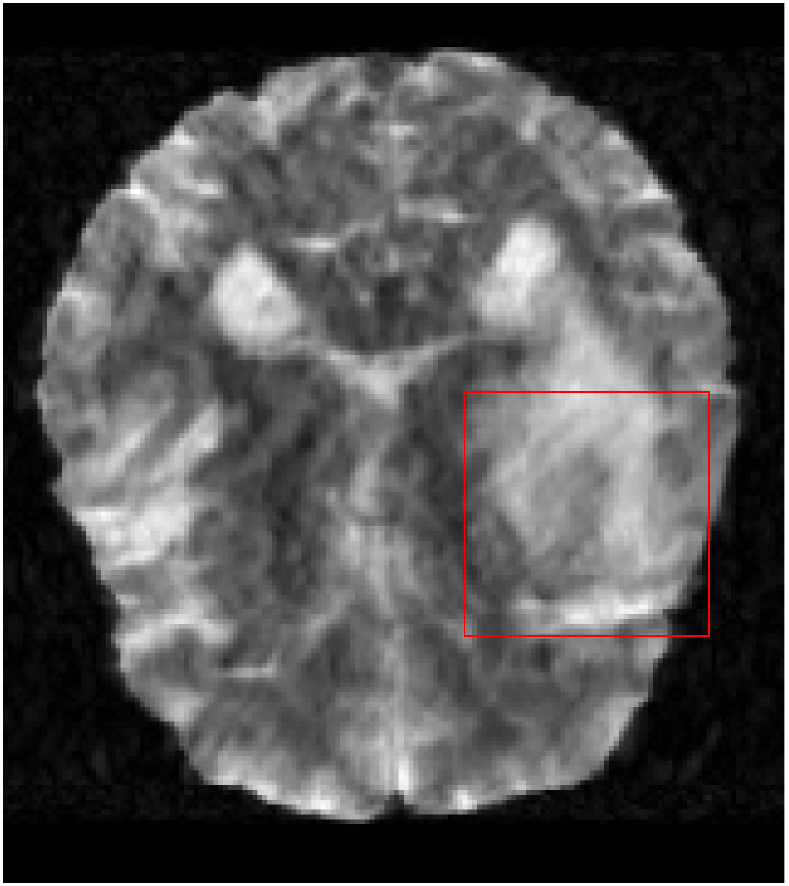}
\includegraphics[width=0.18\linewidth]{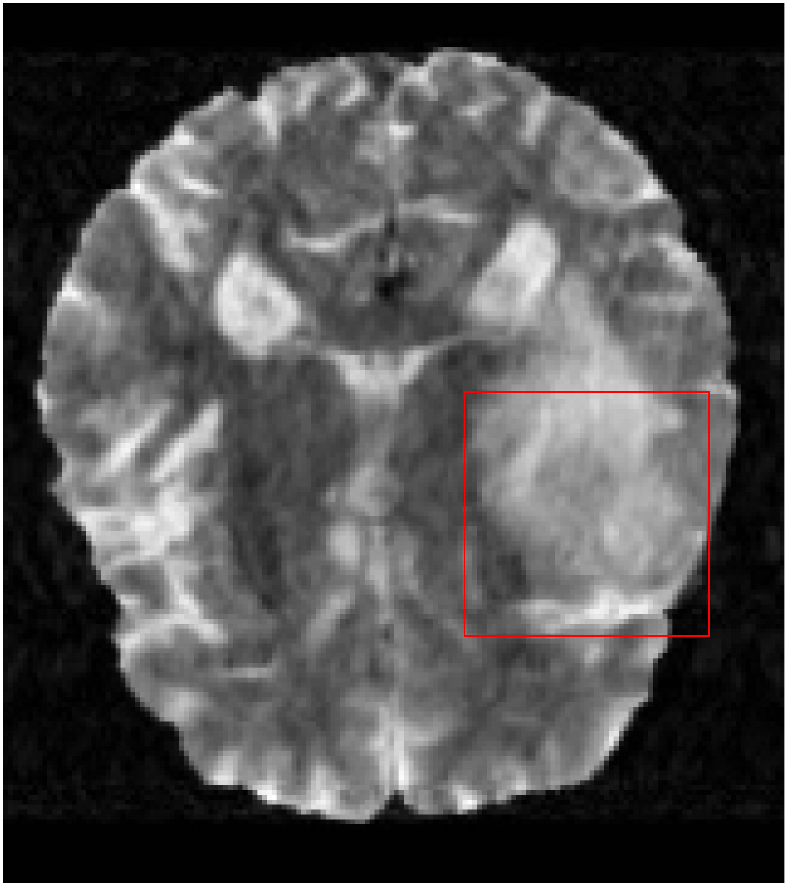}
\includegraphics[width=0.18\linewidth]{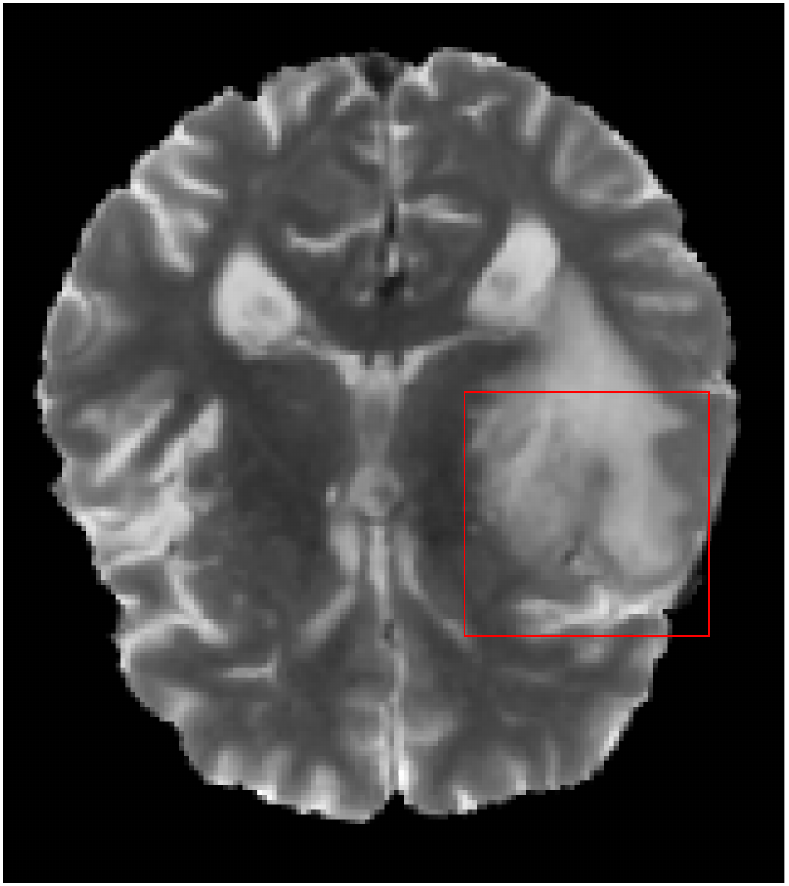}\\
\includegraphics[width=0.2\linewidth, angle=90]{fig/conventional_result.pdf}
\includegraphics[width=0.18\linewidth]{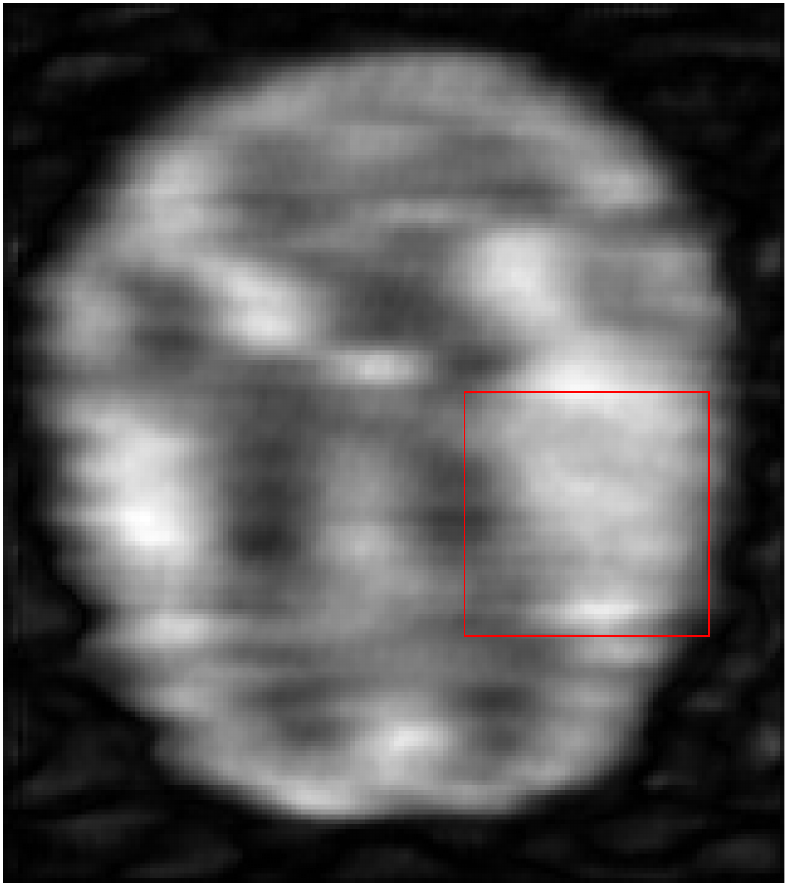}
\includegraphics[width=0.18\linewidth]{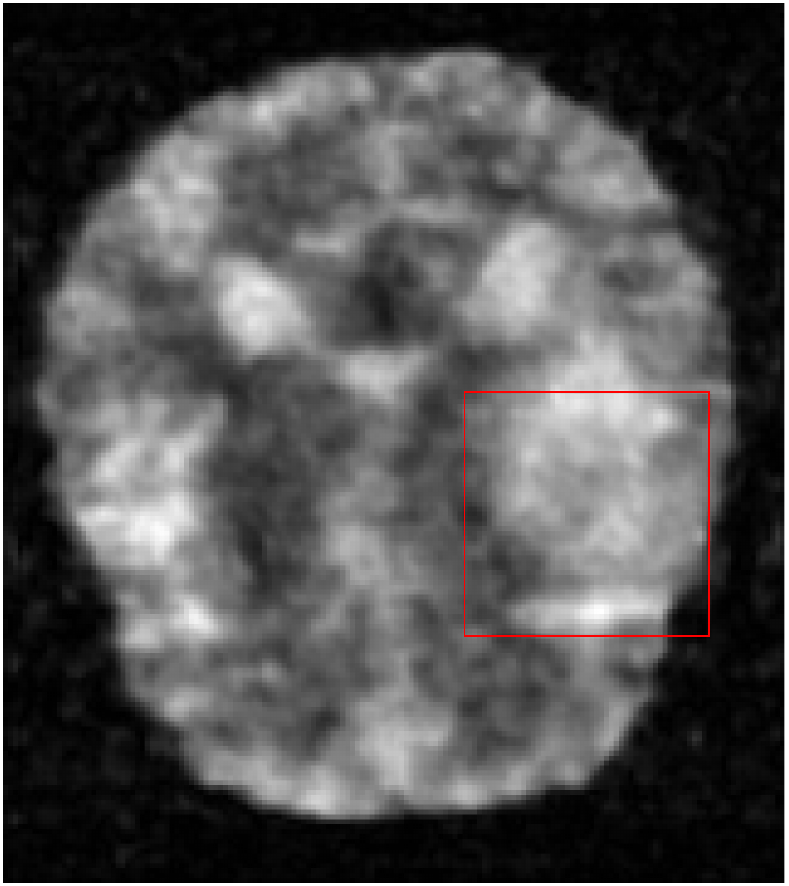}
\includegraphics[width=0.18\linewidth]{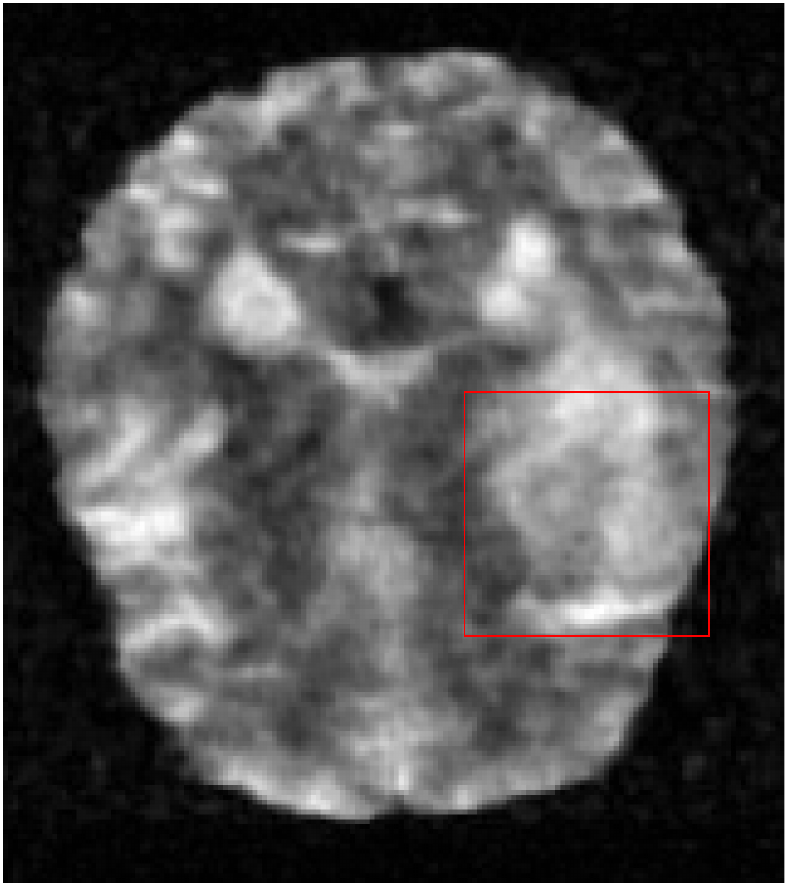}
\includegraphics[width=0.18\linewidth]{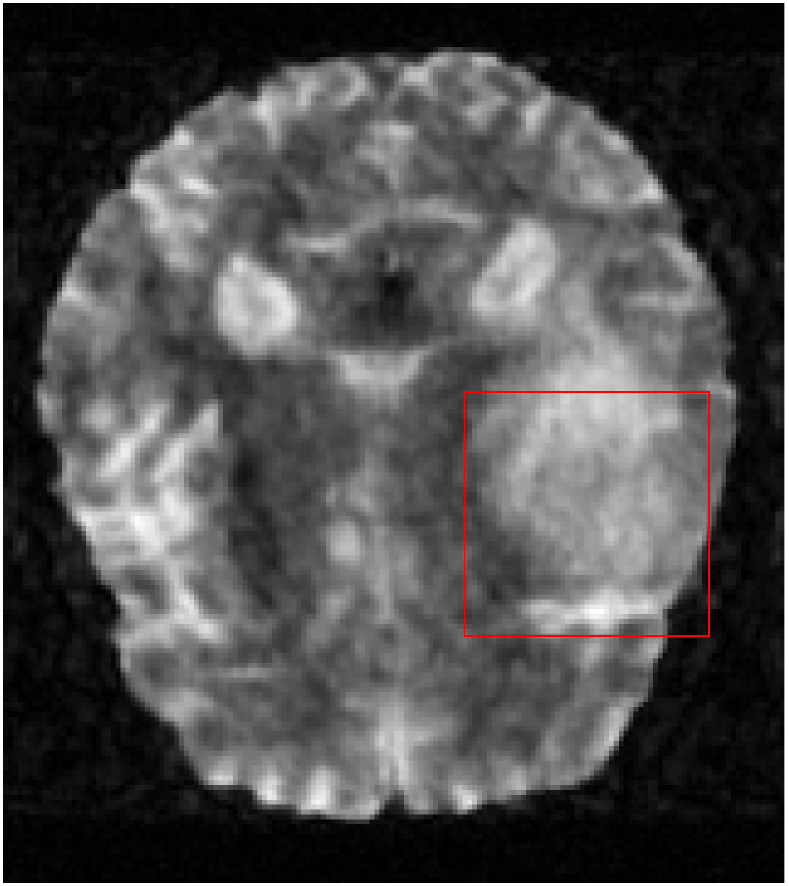}
\includegraphics[width=0.18\linewidth]{fig/white.pdf}\\
\includegraphics[width=0.2\linewidth, angle=90]{fig/meta_detail.pdf}
\includegraphics[width=0.18\linewidth]{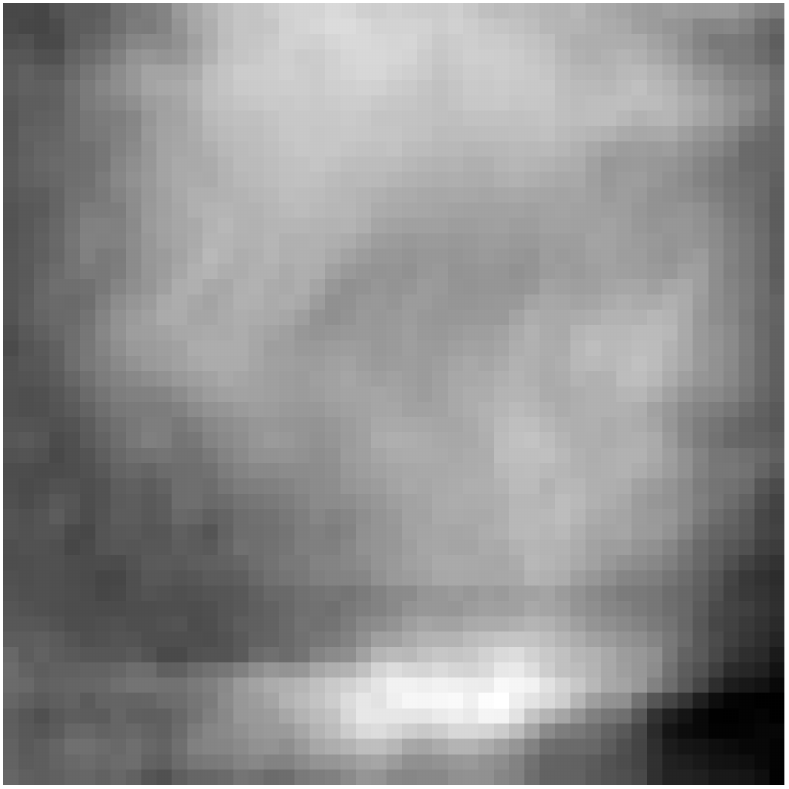}
\includegraphics[width=0.18\linewidth]{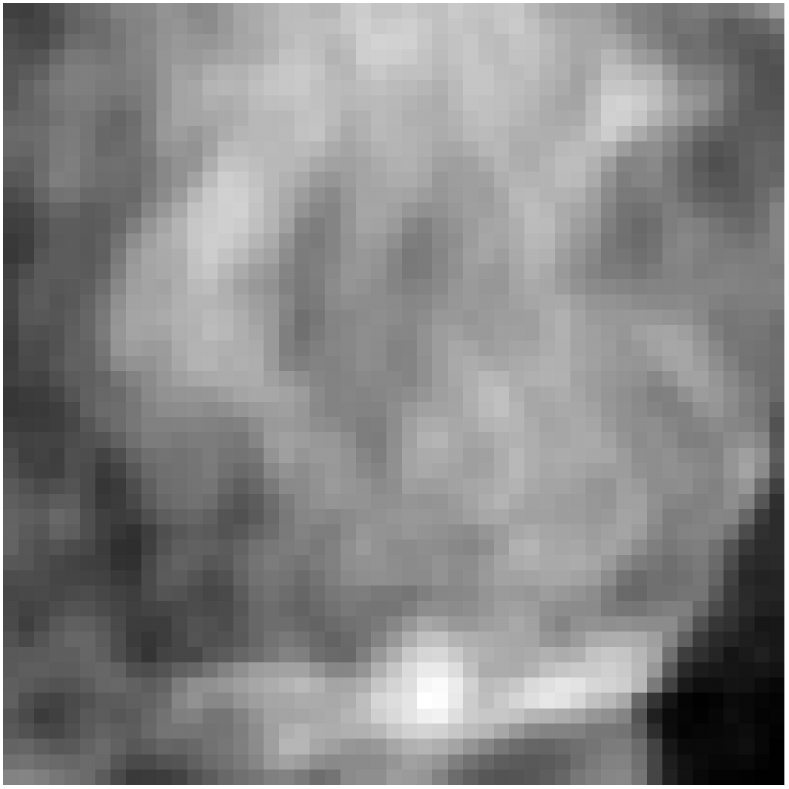}
\includegraphics[width=0.18\linewidth]{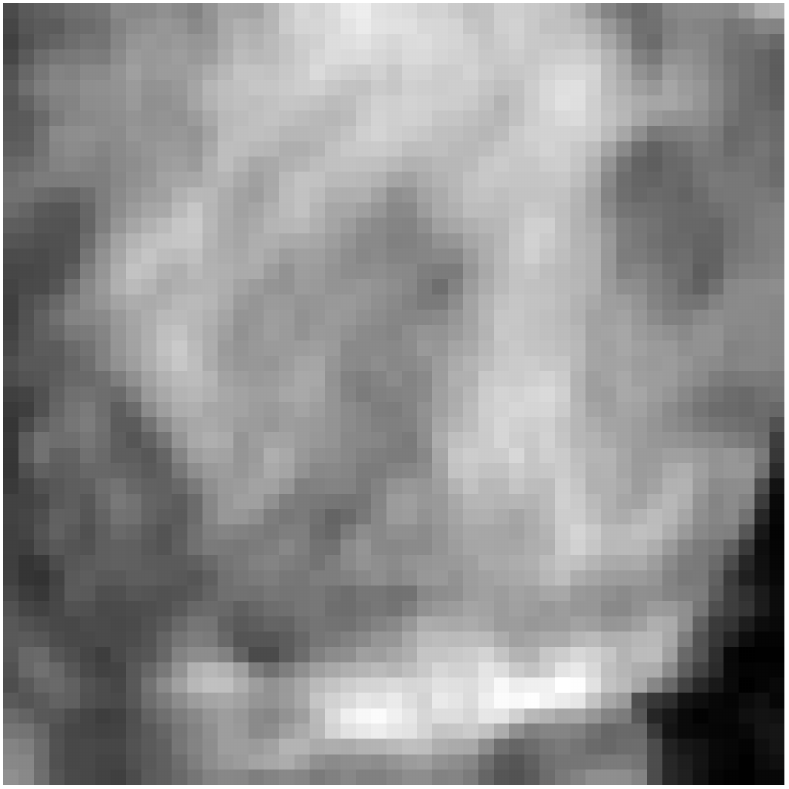}
\includegraphics[width=0.18\linewidth]{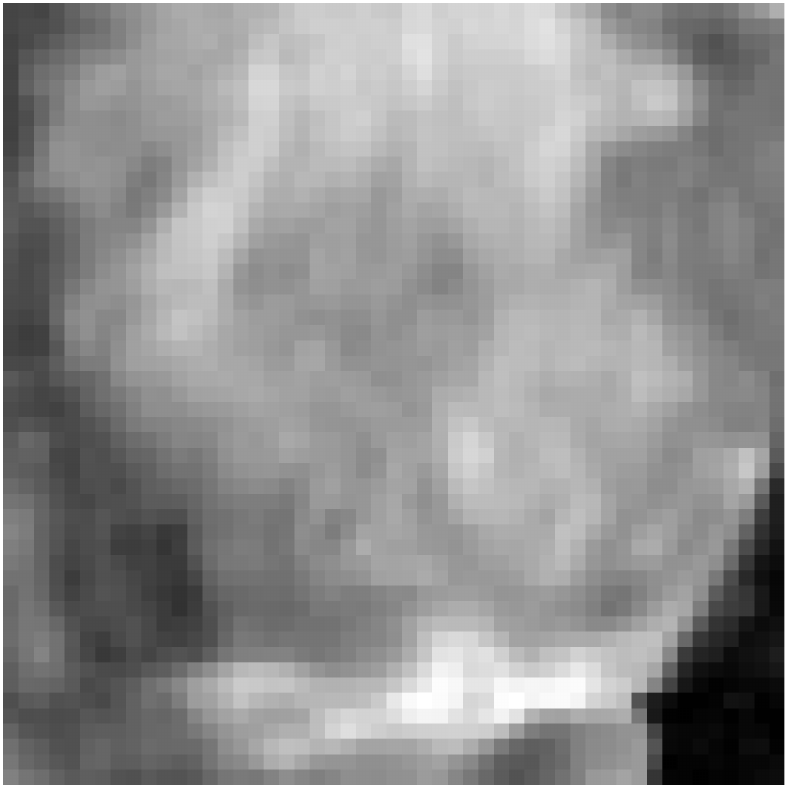}
\includegraphics[width=0.18\linewidth]{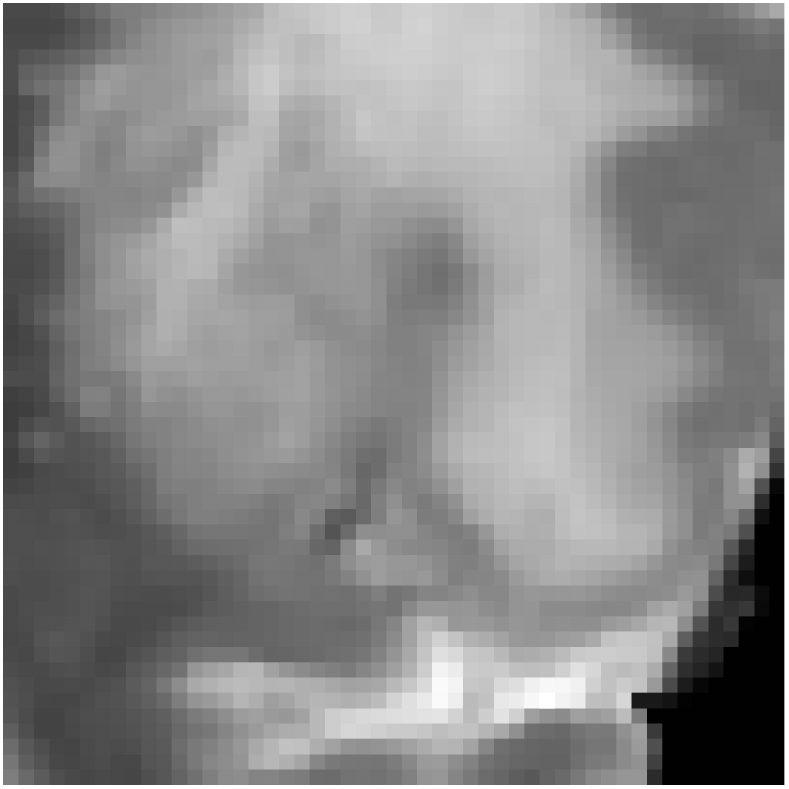}\\
\includegraphics[width=0.2\linewidth, angle=90]{fig/conventional_detail.pdf}
\includegraphics[width=0.18\linewidth]{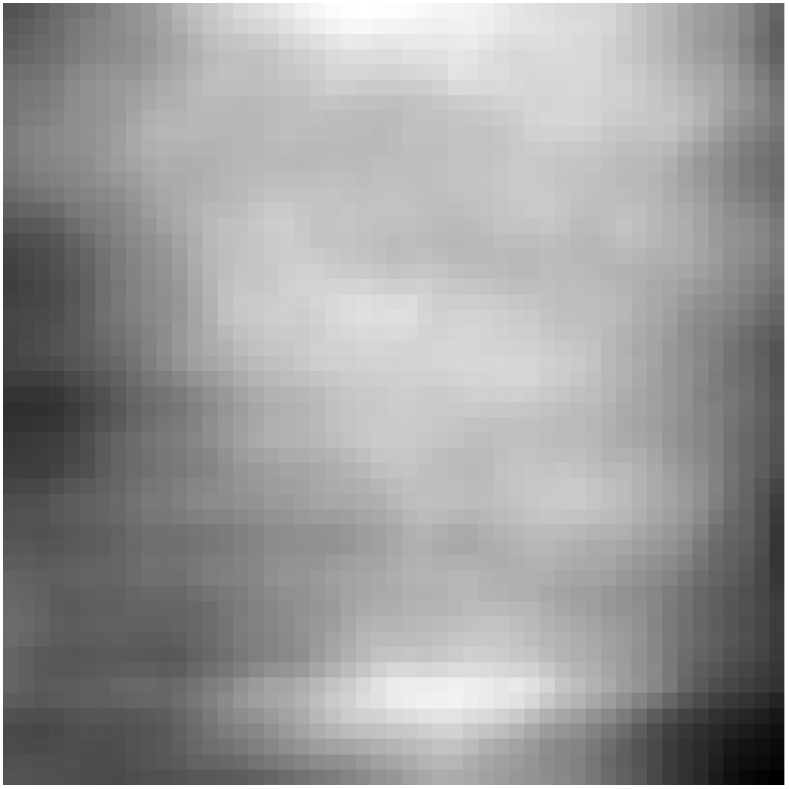}
\includegraphics[width=0.18\linewidth]{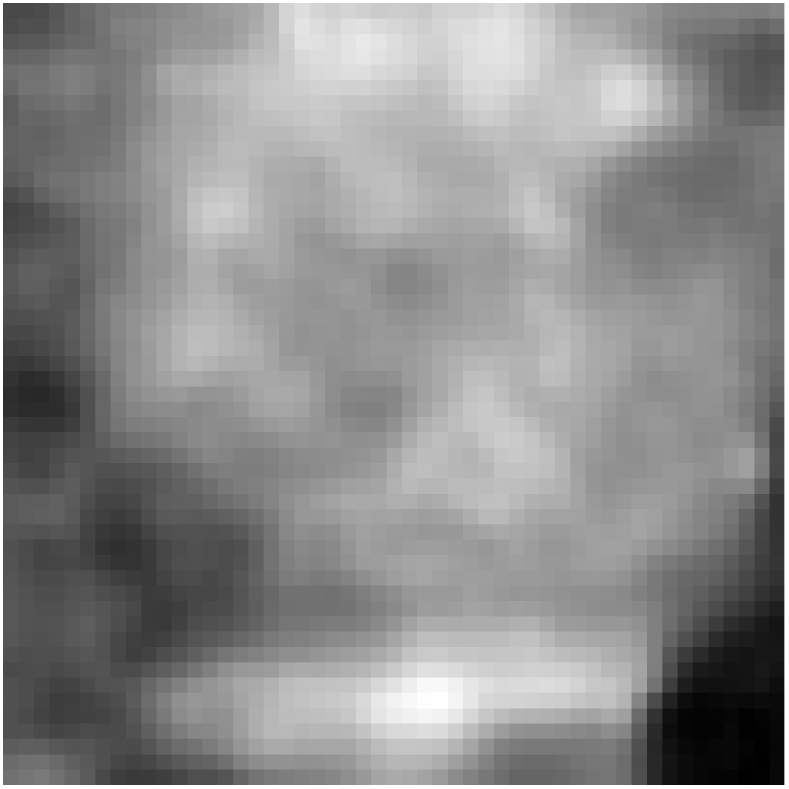}
\includegraphics[width=0.18\linewidth]{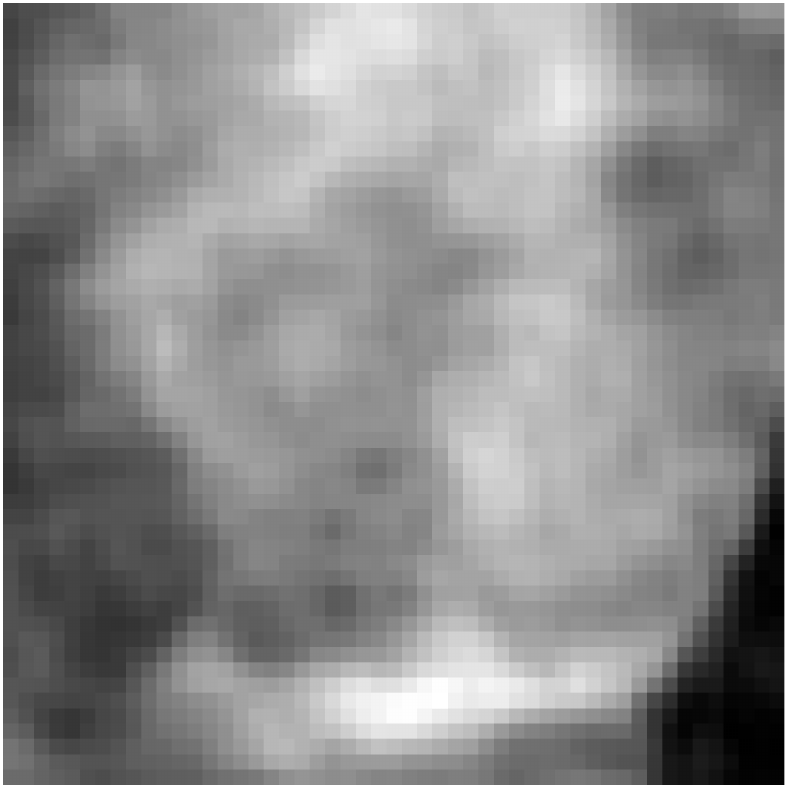}
\includegraphics[width=0.18\linewidth]{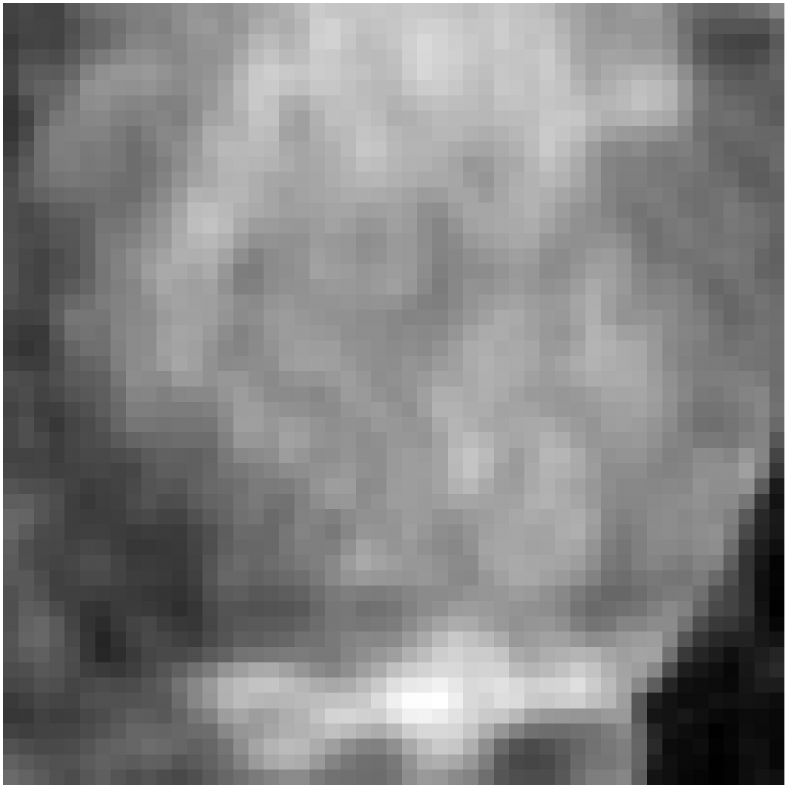}
\includegraphics[width=0.18\linewidth]{fig/white.pdf}\\
\includegraphics[width=0.2\linewidth, angle=90]{fig/meta_error.pdf}
\includegraphics[width=0.18\linewidth]{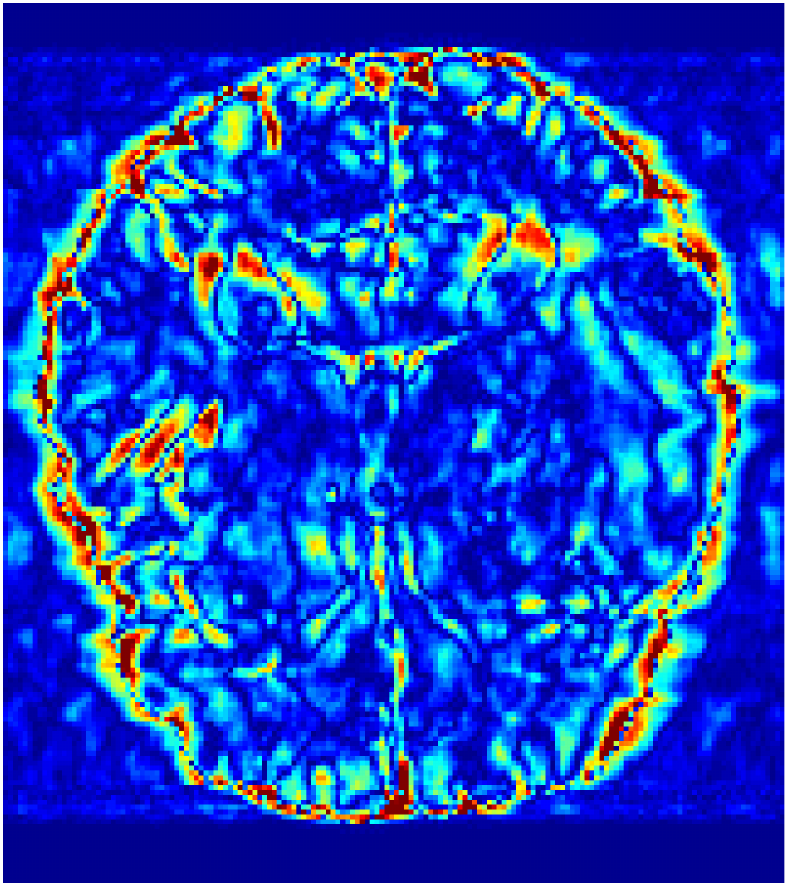}
\includegraphics[width=0.18\linewidth]{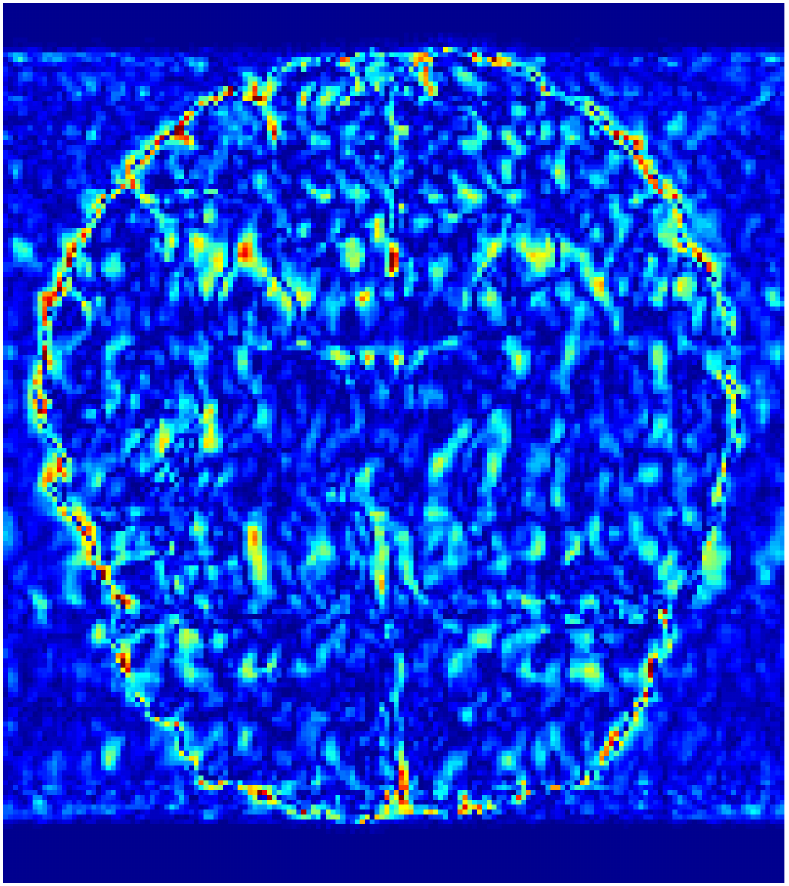}
\includegraphics[width=0.18\linewidth]{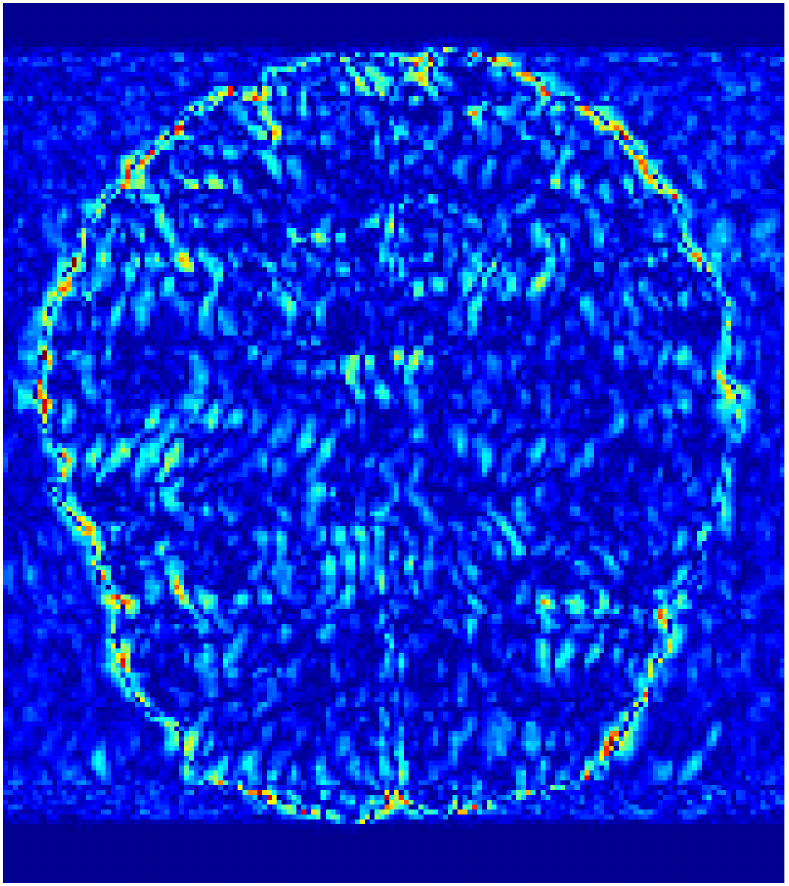}
\includegraphics[width=0.18\linewidth]{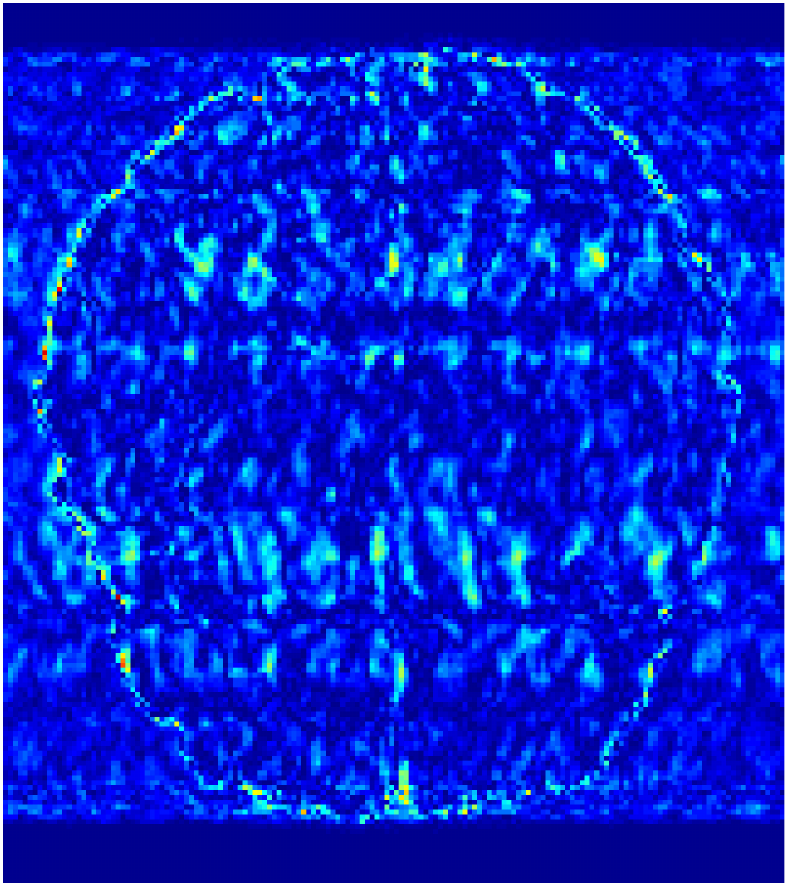}
\includegraphics[width=0.18\linewidth]{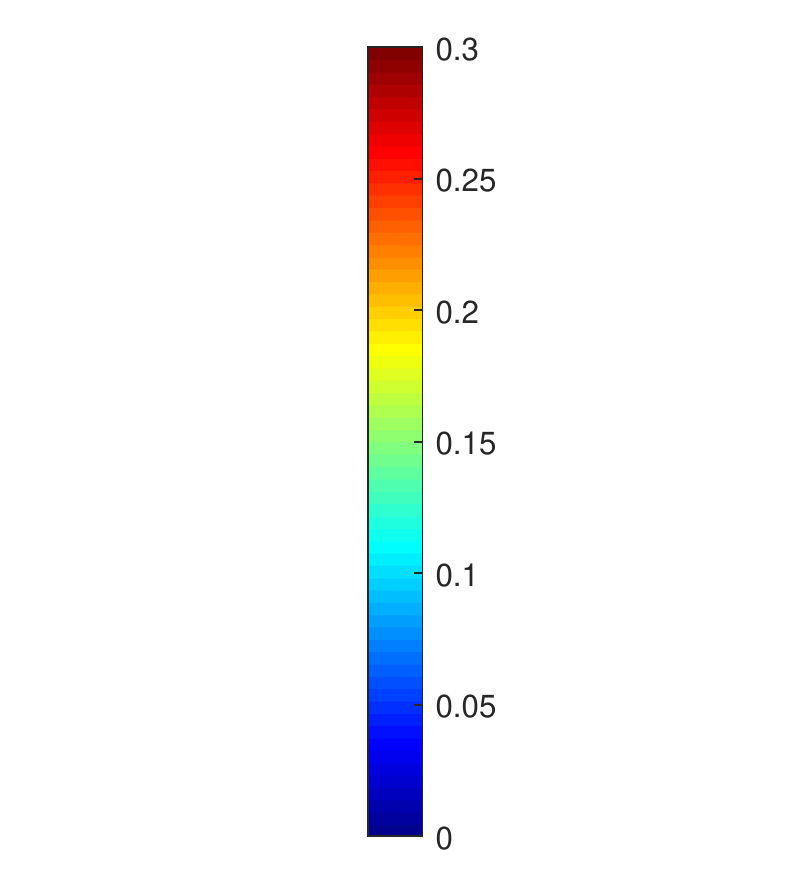}\\
\includegraphics[width=0.2\linewidth, angle=90]{fig/conventional_error.pdf}
\includegraphics[width=0.18\linewidth]{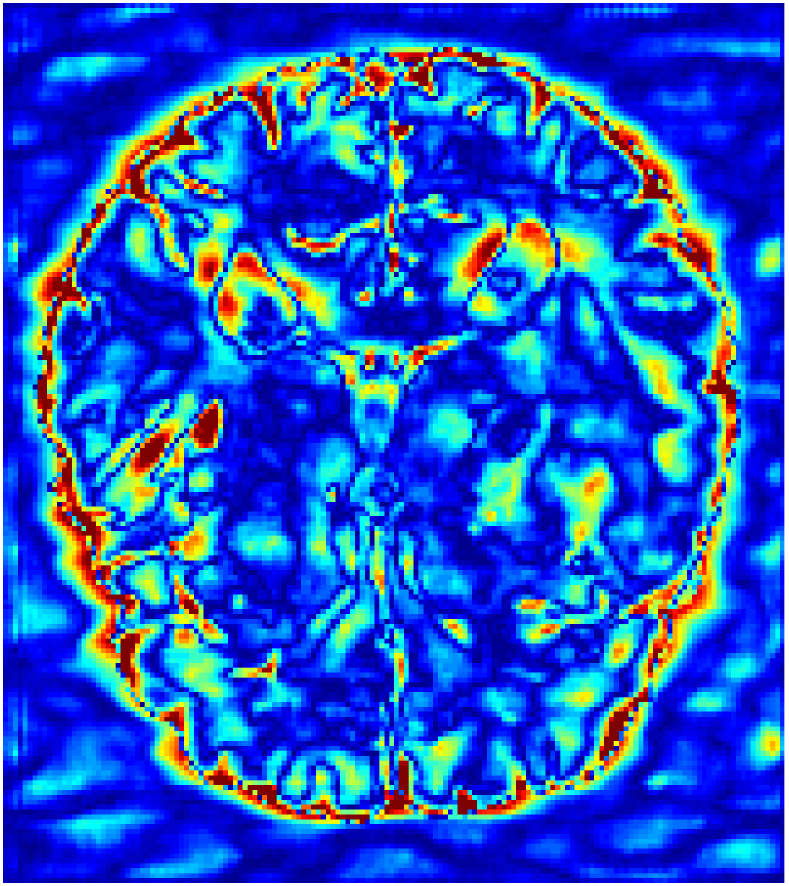}
\includegraphics[width=0.18\linewidth]{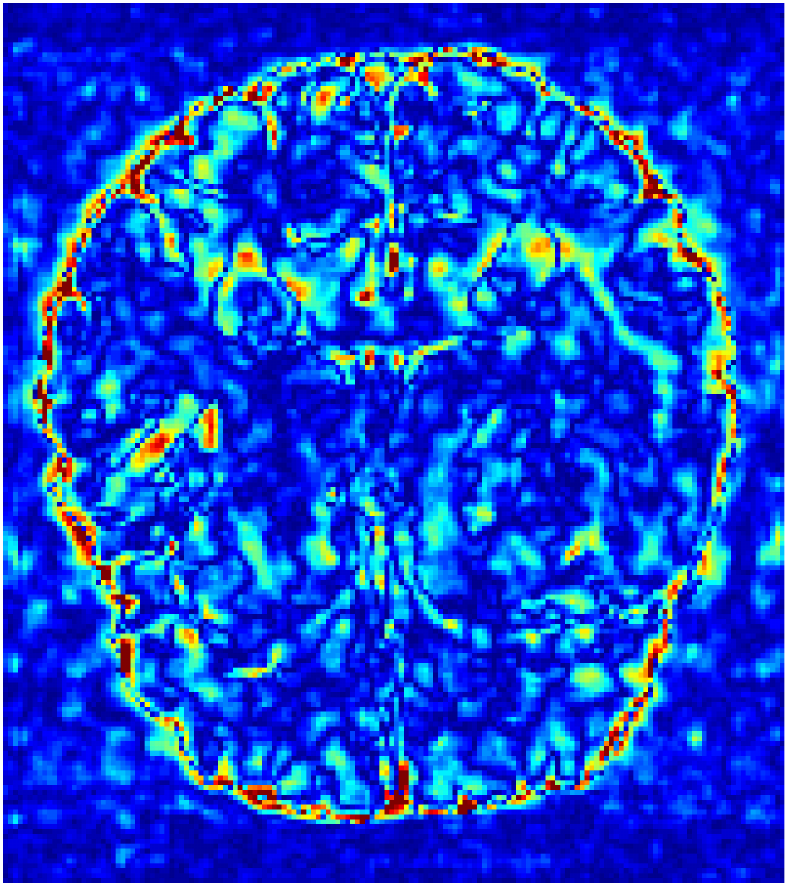}
\includegraphics[width=0.18\linewidth]{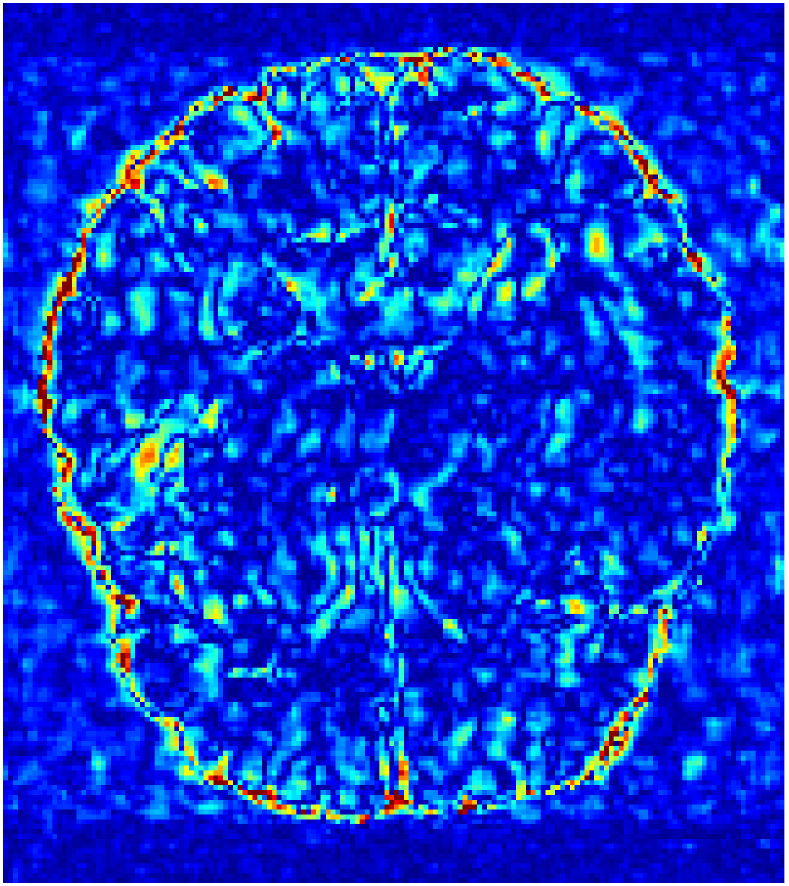}
\includegraphics[width=0.18\linewidth]{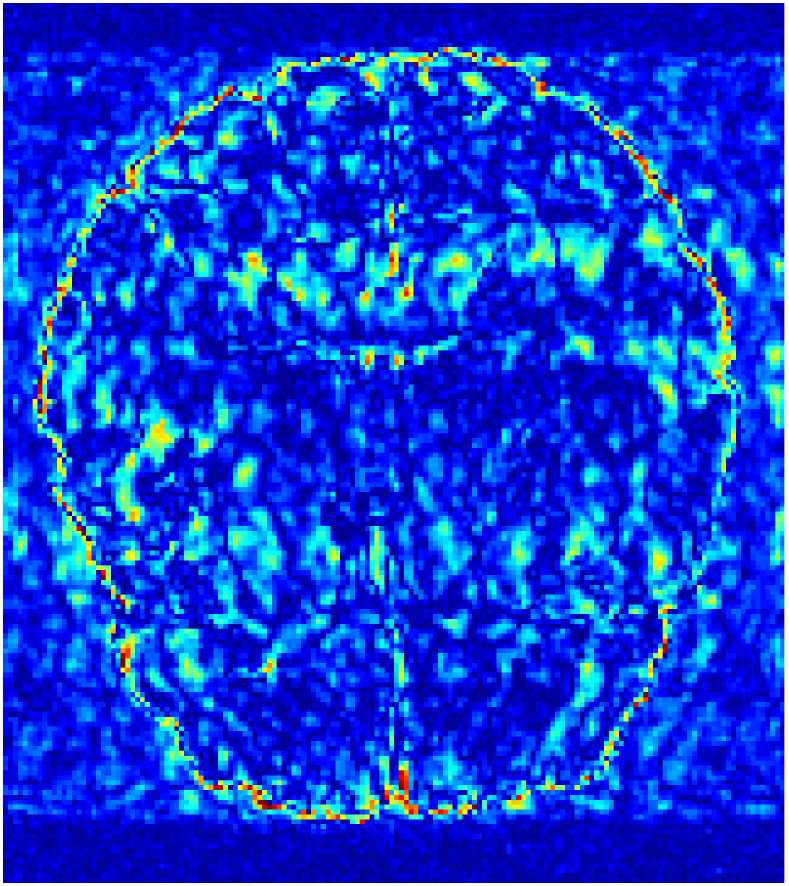}
\includegraphics[width=0.18\linewidth]{fig/white.pdf}\\
\includegraphics[width=0.2\linewidth, angle=90]{fig/masks.pdf}
\includegraphics[width=0.18\linewidth]{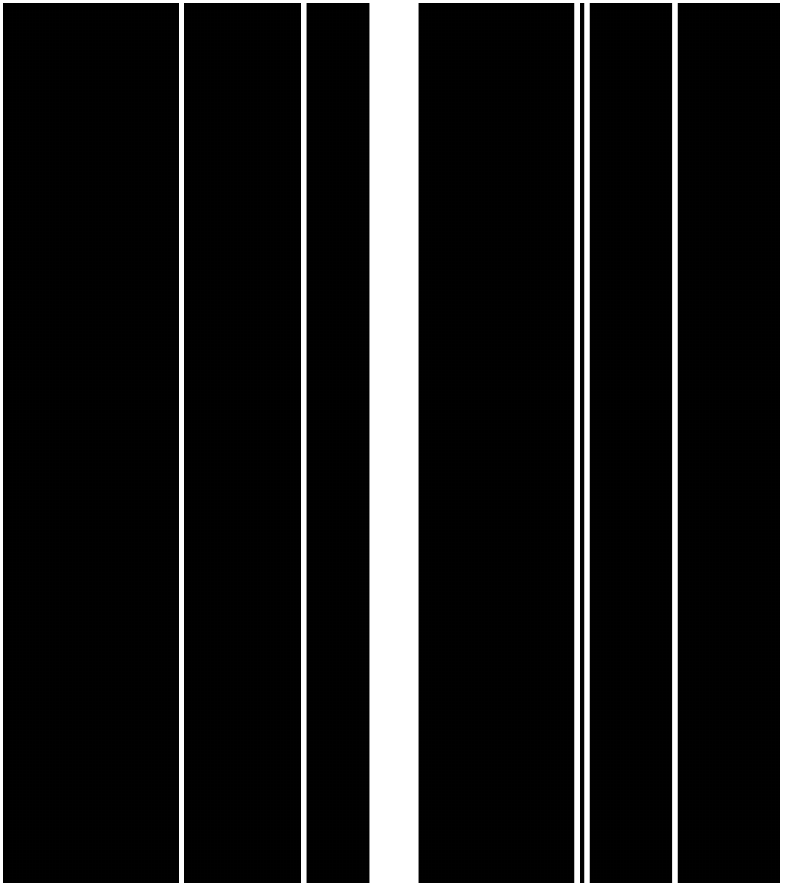}
\includegraphics[width=0.18\linewidth]{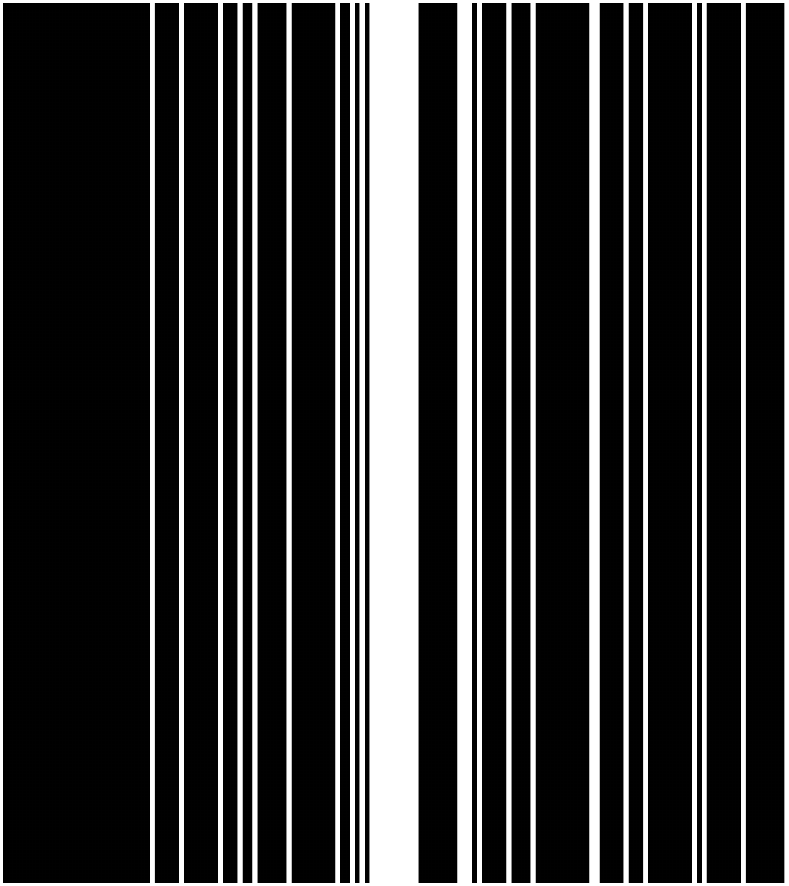}
\includegraphics[width=0.18\linewidth]{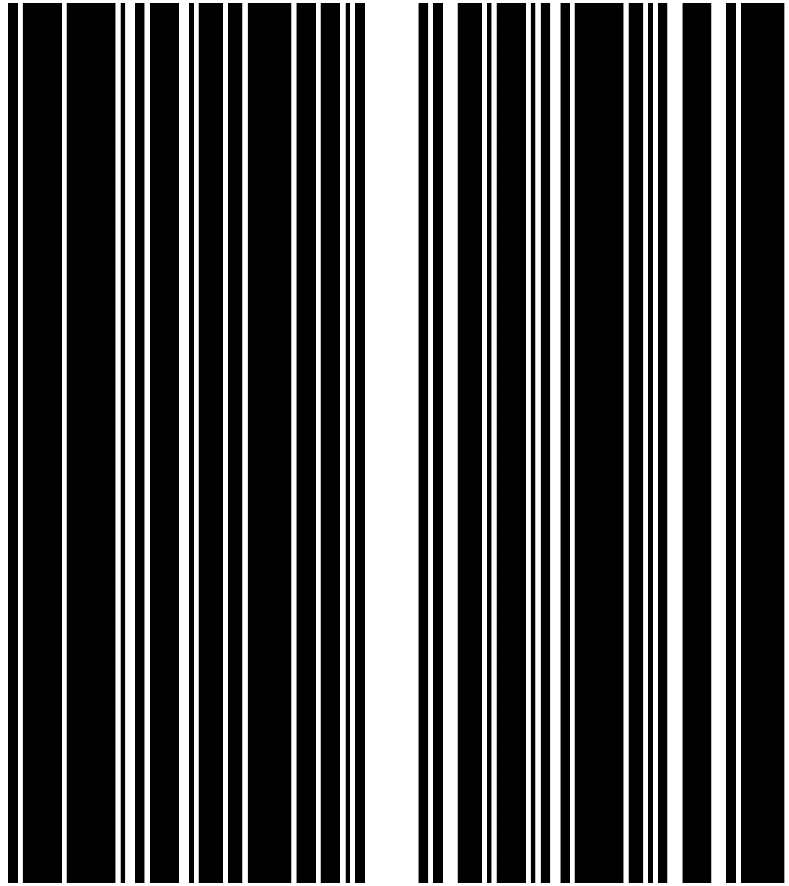}
\includegraphics[width=0.18\linewidth]{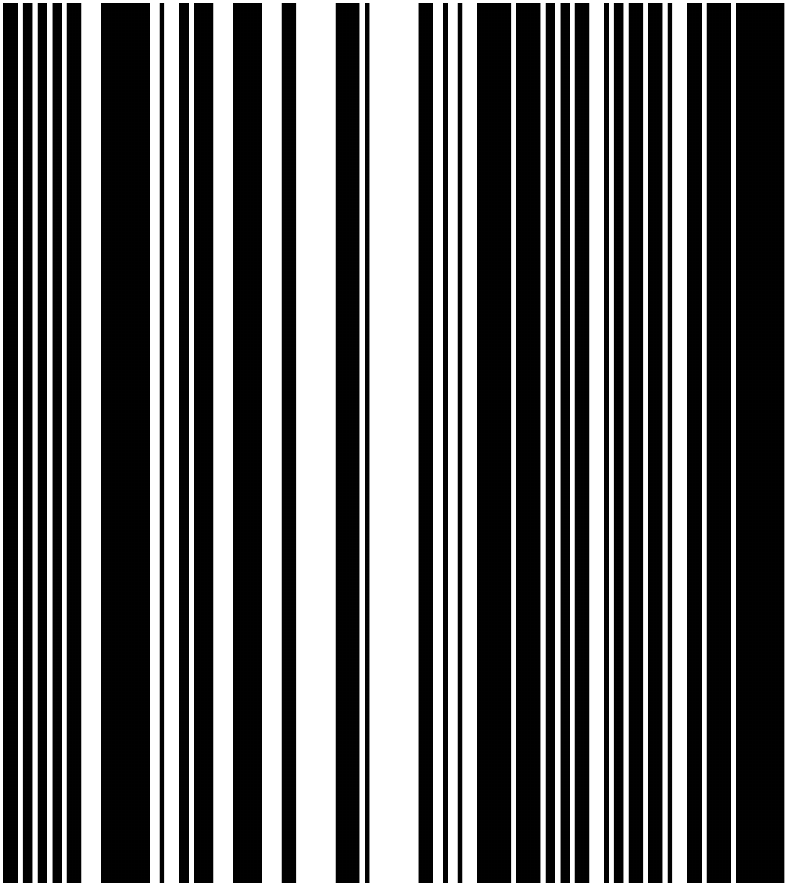}
\includegraphics[width=0.18\linewidth]{fig/white.pdf}
\caption{The pictures (from top to bottom) display the T2 Brain image reconstruction results, zoomed in details, pointwise errors with colorbar and associated \textbf{Cartesian} masks for both these  two compared methods with four different CS ratios 10\%, 20\%, 30\%, 40\%（from left to right). The most top right one is ground truth fully-sampled image. }
\label{figure_same_ratio_t2_cts}
\end{figure}

\iffalse
\begin{figure}[H]
\centering
\includegraphics[width=0.32\linewidth]{fig/meta_feature.pdf}
\includegraphics[width=0.32\linewidth]{fig/adam_feature.pdf}
\includegraphics[width=0.32\linewidth]{fig/true_feature.pdf}
\caption{Features (from left to right) of meta-learning and conventional learning, the last image is the original reference image.}
\label{feature}
\end{figure}
\fi

\subsection{Future work and open challenges}

Deep optimization-based meta-learning techniques have shown great generalizability but there are several open challenges that can be discussed and can potentially be addressed in future work. 
A major issue is the memorization problem since the base learner needs to be optimized for a large number of phases and the training algorithm contains multiple gradient steps, computation cost is very expensive in terms of time and memory costs. 
In addition to reconstruct MRI through different trajectories, another potential application for medical imaging could be multi-modality reconstruction and synthesis. Capturing images of anatomy with multi-modality acquisitions enhances the diagnostic information, and could be cast as a multi-task problem and benefit from meta-learning.

%%%%%%%%%%%%%%%%%%%%%%%%%%%%%%%%%%%%%%%%%%

\section{Conclusions}\label{conclusion}
In this paper, we put forward a novel deep model for MRI reconstructions via meta-learning. The proposed method has the ability to solve multi-tasks synergistically and the well-trained model could generalize well to new tasks. Our baseline network is constructed by unfolds an LOA, which inherits convergence property, improves the interpretability, and promotes parameter efficiency of the designed network structure. The designated adaptive regularizer consists task-invariant learner and a task-specific meta-knowledge. Network training follows a bilevel optimization algorithm that minimizes task-specific parameter $\omega$ in the upper level on validation data and minimizes task-invariant parameters $\theta$ on training data with fixed $\omega$. The proposed approach is the first model for solving the inverse problem by applying meta-training on the adaptive regularization in the variational model. We consider recovering undersampled raw data across different sampling trajectories with various sampling patterns as different tasks. Extensive numerical experiments on various MRI datasets demonstrate that the proposed method generalizes well at various
sampling trajectories and is capable of fast adaption to the unseen trajectories and sampling patterns. The reconstructed images achieve higher quality comparing to conventional supervised learning for both seen and unseen k-space trajectory cases.

%We proposed a meta-learning framework consisting of the base network architecture,  design of regularization, and bi-level optimization-based training. The network inherits the convergence property of the LOA and interpretation of the variational model. The generalization ability is improved by the designated regularization and bilevel optimization-based training algorithm.   

%%%%%%%%%%%%%%%%%%%%%%%%%%%%%%%%%%%%%%%%%%

%% Optional
\appendixtitles{yes} % Leave argument "no" if all appendix headings stay EMPTY (then no dot is printed after "Appendix A"). If the appendix sections contain a heading then change the argument to "yes".
\appendixstart
\appendix
\section{Convergence Analysis}
\label{convergence}
We make the following assumptions on $f$ and $\gbf$ throughout this work:

\begin{itemize}
\item (a1): $f$ is differentiable and (possibly) nonconvex, and $ \nabla f$ is $ L_f$-Lipschitz continuous.

\item (a2): Every component of $\gbf$ is differentiable and (possibly) nonconvex, $\nabla \gbf$ is $ L_g$-Lipschitz continuous. 

\item (a3): $\sup_{\xbf \in \Xcal}   \| \nabla \gbf(\xbf)\| \leq M$ for some constant $ M>0$.

\item (a4): $\phi$ is coercive, and $\phi^* = \min_{\xbf \in \Xcal} \phi(\xbf) > -\infty$.
\end{itemize}

First we state the Clark Subdifferential \cite{chen2020learnable} of $r(\xbf)$ in Lemma \ref{lem:r_subdiff} and we provide that the gradient of $\rbf_\varepsilon$ is Lipschitz continuous in Lemma \ref{r_lips}.
\begin{lemma}\label{lem:r_subdiff}
Let $r(\xbf)$ be defined in \eqref{eq:r}, then the Clarke subdifferential of $r$ at $\xbf$ is
\begin{equation}\label{eq:r_subdiff}
\partial \rbf(\xbf) = \{\sum_{j\in I_0}\nabla \gbf_i(\xbf)^{\top}  \wbf_j + \sum_{j \in I_1}\nabla \gj(\xbf)^{\top}\frac{\gj(\xbf)}{\|\gj(\xbf)\|} \ \bigg\vert \ \wbf_j \in \mathbb{R}^d, \ \|\Pi(\wbf_j; \Ccal(\nabla \gbf_i(\xbf)))\|\leq 1,\ \forall\, j \in I_0 \} ,  
\end{equation}
where $I_0=\{j \in [m] \ | \ \|\gj(\xbf) \|= 0 \}$, $I_1=[m] \setminus I_0$, and $\Pi(\wbf;\Ccal(\Abf))$ is the projection of $\wbf$ onto $\Ccal(\Abf)$ which stands for the column space of $\Abf$.
\end{lemma}
\begin{lemma}\label{r_lips}
The gradient of $\rbf_\varepsilon$ is Lipschitz continuous with constant $m ( L_g +\frac{2M^2}{\varepsilon})$.
\end{lemma}

\begin{proof}
From $\rbf_\varepsilon(\xbf) = \sum^m_{j=1} (\| \gbf_j(\xbf) \|^2 +\varepsilon^2)^{\frac{1}{2}} - \varepsilon $, it follows that 
\begin{equation}
    \nabla \rbf_\varepsilon (\xbf) = \sum^m_{j=1} \nabla \gbf_j(\xbf)^\top \gbf_j(\xbf) (\| \gbf_j(\xbf)\|^2 +\varepsilon^2)^{-\frac{1}{2}}.
\end{equation}
For any $\xbf_1, \xbf_2 \in \Xcal$, we first define $h(\xbf) = \gbf_j(\xbf)(\| \gbf_j(\xbf)\|^2 +\varepsilon^2)^{-\frac{1}{2}}$, so $ \norm{h(\xbf)} <1$.

\begin{subequations}\label{eq:h1-h2}
    \begin{align}
   & \left\| h(\xbf_1) - h(\xbf_2) \right\| \\
    & = \left \| \frac{\gj(\xbf_1)}{\sqrt{\| \gj(\xbf_1)\|^2 +\varepsilon^2}} - \frac{\gj(\xbf_2)}{\sqrt{\| \gj(\xbf_2)\|^2 +\varepsilon^2}} \right \|  \\  
    & = \left \| \frac{\gj(\xbf_1)}{\sqrt{\| \gj(\xbf_1)\|^2 +\varepsilon^2}} - \frac{\gj(\xbf_1)}{\sqrt{\| \gj(\xbf_2)\|^2 +\varepsilon^2}} +
    \frac{\gj(\xbf_1)}{\sqrt{\| \gj(\xbf_2)\|^2 +\varepsilon^2}}
    - \frac{\gj(\xbf_2)}{\sqrt{\| \gj(\xbf_2)\|^2 +\varepsilon^2}}  \right \| \\
    & \leq \left \| \gj(\xbf_1)  \left(\frac{\sqrt{\| \gj(\xbf_2)\|^2 +\varepsilon^2} - \sqrt{\| \gj(\xbf_1)\|^2 +\varepsilon^2} }{\sqrt{\| \gj(\xbf_1)\|^2 +\varepsilon^2} \sqrt{\| \gj(\xbf_2)\|^2 +\varepsilon^2} } \right) \right \| + \left\| \frac{ \gj(\xbf_1) - \gj(\xbf_2)  }{\sqrt{\| \gj(\xbf_2)\|^2 +\varepsilon^2}} \right\| \\
    & \leq   \left\| \frac{ \gj(\xbf_1) }{\sqrt{\| \gj(\xbf_1)\|^2 +\varepsilon^2}} \right\|  \left \| \frac{\sqrt{\| \gj(\xbf_2)\|^2 +\varepsilon^2} - \sqrt{\| \gj(\xbf_1)\|^2 +\varepsilon^2} }{ \sqrt{\| \gj(\xbf_2)\|^2 +\varepsilon^2} }  \right\| + \frac{1}{\varepsilon} \left\| \gj(\xbf_1) - \gj(\xbf_2)  \right\| \\
    & \leq  \frac{1}{\varepsilon} \left\| \sqrt{\| \gj(\xbf_2)\|^2 +\varepsilon^2} - \sqrt{\| \gj(\xbf_1)\|^2 +\varepsilon^2} \right\| + \frac{1}{\varepsilon} \left\| \gj(\xbf_1) - \gj(\xbf_2)  \right\| \label{A3e}\\
    & \leq \frac{1}{\varepsilon} \frac{\| \gj(\xbf_2)\|^2 - \| \gj(\xbf_1)\|^2}{\sqrt{\| \gj(\xbf_2)\|^2 +\varepsilon^2} + \sqrt{\| \gj(\xbf_1)\|^2 +\varepsilon^2} } +  \frac{1}{\varepsilon} \left\| \gj(\xbf_1) - \gj(\xbf_2)  \right\| \\
    & \leq \frac{1}{\varepsilon} \underbrace{\frac{\| \gj(\xbf_2)\| + \| \gj(\xbf_1)\| }{\sqrt{\| \gj(\xbf_2)\|^2 +\varepsilon^2} + \sqrt{\| \gj(\xbf_1)\|^2 +\varepsilon^2}} }_{<1} \left( \| \gj(\xbf_2)\| - \| \gj(\xbf_1)\|\right)  +  \frac{1}{\varepsilon} \left\| \gj(\xbf_1) - \gj(\xbf_2)  \right\| \\
    & \leq \frac{1}{\varepsilon}  \left\| \gj(\xbf_2) - \gj(\xbf_1)  \right\|+ \frac{1}{\varepsilon}  \left\| \gj(\xbf_1) - \gj(\xbf_2)  \right\|\\
    & = \frac{2}{\varepsilon}  \left\| \gj(\xbf_1) - \gj(\xbf_2)  \right\|.
    \end{align}
\end{subequations}
where to get \eqref{A3e} we used $\left\| \frac{ \gj(\xbf_1) }{\sqrt{\| \gj(\xbf_1)\|^2 +\varepsilon^2}} \right\| < 1 \text{ and }  \frac{ 1 }{\sqrt{\| \gj(\xbf_1)\|^2 +\varepsilon^2}} < \frac{1}{\varepsilon}$.

Therefore we have:
\begin{subequations}
    \begin{align}
   & \left\|\nabla \rbf_\varepsilon (\xbf_1) - \nabla  \rbf_\varepsilon (\xbf_2) \right\| \\
   &  = \sum^m_{j=1}  \left\|\nabla \gj(\xbf_1)^\top h(\xbf_1) - \nabla \gj(\xbf_2)^\top h(\xbf_2) \right\| \\
    & = \sum^m_{j=1}  \left\|\nabla \gbf_j(\xbf_1)^\top h(\xbf_1) - \nabla \gj(\xbf_2)^\top h(\xbf_1) +  \nabla \gj(\xbf_2)^\top  h(\xbf_1) - \nabla \gj(\xbf_2)^\top h(\xbf_2) \right\|  \\
    & \leq \sum^m_{j=1}  \left\| \left( \nabla \gbf_j(\xbf_1) - \nabla \gbf_j(\xbf_2) \right)^\top h(\xbf_1)  \right\| + \left\| \nabla \gbf_j(\xbf_2) \left( h(\xbf_1) - h(\xbf_2) \right) \right\| \\
& \leq \sum^m_{j=1}  \left\| \nabla \gbf_j(\xbf_1) - \nabla \gbf_j(\xbf_2) \right\|  \norm{h(\xbf_1)}  +\norm{\nabla \gbf_j(\xbf_2)} \left\| h(\xbf_1) - h(\xbf_2) \right\| \notag \\
& \leq \sum^m_{j=1} \left\| \nabla \gbf_j(\xbf_1) - \nabla \gbf_j(\xbf_2) \right\| + \norm{\nabla \gbf_j(\xbf_2)} \frac{2}{\varepsilon} \norm{\gbf_j(\xbf_1)- \gbf_j(\xbf_2)} \text{ by } \eqref{eq:h1-h2} \text{ and }\\
& \leq m (L_g \norm{\xbf_1 - \xbf_2}+ M \frac{2}{\varepsilon} \cdot M \norm{\xbf_1 - \xbf_2}),
    \end{align}
\end{subequations}
where the first term  of the last inequality is due to the $L_g$-Lipschitz continuity of $ \nabla \gj$. The second term is because of $\norm{\gbf_j(\xbf_1)- \gbf_j(\xbf_2)}  = \norm{\nabla \gbf_j(\tilde{\xbf})(\xbf_1-\xbf_2)}  $ for some $\tilde{\xbf} \in \Xcal$ due to the mean value theorem and  $ \| \nabla \gj(\tilde{\xbf}) \| \le \sup_{\xbf \in \Xcal} \| \nabla \gbf_j(\xbf) \| \leq M$. Therefore we get
\begin{equation}
    \left\|\nabla \rbf_\varepsilon (\xbf_1) - \nabla  \rbf_\varepsilon (\xbf_2) \right\| \leq m(L_g + \frac{2M^2}{\varepsilon})  \norm{\xbf_1 - \xbf_2}.
\end{equation}
\end{proof}

\begin{lemma}\label{lem:inner}
Let $\varepsilon, \eta, \taut, a>0$,  $ 0<\rho<1$,  choose initial $ \xbf_0 \in \Xcal$. Suppose the sequence $ \{ \xt \}$ is generated by  executing Lines 3-14 of Algorithm  \ref{alg:lda} with fixed $ \epst = \varepsilon$ and exists $  0<\delta<\frac{L_\varepsilon}{2/a+L_\varepsilon} <1 $ such that $  \alpha_t \geq \frac{\delta}{L_\varepsilon} >0$, where $ L_{\varepsilon} = L_f + L_h + m L_g + \frac{2M^2}{\varepsilon}$ and $ \phi^* := \min_{\xbf \in  \Xcal} \phi(\xbf)$. Then the following statements hold:
\begin{enumerate}
\item $ \| \nabla \phi_\varepsilon (\xt) \| \to 0$ as $t \to \infty$.

\item $ \max \{ t\in \NN \ |  \  \norm{\nabla \phi_{\varepsilon} (\xtp) } \geq \eta \} \leq \max \{ \frac{{a L^2_\varepsilon} }{ \delta^2}, a^3\}  \left(\phi_\varepsilon(\xbf_0) - \phi^*+\varepsilon \right) $.
\end{enumerate}
\end{lemma}
\begin{proof}
\begin{enumerate}
\item
In each iteration, we compute $\utp = \ztp - \taut\sigma(\omega_i) \nabla \rbf_{\epst} (\ztp)$. 
\begin{enumerate}[label*=\arabic*.]
\item \label{case1} In the case the condition  
\begin{equation}\label{eq:con1}
   \| \nabla \phi_{\varepsilon} (\xt) \| \leq a \| \utp - \xt \| \ \ \ \mbox{and}  \ \ \  \phi_{\varepsilon}(\utp) - \phi_{\varepsilon}(\xt) \leq - \frac{1}{a}\| \utp - \xt \|^2 
\end{equation} holds with $a>0$,
we put $ \xtp = \utp$, and we have $\phi_{\varepsilon}(\utp) \leq \phi_{\varepsilon}(\xt)$. 
\item \label{case2} Otherwise, we compute $\vtp = \xt - \alpha_{t} \nabla \phi_{\varepsilon}(\xt)$, where $ \alpha_{t}$ is found through the line search until the criteria 
\begin{equation}\label{eq:con2}
  \phi_{\varepsilon}(\vtp) - \phi_{\varepsilon}(\xt) \le - \frac{1}{a} \| \vtp - \xt\|^2  
\end{equation} holds and then put $ \xtp = \vtp $. From Lemma \ref{r_lips} we know that the gradient $ \nabla \rbf_{\varepsilon} (\xbf)$ is Lipschitz continuous with constant $m (L_g +\frac{2M^2}{\varepsilon})$. Also we assumed in (a1) $ \nabla f$ is $L_f$-Lipschitz continuous. Hence, putting $ L_\varepsilon = L_f + m( L_g +\frac{2M^2}{\varepsilon})$, we get that $\nabla \phi_\varepsilon$ is $L_\varepsilon$-Lipschitz continuous, which implies
\begin{equation}\label{eq:phi}
    \phi_\varepsilon(\vtp) \leq \phi_\varepsilon(\xt) + \langle \nabla \phi_\varepsilon(\xt) , \vtp-\xt \rangle + \frac{L_\varepsilon}{2} \| \vtp-\xt \|^2.
\end{equation}
Also, by the optimality condition of $$ \vtp = \argmin_{\xbf} \langle \nabla f(\xt), \xbf - \xt \rangle +  \sigma(\omega_i) \langle \nabla \rbf_{\varepsilon}(\xt) , \xbf - \xt \rangle + \frac{1}{2 \alpha_t} \| \xbf - \xt \|^2, $$ we have 
\begin{equation}\label{eq:opt}
\langle \nabla \phi_\varepsilon (\xt), \vtp - \xt \rangle + \frac{1}{2 \alpha_t} \norm{\vtp - \xt}^2 \leq 0.
\end{equation}
Combine \eqref{eq:phi} and \eqref{eq:opt} and $ \vtp = \xt -\alpha_t \nabla \phi_\varepsilon (\xt)$ in line 8 of Algorithm \ref{alg:lda} yields 
\begin{equation}
    \phi_\varepsilon(\vtp) - \phi_\varepsilon(\xt) \leq -\left(\frac{1}{2\alpha_t} - \frac{L_\varepsilon}{2} \right) \norm{\vtp - \xt}^2.  %= -\frac{\alpha_t(1 - \alpha_t L_\varepsilon)}{2} \norm{\nabla \phi_\varepsilon (\xt)}^2 \leq 0,
\end{equation}
Therefore, it is enough for $ \alpha_t \leq \frac{1}{ 2/a + L_\varepsilon}$ so that the criteria \eqref{eq:con2}
 is satisfied. This process only take finitely many iterations since we can find a finite $t$ such that $ \rho^t \alpha_t \leq \frac{1}{ 2/a + L_\varepsilon}$ and through the line search we can get $ \phi_\varepsilon(\vtp) \leq \phi_\varepsilon (\xt)$.
\end{enumerate}
Therefore in either case of \ref{case1} or \ref{case2} where we take $\xtp = \utp$ or $\vtp$,  we can get
\begin{equation}\label{eq:dec}
    \phi_\varepsilon (\xtp) \leq \phi_\varepsilon (\xt) , \text{ for all } t\geq 0.
\end{equation}
Now from case \ref{case1}, \eqref{eq:con1} gives 
\begin{subequations}
\begin{align}
\norm{\nabla \phi_\varepsilon(\xt)}^2 \leq a^2\left\| \utp - \xt \right\|^2 \leq &  a^3 \left(\phi_\varepsilon(\xt) - \phi_\varepsilon(\utp) \right), \\
\text{therefore if $\xtp = \utp $ we get } & \norm{\nabla \phi_\varepsilon(\xt)}^2 \leq a^3 \left(\phi_\varepsilon(\xt) - \phi_\varepsilon(\xtp) \right).   \label{eq:c1}    
\end{align}
\end{subequations}
From  case \ref{case2} and $\vtp = \xt - \alpha_{t} \nabla \phi_{\varepsilon}(\xt)$ we have  
\begin{subequations}
    \begin{align}
     \phi_\varepsilon(\vtp) - \phi_\varepsilon(\xt) \leq  & - \frac{1}{a} \norm{\vtp - \xt}^2 = -\frac{1}{a} \alpha_t^2 \norm{\nabla \phi_\varepsilon(\xt)}^2\\
    \Longrightarrow 
    &  \norm{\nabla \phi_\varepsilon(\xt)}^2 \leq \frac{a}{\alpha_t^2} \big(\phi_\varepsilon(\xt) - \phi_\varepsilon(\vtp) \big),\\
     \text{ then if $\xtp =\vtp$, we have }
    &  \norm{\nabla \phi_\varepsilon(\xt)}^2 \leq \frac{a}{ \alpha_t^2} \big(\phi_\varepsilon(\xt) - \phi_\varepsilon(\xtp) \big). \label{eq:c2}
    \end{align}
\end{subequations}
Since $ \frac{\delta}{L_\varepsilon}\leq \alpha_t \leq \frac{1}{ 2/a + L_\varepsilon}$, we have
\begin{equation}\label{eq:grad_phi_bound}
     \norm{\nabla \phi_\varepsilon(\vtp)}^2  \leq \frac{{a L^2_\varepsilon} }{ \delta^2}\big(\phi_\varepsilon(\vtp) - \phi_\varepsilon(\xtp) \big).
\end{equation}

Combining \eqref{eq:c1} and \eqref{eq:grad_phi_bound} and select $C = \max \{ \frac{{a L^2_\varepsilon} }{ \delta^2}, a^3\} $, we get
\begin{equation}\label{eq:grad_phi_bound_2}
     \norm{\nabla \phi_\varepsilon(\xt)}^2  \le  C \big(\phi_\varepsilon(\xt) - \phi_\varepsilon(\xtp) \big).
\end{equation}

Summing up \eqref{eq:grad_phi_bound_2} for $t=0,\cdots, T$, we have
\begin{equation}\label{eq:sum}
    \sum^{T}_{t=0} \norm{\nabla \phi_\varepsilon(\xt)}^2  \leq C \big(\phi_\varepsilon(\xbf_0) - \phi_\varepsilon(\xbf_{T+1}) \big), 
\end{equation}
combining with the fact $ \phi_\varepsilon(\xbf) \geq \phi(\xbf) - \varepsilon \geq \phi^* -\varepsilon$ for every $ \xbf\in \Xcal$, we have 
\begin{equation}\label{eq:sum_}
    \sum^T_{t=0} \norm{\nabla \phi_\varepsilon(\xt)}^2 \leq C  \left(\phi_\varepsilon(\xbf_0) - \phi^*+\varepsilon \right). 
\end{equation}
The right hand side is a finite constant and hence by letting $t\to \infty$ we know that $\norm{\nabla \phi_\varepsilon(\xt)} \to 0$, which proves the first statement.

\item Denote $ \kappa :=  \max \{ t\in \NN \ |  \  \norm{\nabla \phi_{\varepsilon} (\xtp) } \geq \eta \}$, then we know that $ \norm{\nabla \phi_{\varepsilon} (\xtp) } \geq \eta $ for all $ t\leq \kappa-1$. Hence we have 
\begin{equation}
    \kappa \eta^2 \leq \sum^{\kappa-1}_{t=0} \norm{ \nabla \phi_{\varepsilon} (\xtp) }^2 = \sum^{\kappa}_{t=1}  \norm{ \nabla \phi_{\varepsilon} (\xt) }^2\leq C \left(\phi_\varepsilon(\xbf_0) - \phi^*+\varepsilon \right). 
\end{equation}
which implies the second statement.
\end{enumerate}
\end{proof}

\begin{lemma}\label{lem:phi_decay}
Suppose that the sequence $\{ \xt\}$ is generated by Algorithm 1 with an initial guess $\xbf_0$. Then for any $t\geq 0$, we have 
$ \phi_{\epstp}(\xtp) + \epstp \leq \phi_{\epst}(\xt) + \epst$.
\end{lemma}
\begin{proof}
To prove this statement, we can prove
\begin{equation}\label{eq:ineq}
\phi_{\epstp}(\xtp) + \epstp \leq  \phi_{\epst}(\xtp) + \epst \leq \phi_{\epst}(\xt) + \epst.
\end{equation}
The second inequality is immediately obtained from \eqref{eq:dec}. Now we prove the first inequality. 

For any $\varepsilon>0$, denote
\begin{equation}\label{eq:r_}
\rbf_{\varepsilon, j} (\xbf) = \sqrt{\| \gbf_j (\xbf) \|^2_{2} + \varepsilon^2} -\varepsilon.
\end{equation}

Since $ \phi_\varepsilon(\xbf)=f(\xbf)+\sigma(\omega_i)\sum^m_{j=1} \rbf_{\varepsilon, j}(\xbf)$, it suffices to show that 
\begin{equation}
    \rbf_{\epstp, j}(\xtp) + \epstp \leq  \rbf_{\epst, j}(\xtp) + \epst 
\end{equation}
If $ \epstp =\epst$ then the two quantities above are identical and the first inequality holds. Now suppose $ \epstp = \gamma\epst \le \epst$ then 
\begin{equation}
    \rbf_{\epstp, j}(\xtp) + \epstp =  \sqrt{\| \gbf_j (\xbf) \|^2_{2} + \epstp^2} \leq   \sqrt{\| \gbf_j (\xbf) \|^2_{2} + \epst^2} = \rbf_{\epst, j}(\xtp) + \epst,
\end{equation}
which implies the first inequality of \eqref{eq:ineq}.
\end{proof}
\begin{Theorem}
Suppose that $\{\xt \}$ is the sequence generated by Algorithm \ref{alg:lda} with any initial $\xbf_0$, $\etol=0$ and $T=\infty$. Let $ \{ \xbf_{t_l+1}\}$ be the subsequence that satisfies the reduction criterion  in step 15 of Algorithm \ref{alg:lda}, i.e. $  \norm{\nabla \phi_{\varepsilon_{t_l}} (\xbf_{t_l+1})} \leq \sigma  \varepsilon_{t_l} \gamma $ for $t=t_l$ and $ l=1,2,\cdots$. Then $ \{ \xbf_{t_l+1}\}$ has at least one accumulation point, and every accumulation point of $\{ \xbf_{t_l+1}\}$ is a clarke stationary point of $ \min_{\xbf} \phi(\xbf) := f(\xbf) +\sigma(\omega_i) r(\xbf)$.
    \label{theorem a6}
\end{Theorem}
\begin{proof}
By the Lemma \ref{lem:phi_decay} and $ \phi(\xbf) \leq \phi_\varepsilon(\xbf) +\varepsilon$ for all $\varepsilon>0$ and $ \xbf\in \Xcal$, we know that:
\begin{equation}
    \phi(\xt) \leq \phi_{\epst}(\xt) +\epst\leq \cdots \leq \phi_{\varepsilon_0}(\xbf_0) +\varepsilon_0 <\infty.
\end{equation}
Since $\phi$ is coercive, we know that $\{ \xt\}$ is bounded,  the selected subsequence $ \{ \xbf_{t_l+1} \}$ is also bounded and has at least one accumulation point.

Note that $ \norm{\nabla \phi_{\varepsilon_{t_l}} (\xbf_{t_l+1})} \leq \sigma  \varepsilon_{t_l} \gamma = \sigma \varepsilon_{0} \gamma^{l+1} \to 0$ as $l\to \infty$.  Let $ \{ \xpp\} $ be any convergent subsequence of $\{ \xbf_{t_l+1} \}$ and denote $\epsp$ the corresponding $\epst$ used in the Algorithm \ref{alg:lda} that generate $\xpp$. Then there exists $ \hat{\xbf} \in \Xcal$ such that $ \xpp \to  \hat{\xbf}$ as $ \varepsilon_p \to 0,$ and $ \nabla \phi_{\epsp}(\xpp) \to 0$ as $ p\to \infty$.

Note that the Clarke subdifferential of $\phi$ at $\hat{\xbf}$ is given by $\partial \phi(\hat{\xbf}) =  \partial f(\hat{\xbf}) + \sigma(\omega_i) \partial \rbf(\hat{\xbf}) $:

\begin{multline} \label{eq:d_phi_xhat}
\partial \phi(\xhat) = \{\nabla f(\xhat) + \sigma(\omega_i) \sum_{j \in I_0} \nabla \gj(\xhat)^{\top} \wbf_j + \sigma(\omega_i) \sum_{ j \in I_1} \nabla \gj(\xhat)^{\top} \frac{\gj(\xhat)}{\| \gj(\xhat) \|} \ \bigg\vert \\
 \ \norm{\Pi(\wbf_j; \Ccal(\nabla \gj(\xhat)))}  \le 1,\ \forall\, j\in I_0\},
\end{multline}
where $I_0 = \{j\in[m]\ \vert \ \|\gi(\xhat)\| = 0 \}$ and $I_1 = [m] \setminus I_0$.

If $j\in I_0$, we have $ \norm{\gj(\xbf)} =0 \iff \gj(\xbf) =0 $, then  
\begin{subequations}
    \begin{align}
    \partial \rbf_{\varepsilon}(\xbf) & = \sum_{j\in I_0} \nabla \gbf_j(\xbf)^{\top} \frac{\gbf_j(\xbf) }{\Big(\norm{\gbf_j(\xbf)}^2+\varepsilon^2 \Big)^{\frac{1}{2}}}  +  \sum_{j\in I_1} \nabla \gbf_j(\xbf)^{\top} \frac{\gbf_j(\xbf) }{\Big(\norm{\gbf_j(\xbf)}^2+\varepsilon^2 \Big)^{\frac{1}{2}}}  \\
    & = \mathbf{0} + \sum_{j\in I_1} \nabla \gbf_j(\xbf)^{\top} \frac{\gbf_j(\xbf) }{\Big(\norm{\gbf_j(\xbf)}^2+\varepsilon^2 \Big)^{\frac{1}{2}}} 
    \end{align}
\end{subequations}

Therefore we get
\begin{equation}\label{eq:d_phi_epsj}
\nabla \phi_{\epsp}(\xpp)  =   \nabla f(\xpp) +  \sigma(\omega_i) \sum_{j\in I_1} \nabla \gbf_j(\xbf)^{\top} \frac{\gbf_j(\xbf) }{\Big(\norm{\gbf_j(\xbf)}^2+\varepsilon_p^2 \Big)^{\frac{1}{2}}} . 
\end{equation}
Comparing \eqref{eq:d_phi_xhat} and \eqref{eq:d_phi_epsj}  we can see that the first term on the right hand side of \eqref{eq:d_phi_epsj} converge to that of \eqref{eq:d_phi_xhat}, due to the facts that $\xpp \to \xhat$ and the the continuity of $\nabla f$. Together with the continuity of $\gi$  and
$\nabla \gi$, the last term of \eqref{eq:d_phi_xhat} converges to the last term of \eqref{eq:d_phi_epsj} as $\varepsilon_p \rightarrow 0$ and $\norm{\gbf_j(\xbf)} > 0$. And apparently $\mathbf{0}$ is a special case of the second term in \eqref{eq:d_phi_epsj}.
Hence we know that
\[
\mathrm{dist}( \nabla \phi_{\epsp}(\xpp), \partial \phi(\xhat)) \to 0,
\]
as $p \to \infty$. Since $\nabla \phi_{\epsp}(\xpp) \to 0$ and $\partial \phi(\xhat)$ is closed, we conclude that $0 \in \partial \phi(\xhat)$.
\end{proof}

\end{paracol}
\reftitle{References}


\externalbibliography{yes}
\bibliography{ref}

\begin{thebibliography}{999}

\bibitem[Munkhdalai and Yu(2017)]{munkhdalai2017meta}
Munkhdalai, T.; Yu, H.
\newblock Meta networks.
\newblock  International Conference on Machine Learning. PMLR,  2017, pp.
  2554--2563.

\bibitem[Finn \em{et~al.}(2017)Finn, Abbeel, and Levine]{finn2017model}
Finn, C.; Abbeel, P.; Levine, S.
\newblock Model-agnostic meta-learning for fast adaptation of deep networks.
\newblock  International Conference on Machine Learning. PMLR,  2017, pp.
  1126--1135.

\bibitem[Li \em{et~al.}(2018)Li, Yang, Song, and Hospedales]{li2018learning}
Li, D.; Yang, Y.; Song, Y.Z.; Hospedales, T.M.
\newblock Learning to generalize: Meta-learning for domain generalization.
\newblock  Thirty-Second AAAI Conference on Artificial Intelligence,  2018.

\bibitem[Rusu \em{et~al.}(2018)Rusu, Rao, Sygnowski, Vinyals, Pascanu,
  Osindero, and Hadsell]{rusu2018meta}
Rusu, A.A.; Rao, D.; Sygnowski, J.; Vinyals, O.; Pascanu, R.; Osindero, S.;
  Hadsell, R.
\newblock Meta-Learning with Latent Embedding Optimization.
\newblock  International Conference on Learning Representations,  2018.

\bibitem[Yao \em{et~al.}(2021)Yao, Huang, Zhang, Wei, Tian, Zou, Huang,
  et~al.]{yao2021improving}
Yao, H.; Huang, L.K.; Zhang, L.; Wei, Y.; Tian, L.; Zou, J.; Huang, J.; others.
\newblock Improving generalization in meta-learning via task augmentation.
\newblock  International Conference on Machine Learning. PMLR,  2021, pp.
  11887--11897.

\bibitem[Balaji \em{et~al.}(2018)Balaji, Sankaranarayanan, and
  Chellappa]{balaji2018metareg}
Balaji, Y.; Sankaranarayanan, S.; Chellappa, R.
\newblock Metareg: Towards domain generalization using meta-regularization.
\newblock {\em Advances in Neural Information Processing Systems} {\bf 2018},
  {\em 31},~998--1008.

\bibitem[Thrun and Pratt(1998)]{thrun1998learning}
Thrun, S.; Pratt, L.
\newblock Learning to learn: Introduction and overview. In {\em Learning to
  learn}; Springer,  1998; pp. 3--17.

\bibitem[Hospedales \em{et~al.}(2021)Hospedales, Antoniou, Micaelli, and
  Storkey]{hospedales2021meta}
Hospedales, T.M.; Antoniou, A.; Micaelli, P.; Storkey, A.J.
\newblock Meta-Learning in Neural Networks: A Survey.
\newblock {\em IEEE Transactions on Pattern Analysis and Machine Intelligence}
  {\bf 2021}.

\bibitem[Chen \em{et~al.}(2020)Chen, Liu, Ye, and Zhang]{chen2020learnable}
Chen, Y.; Liu, H.; Ye, X.; Zhang, Q.
\newblock Learnable Descent Algorithm for Nonsmooth Nonconvex Image
  Reconstruction.
\newblock {\em arXiv preprint arXiv:2007.11245} {\bf 2020},
  \href{http://xxx.lanl.gov/abs/2007.11245}{{\normalfont
  [arXiv:cs.CV/2007.11245]}}.

\bibitem[Huisman \em{et~al.}(2021)Huisman, van Rijn, and
  Plaat]{huisman2021survey}
Huisman, M.; van Rijn, J.N.; Plaat, A.
\newblock A survey of deep meta-learning.
\newblock {\em Artificial Intelligence Review} {\bf 2021}, pp. 1--59.

\bibitem[Yao \em{et~al.}(2020)Yao, Wu, Tao, Li, Ding, Li, and
  Li]{yao2020automated}
Yao, H.; Wu, X.; Tao, Z.; Li, Y.; Ding, B.; Li, R.; Li, Z.
\newblock Automated relational meta-learning.
\newblock {\em arXiv preprint arXiv:2001.00745} {\bf 2020}.

\bibitem[Lee and Choi(2018)]{lee2018gradient}
Lee, Y.; Choi, S.
\newblock Gradient-based meta-learning with learned layerwise metric and
  subspace.
\newblock  International Conference on Machine Learning. PMLR,  2018, pp.
  2927--2936.

\bibitem[Koch \em{et~al.}(2015)Koch, Zemel, Salakhutdinov,
  et~al.]{koch2015siamese}
Koch, G.; Zemel, R.; Salakhutdinov, R.; others.
\newblock Siamese neural networks for one-shot image recognition.
\newblock  ICML deep learning workshop. Lille,  2015, Vol.~2.

\bibitem[Vinyals \em{et~al.}(2016)Vinyals, Blundell, Lillicrap, Wierstra,
  et~al.]{vinyals2016matching}
Vinyals, O.; Blundell, C.; Lillicrap, T.; Wierstra, D.; others.
\newblock Matching networks for one shot learning.
\newblock {\em Advances in neural information processing systems} {\bf 2016},
  {\em 29},~3630--3638.

\bibitem[Snell \em{et~al.}(2017)Snell, Swersky, and
  Zemel]{snell2017prototypical}
Snell, J.; Swersky, K.; Zemel, R.S.
\newblock Prototypical networks for few-shot learning.
\newblock {\em arXiv preprint arXiv:1703.05175} {\bf 2017}.

\bibitem[Mishra \em{et~al.}(2017)Mishra, Rohaninejad, Chen, and
  Abbeel]{mishra2017simple}
Mishra, N.; Rohaninejad, M.; Chen, X.; Abbeel, P.
\newblock A simple neural attentive meta-learner.
\newblock {\em arXiv preprint arXiv:1707.03141} {\bf 2017}.

\bibitem[Ravi and Larochelle(2016)]{ravi2016optimization}
Ravi, S.; Larochelle, H.
\newblock Optimization as a model for few-shot learning {\bf 2016}.

\bibitem[Qiao \em{et~al.}(2018)Qiao, Liu, Shen, and Yuille]{qiao2018few}
Qiao, S.; Liu, C.; Shen, W.; Yuille, A.L.
\newblock Few-shot image recognition by predicting parameters from activations.
\newblock  Proceedings of the IEEE Conference on Computer Vision and Pattern
  Recognition,  2018, pp. 7229--7238.

\bibitem[Graves \em{et~al.}(2014)Graves, Wayne, and
  Danihelka]{graves2014neural}
Graves, A.; Wayne, G.; Danihelka, I.
\newblock Neural turing machines.
\newblock {\em arXiv preprint arXiv:1410.5401} {\bf 2014}.

\bibitem[Rajeswaran \em{et~al.}(2019)Rajeswaran, Finn, Kakade, and
  Levine]{rajeswaran2019meta}
Rajeswaran, A.; Finn, C.; Kakade, S.; Levine, S.
\newblock Meta-learning with implicit gradients {\bf 2019}.

\bibitem[Li \em{et~al.}(2017)Li, Zhou, Chen, and Li]{li2017meta}
Li, Z.; Zhou, F.; Chen, F.; Li, H.
\newblock Meta-sgd: Learning to learn quickly for few-shot learning.
\newblock {\em arXiv preprint arXiv:1707.09835} {\bf 2017}.

\bibitem[Antoniou \em{et~al.}(2018)Antoniou, Edwards, and
  Storkey]{antoniou2018train}
Antoniou, A.; Edwards, H.; Storkey, A.
\newblock How to train your MAML.
\newblock {\em arXiv preprint arXiv:1810.09502} {\bf 2018}.

\bibitem[Nichol \em{et~al.}(2018)Nichol, Achiam, and Schulman]{nichol2018first}
Nichol, A.; Achiam, J.; Schulman, J.
\newblock On first-order meta-learning algorithms.
\newblock {\em arXiv preprint arXiv:1803.02999} {\bf 2018}.

\bibitem[Finn \em{et~al.}(2019)Finn, Rajeswaran, Kakade, and
  Levine]{finn2019online}
Finn, C.; Rajeswaran, A.; Kakade, S.; Levine, S.
\newblock Online meta-learning.
\newblock  International Conference on Machine Learning. PMLR,  2019, pp.
  1920--1930.

\bibitem[Grant \em{et~al.}(2018)Grant, Finn, Levine, Darrell, and
  Griffiths]{grant2018recasting}
Grant, E.; Finn, C.; Levine, S.; Darrell, T.; Griffiths, T.
\newblock Recasting gradient-based meta-learning as hierarchical bayes.
\newblock {\em arXiv preprint arXiv:1801.08930} {\bf 2018}.

\bibitem[Finn \em{et~al.}(2018)Finn, Xu, and Levine]{finn2018probabilistic}
Finn, C.; Xu, K.; Levine, S.
\newblock Probabilistic model-agnostic meta-learning.
\newblock {\em arXiv preprint arXiv:1806.02817} {\bf 2018}.

\bibitem[Yoon \em{et~al.}(2018)Yoon, Kim, Dia, Kim, Bengio, and
  Ahn]{yoon2018bayesian}
Yoon, J.; Kim, T.; Dia, O.; Kim, S.; Bengio, Y.; Ahn, S.
\newblock Bayesian model-agnostic meta-learning.
\newblock  Proceedings of the 32nd International Conference on Neural
  Information Processing Systems,  2018, pp. 7343--7353.

\bibitem[Vuorio \em{et~al.}(2019)Vuorio, Sun, Hu, and
  Lim]{vuorio2019multimodal}
Vuorio, R.; Sun, S.H.; Hu, H.; Lim, J.J.
\newblock Multimodal model-agnostic meta-learning via task-aware modulation.
\newblock {\em arXiv preprint arXiv:1910.13616} {\bf 2019}.

\bibitem[Yao \em{et~al.}(2019)Yao, Wei, Huang, and Li]{yao2019hierarchically}
Yao, H.; Wei, Y.; Huang, J.; Li, Z.
\newblock Hierarchically structured meta-learning.
\newblock  International Conference on Machine Learning. PMLR,  2019, pp.
  7045--7054.

\bibitem[Yin \em{et~al.}(2020)Yin, Tucker, Zhou, Levine, and
  Finn]{yin2020metalearning}
Yin, M.; Tucker, G.; Zhou, M.; Levine, S.; Finn, C.
\newblock Meta-Learning without Memorization.
\newblock  International Conference on Learning Representations,  2020.

\bibitem[Jenni and Favaro(2018)]{jenni2018deep}
Jenni, S.; Favaro, P.
\newblock Deep bilevel learning.
\newblock  Proceedings of the European conference on computer vision (ECCV),
  2018, pp. 618--633.

\bibitem[Chen \em{et~al.}(2019)Chen, Liu, Kira, Wang, and
  Huang]{chen19closerfewshot}
Chen, W.Y.; Liu, Y.C.; Kira, Z.; Wang, Y.C.; Huang, J.B.
\newblock A Closer Look at Few-shot Classification.
\newblock  International Conference on Learning Representations,  2019.

\bibitem[Li \em{et~al.}(2019)Li, Yang, Zhou, and Hospedales]{li2019feature}
Li, Y.; Yang, Y.; Zhou, W.; Hospedales, T.
\newblock Feature-critic networks for heterogeneous domain generalization.
\newblock  International Conference on Machine Learning. PMLR,  2019, pp.
  3915--3924.

\bibitem[Rebuffi \em{et~al.}(2017)Rebuffi, Bilen, and Vedaldi]{Rebuffi17}
Rebuffi, S.A.; Bilen, H.; Vedaldi, A.
\newblock Learning multiple visual domains with residual adapters.
\newblock  Advances in Neural Information Processing Systems,  2017.

\bibitem[Triantafillou \em{et~al.}(2020)Triantafillou, Zhu, Dumoulin, Lamblin,
  Evci, Xu, Goroshin, Gelada, Swersky, Manzagol, and Larochelle]{48798}
Triantafillou, E.; Zhu, T.; Dumoulin, V.; Lamblin, P.; Evci, U.; Xu, K.;
  Goroshin, R.; Gelada, C.; Swersky, K.J.; Manzagol, P.A.; Larochelle, H.
\newblock Meta-Dataset: A Dataset of Datasets for Learning to Learn from Few
  Examples.
\newblock  International Conference on Learning Representations,  2020.

\bibitem[Yu \em{et~al.}(2020)Yu, Quillen, He, Julian, Hausman, Finn, and
  Levine]{yu2020meta}
Yu, T.; Quillen, D.; He, Z.; Julian, R.; Hausman, K.; Finn, C.; Levine, S.
\newblock Meta-world: A benchmark and evaluation for multi-task and meta
  reinforcement learning.
\newblock  Conference on Robot Learning. PMLR,  2020, pp. 1094--1100.

\bibitem[Lundervold and Lundervold(2019)]{lundervold2019overview}
Lundervold, A.S.; Lundervold, A.
\newblock An overview of deep learning in medical imaging focusing on MRI.
\newblock {\em Zeitschrift f{\"u}r Medizinische Physik} {\bf 2019}, {\em
  29},~102--127.

\bibitem[Liang \em{et~al.}(2020)Liang, Cheng, Ke, and Ying]{liang2020deep}
Liang, D.; Cheng, J.; Ke, Z.; Ying, L.
\newblock Deep magnetic resonance image reconstruction: Inverse problems meet
  neural networks.
\newblock {\em IEEE Signal Processing Magazine} {\bf 2020}, {\em 37},~141--151.

\bibitem[Sandino \em{et~al.}(2020)Sandino, Cheng, Chen, Mardani, Pauly, and
  Vasanawala]{sandino2020compressed}
Sandino, C.M.; Cheng, J.Y.; Chen, F.; Mardani, M.; Pauly, J.M.; Vasanawala,
  S.S.
\newblock Compressed sensing: From research to clinical practice with deep
  neural networks: Shortening scan times for magnetic resonance imaging.
\newblock {\em IEEE signal processing magazine} {\bf 2020}, {\em 37},~117--127.

\bibitem[McCann \em{et~al.}(2017)McCann, Jin, and
  Unser]{mccann2017convolutional}
McCann, M.T.; Jin, K.H.; Unser, M.
\newblock Convolutional neural networks for inverse problems in imaging: A
  review.
\newblock {\em IEEE Signal Processing Magazine} {\bf 2017}, {\em 34},~85--95.

\bibitem[Zhou \em{et~al.}(2020)Zhou, Greenspan, Davatzikos, Duncan, van
  Ginneken, Madabhushi, Prince, Rueckert, and Summers]{zhou2020review}
Zhou, S.K.; Greenspan, H.; Davatzikos, C.; Duncan, J.S.; van Ginneken, B.;
  Madabhushi, A.; Prince, J.L.; Rueckert, D.; Summers, R.M.
\newblock A review of deep learning in medical imaging: Image traits,
  technology trends, case studies with progress highlights, and future
  promises.
\newblock {\em Unknown Journal} {\bf 2020}.

\bibitem[Singha \em{et~al.}(2021)Singha, Thakur, and Patel]{singha2021deep}
Singha, A.; Thakur, R.S.; Patel, T.
\newblock Deep Learning Applications in Medical Image Analysis.
\newblock {\em Biomedical Data Mining for Information Retrieval: Methodologies,
  Techniques and Applications} {\bf 2021}, pp. 293--350.

\bibitem[Chandra \em{et~al.}(2021)Chandra, Bran~Lorenzana, Liu, Liu, Bollmann,
  and Crozier]{chandra2021deep}
Chandra, S.S.; Bran~Lorenzana, M.; Liu, X.; Liu, S.; Bollmann, S.; Crozier, S.
\newblock Deep learning in magnetic resonance image reconstruction.
\newblock {\em Journal of Medical Imaging and Radiation Oncology} {\bf 2021}.

\bibitem[Ahishakiye \em{et~al.}(2021)Ahishakiye, Van~Gijzen, Tumwiine, Wario,
  and Obungoloch]{ahishakiye2021survey}
Ahishakiye, E.; Van~Gijzen, M.B.; Tumwiine, J.; Wario, R.; Obungoloch, J.
\newblock A survey on deep learning in medical image reconstruction.
\newblock {\em Intelligent Medicine} {\bf 2021}.

\bibitem[Liu \em{et~al.}(2020)Liu, Zhang, Cheng, Luo, and Fan]{liu2020deep}
Liu, R.; Zhang, Y.; Cheng, S.; Luo, Z.; Fan, X.
\newblock A Deep Framework Assembling Principled Modules for CS-MRI: Unrolling
  Perspective, Convergence Behaviors, and Practical Modeling.
\newblock {\em IEEE Transactions on Medical Imaging} {\bf 2020}, {\em
  39},~4150--4163.

\bibitem[Cheng \em{et~al.}(2019)Cheng et~al.]{cheng2019model}
Cheng, J.; others.
\newblock Model learning: Primal dual networks for fast MR imaging.
\newblock  International Conference on Medical Image Computing and
  Computer-Assisted Intervention. Springer,  2019, pp. 21--29.

\bibitem[Bian \em{et~al.}(2020)Bian, Chen, and Ye]{bian2020deep}
Bian, W.; Chen, Y.; Ye, X.
\newblock Deep Parallel MRI Reconstruction Network Without Coil Sensitivities.
\newblock  International Workshop on Machine Learning for Medical Image
  Reconstruction. Springer,  2020, pp. 17--26.

\bibitem[Yang \em{et~al.}(2018)Yang, Sun, Li, and Xu]{yang2018admm}
Yang, Y.; Sun, J.; Li, H.; Xu, Z.
\newblock ADMM-CSNet: A deep learning approach for image compressive sensing.
\newblock {\em IEEE transactions on pattern analysis and machine intelligence}
  {\bf 2018}, {\em 42},~521--538.

\bibitem[Hammernik \em{et~al.}(2018)Hammernik, Klatzer, Kobler, Recht,
  Sodickson, Pock, and Knoll]{hammernik2018learning}
Hammernik, K.; Klatzer, T.; Kobler, E.; Recht, M.P.; Sodickson, D.K.; Pock, T.;
  Knoll, F.
\newblock Learning a variational network for reconstruction of accelerated MRI
  data.
\newblock {\em Magnetic resonance in medicine} {\bf 2018}, {\em
  79},~3055--3071.

\bibitem[Zhang and Ghanem(2018)]{zhang2018ista}
Zhang, J.; Ghanem, B.
\newblock ISTA-Net: Interpretable optimization-inspired deep network for image
  compressive sensing.
\newblock  Proceedings of the IEEE Conference on Computer Vision and Pattern
  Recognition,  2018, pp. 1828--1837.

\bibitem[Aggarwal \em{et~al.}(2018)Aggarwal, Mani, and Jacob]{aggarwal2018modl}
Aggarwal, H.K.; Mani, M.P.; Jacob, M.
\newblock MoDL: Model-based deep learning architecture for inverse problems.
\newblock {\em IEEE transactions on medical imaging} {\bf 2018}, {\em
  38},~394--405.

\bibitem[Schlemper \em{et~al.}(2017)Schlemper, Caballero, Hajnal, Price, and
  Rueckert]{schlemper2017deep}
Schlemper, J.; Caballero, J.; Hajnal, J.V.; Price, A.N.; Rueckert, D.
\newblock A deep cascade of convolutional neural networks for dynamic MR image
  reconstruction.
\newblock {\em IEEE transactions on Medical Imaging} {\bf 2017}, {\em
  37},~491--503.

\bibitem[He \em{et~al.}(2016)He, Zhang, Ren, and Sun]{he2016deep}
He, K.; Zhang, X.; Ren, S.; Sun, J.
\newblock Deep residual learning for image recognition.
\newblock  Proceedings of the IEEE conference on computer vision and pattern
  recognition,  2016, pp. 770--778.

\bibitem[Mehra and Hamm(2019)]{mehra2019penalty}
Mehra, A.; Hamm, J.
\newblock Penalty method for inversion-free deep bilevel optimization.
\newblock {\em arXiv preprint arXiv:1911.03432} {\bf 2019}.

\bibitem[Kingma and Ba(2015)]{kingma2014adam}
Kingma, D.P.; Ba, J.
\newblock Adam: {A} Method for Stochastic Optimization.
\newblock  3rd International Conference on Learning Representations, {ICLR}
  2015, San Diego, CA, USA, May 7-9, 2015, Conference Track Proceedings;
  Bengio, Y.; LeCun, Y., Eds.,  2015.

\bibitem[Glorot and Bengio(2010)]{Glorot10understandingthe}
Glorot, X.; Bengio, Y.
\newblock Understanding the difficulty of training deep feedforward neural
  networks.
\newblock  In Proceedings of the International Conference on Artificial
  Intelligence and Statistics. Society for Artificial Intelligence and
  Statistics,  2010.

\bibitem[Bernstein \em{et~al.}(2001)Bernstein, Fain, and
  Riederer]{bernstein2001effect}
Bernstein, M.A.; Fain, S.B.; Riederer, S.J.
\newblock Effect of windowing and zero-filled reconstruction of {MRI} data on
  spatial resolution and acquisition strategy.
\newblock {\em Journal of Magnetic Resonance Imaging: An Official Journal of
  the International Society for Magnetic Resonance in Medicine} {\bf 2001},
  {\em 14},~270--280.

\bibitem[Abadi \em{et~al.}(2015)Abadi, Agarwal, Barham, Brevdo, Chen, Citro,
  Corrado, Davis, Dean, Devin, Ghemawat, Goodfellow, Harp, Irving, Isard, Jia,
  Jozefowicz, Kaiser, Kudlur, Levenberg, Man\'{e}, Monga, Moore, Murray, Olah,
  Schuster, Shlens, Steiner, Sutskever, Talwar, Tucker, Vanhoucke, Vasudevan,
  Vi\'{e}gas, Vinyals, Warden, Wattenberg, Wicke, Yu, and
  Zheng]{tensorflow2015-whitepaper}
Abadi, M.; Agarwal, A.; Barham, P.; Brevdo, E.; Chen, Z.; Citro, C.; Corrado,
  G.S.; Davis, A.; Dean, J.; Devin, M.; Ghemawat, S.; Goodfellow, I.; Harp, A.;
  Irving, G.; Isard, M.; Jia, Y.; Jozefowicz, R.; Kaiser, L.; Kudlur, M.;
  Levenberg, J.; Man\'{e}, D.; Monga, R.; Moore, S.; Murray, D.; Olah, C.;
  Schuster, M.; Shlens, J.; Steiner, B.; Sutskever, I.; Talwar, K.; Tucker, P.;
  Vanhoucke, V.; Vasudevan, V.; Vi\'{e}gas, F.; Vinyals, O.; Warden, P.;
  Wattenberg, M.; Wicke, M.; Yu, Y.; Zheng, X.
\newblock {TensorFlow}: Large-Scale Machine Learning on Heterogeneous Systems,
  2015.
\newblock Software available from tensorflow.org.

\bibitem[Menze \em{et~al.}(2014)Menze, Jakab, Bauer, Kalpathy-Cramer, Farahani,
  Kirby, Burren, Porz, Slotboom, Wiest, et~al.]{menze2014multimodal}
Menze, B.H.; Jakab, A.; Bauer, S.; Kalpathy-Cramer, J.; Farahani, K.; Kirby,
  J.; Burren, Y.; Porz, N.; Slotboom, J.; Wiest, R.; others.
\newblock The multimodal brain tumor image segmentation benchmark (BRATS).
\newblock {\em IEEE transactions on medical imaging} {\bf 2014}, {\em
  34},~1993--2024.

\bibitem[Abadi \em{et~al.}(2016)Abadi et~al.]{abadi2016tensorflow}
Abadi, M.; others.
\newblock Tensorflow: A system for large-scale machine learning.
\newblock  12th $\{$USENIX$\}$ symposium on operating systems design and
  implementation ($\{$OSDI$\}$ 16),  2016, pp. 265--283.

\bibitem[Glorot and Bengio(2010)]{glorot2010understanding}
Glorot, X.; Bengio, Y.
\newblock Understanding the difficulty of training deep feedforward neural
  networks.
\newblock  Proceedings of the thirteenth international conference on artificial
  intelligence and statistics,  2010, pp. 249--256.

\bibitem[Wang \em{et~al.}(2020)Wang et~al.]{WANG2020136}
Wang, S.; others.
\newblock DeepcomplexMRI: Exploiting deep residual network for fast parallel MR
  imaging with complex convolution.
\newblock {\em Magnetic Resonance Imaging} {\bf 2020}, {\em 68},~136 -- 147.

\bibitem[Wang \em{et~al.}(2004)Wang et~al.]{wang2004image}
Wang, Z.; others.
\newblock Image quality assessment: from error visibility to structural
  similarity.
\newblock {\em IEEE transactions on image processing} {\bf 2004}, {\em
  13},~600--612.

\end{thebibliography}

\end{document}